\documentclass[preprint,11pt,authoryear]{elsarticle}

\usepackage{enumerate}
\usepackage[acronym,nonumberlist,numberedsection,nopostdot,automake]{glossaries}

\newglossarystyle{mystyle}{%
  \setglossarystyle{long}%
     {\begin{longtable}[l]{@{}p{0.13\hsize}p{0.87\hsize}}}%
     {\end{longtable}}%
}

\makeglossaries

\usepackage{amsthm}

\newcommand{\qedAtBottomLine}{\ensuremath{\\[-0.82\baselineskip] & \qedhere}}
\usepackage{thmtools}
\usepackage{thm-restate}

\newtheorem{mythm}{Theorem}
\newtheorem{mycor}[mythm]{Corollary}
\newtheorem{mylem}[mythm]{Lemma}

\newtheorem{mydef}[mythm]{Definition}
\newtheorem{myconjecture}[mythm]{Conjecture}
\newtheorem{myRemark}[mythm]{Remark}

\makeatletter
\def\cleartheorem#1{%
    \expandafter\let\csname#1\endcsname\relax
    \expandafter\let\csname c@#1\endcsname\relax
}
\makeatother

\newcommand{\secref}[1]{\text{Sec.}~\ref{#1}}
\newcommand{\appref}[1]{\ref{#1}}
\newcommand{\figref}[1]{\text{Fig.}~\ref{#1}}
\newcommand{\algoref}[1]{\text{Algorithm}~\ref{#1}}

\newcommand{\thmref}[1]{\text{Theorem}~\ref{#1}}
\newcommand{\lemref}[1]{\text{Lemma}~\ref{#1}}
\newcommand{\corollaryref}[1]{\text{Corollary~\ref{#1}}}
\newcommand{\propref}[1]{\text{Proposition}~\ref{#1}}
\newcommand{\defref}[1]{\text{Definition}~\ref{#1}}

\newcommand{\remref}[1]{\text{Remark}~\ref{#1}}

\usepackage{dsfont}
\usepackage[utf8]{inputenc} 
\usepackage[T1]{fontenc}    
\usepackage{hyperref}       
\usepackage{url}            
\usepackage{booktabs}       
\usepackage{amsfonts}       
\usepackage{amsmath}
\usepackage{amssymb}
\usepackage{bm}
\usepackage{bbm}
\usepackage{nicefrac}       
\usepackage{stmaryrd}
\usepackage{relsize}
\usepackage{scalerel}
\usepackage{microtype}      
\usepackage{xspace}

\usepackage{color}
\usepackage{algorithm}
\usepackage[noend]{algpseudocode}
\usepackage{dashrule}
\usepackage{multirow}
\newcommand\Algphase[1]{%
\vspace*{-.7\baselineskip}\Statex\hspace*{\dimexpr-\algorithmicindent-2pt\relax}\hdashrule[-0.4ex]{1.\textwidth}{0.4pt}{4pt 3pt}%
\Statex\hspace*{-\algorithmicindent}\textbf{#1}%
\vspace*{-.7\baselineskip}\Statex\hspace*{\dimexpr-\algorithmicindent-2pt\relax}\hdashrule{1.\textwidth}{0.4pt}{4pt 3pt}%
}
\newcommand{\isHomoscedastic}{\textbf{homoscedastic}\xspace}

\algnewcommand{\LineComment}[1]{\State \(\triangleright\) #1}
\usepackage{graphicx}
\newcommand{\fig}[5]{\begin{figure}[tb]
  \centering
  \includegraphics[width=#2\linewidth]{#1}
  \vspace*{-#4}
  \caption{#3}
  \vspace*{-#5}
  \label{fig:#1}
\end{figure}
}

\usepackage{soul}
\DeclareRobustCommand*{\showDifferenceDanny}[2]{\ifthenelse{\equal{#1}{}}{}{\textcolor{red}{\st{#1}}}\xspace \ifthenelse{\equal{#2}{}}{}{\textcolor{blue}{#2}}\xspace}
\DeclareRobustCommand*{\showDifferenceShin}[2]{\ifthenelse{\equal{#1}{}}{}{\textcolor{red}{\st{#1}}} \ifthenelse{\equal{#2}{}}{}{\textcolor{magenta}{#2}}\xspace}
\DeclareRobustCommand*{\showDifferenceKlaus}[2]{\ifthenelse{\equal{#1}{}}{}{\textcolor{red}{\st{#1}}} \ifthenelse{\equal{#2}{}}{}{\textcolor{green}{#2}}\xspace}

\newcommand{\mySum}[2]{\sideset{}{_{#1}^{#2}}\sum}

\DeclareMathOperator*{\Prob}{\mathds{P}}
\DeclareMathOperator*{\R}{\mathds{R}}
\newcommand{\reals}[1]{\ensuremath{\R^{#1}_{}}\xspace}
\DeclareMathOperator*{\N}{\mathds{N}}
\DeclareMathOperator*{\E}{\mathds{E}}
\DeclareMathOperator*{\Bias}{\mathds{B}}
\DeclareMathOperator*{\Var}{\mathds{V}}
\DeclareMathOperator*{\indicator}{\scalerel*{\mathbbmss{1}}{\textstyle\sum}}
\newcommand{\indicatorFunction}[2]{\ensuremath{\sideset{}{_{#1}^{}}\indicator(#2)}}
\newcommand{\idMatrix}[1]{\ensuremath{\mathcal{I}_{#1}}}
\newcommand{\ones}[1]{\ensuremath{\mathbbmss{1}_{#1}^{}}}
\newcommand{\onesT}[1]{\ensuremath{\mathbbmss{1}_{#1}^\top}}
\newcommand{\abs}[1]{\left|#1\right|}
\newcommand{\pnorm}[2]{\|#1\|_{\raisebox{-2pt}{\tiny\ensuremath{#2}}}}

\DeclareMathOperator*{\argmin}{\textbf{arg\hspace{0.1em}min}}

\newcommand{\sign}{\textbf{\textit{sgn}}\xspace}
\newcommand{\diag}{\textbf{\textit{diag}}\xspace}
\newcommand{\vect}{\textbf{\textit{vec}}\xspace}

\newcommand{\trace}{\textbf{\textit{trace}}\xspace}
\newcommand{\vol}{\textbf{\textit{Vol}}\xspace}
\newcommand{\symmetricSet}[1]{\ensuremath{\mathbb{S}^{#1}}\xspace}

\newcommand{\nonnegativeReal}[1]{\ensuremath{\mathds{R}^{#1}_{\raisebox{0pt}{\tiny\ensuremath{+}}}}\xspace}
\newcommand{\posDefSet}[1]{\ensuremath{\mathbb{S}^{#1}_{\raisebox{0pt}{\tiny\ensuremath{++}}}}\xspace}
\newcommand{\positiveReal}[1]{\ensuremath{\mathds{R}^{#1}_{\raisebox{0pt}{\tiny\ensuremath{++}}}}\xspace}
\newcommand{\inputSpace}{\ensuremath{\mathcal{X}}\xspace}
\newcommand{\inputSpaceInterior}{\ensuremath{\mathcal{X}^\circ}\xspace}
\newcommand{\effectiveInputSpaceInterior}{\ensuremath{\mathcal{X}^\circ_n}\xspace}
\newcommand{\supportInteriorStdDevs}{\ensuremath{\mathfrak{s}}\xspace}

\newcommand{\uniformDist}[1]{\ensuremath{\mathcal{U}(#1)}\xspace}
\newcommand{\Gauss}[3]{\ensuremath{\mathcal{N}(#1; #2, #3)}\xspace}
\newcommand{\fctn}[3]{\ensuremath{#1\!:#2\,\rightarrow\,#3}\xspace}

\newcommand{\condset}[2]{\ensuremath{\left\{#1\;\middle|\;#2\right\}}\xspace}
\newcommand{\diffableFunctions}[2]{\ensuremath{\mathcal{C}^{#2}\left(#1\right)}\xspace}
\newcommand{\relSampleSize}[2]{\raisebox{1pt}{\ensuremath{\varrho}}(#1, #2)\xspace}
\newcommand{\relSampleSizeSymbol}{\raisebox{1pt}{\ensuremath{\varrho}}\xspace}

\newcommand{\LPSorder}{\ensuremath{Q}\xspace}
\newcommand{\predictorLPS}[2]{\ensuremath{m^{#2}_{#1}}\xspace}
\newcommand{\biasLLS}[3]{\ensuremath{\text{bias}_{1,#3}^{}\left[#2,#1\right]}\xspace}
\newcommand{\biasLPS}[3]{\ensuremath{\text{bias}_{#3}^{}\left[#2,#1\right]}\xspace}

\newcommand{\varianceLPS}[3]{\ensuremath{\text{var}_{#3}^{}\left[#2,#1\middle|\bm{X}_n\right]}\xspace}
\newcommand{\bandwidth}{\ensuremath{\Sigma}\xspace}
\newcommand{\SigmaSpace}{\ensuremath{\mathcal{S}}\xspace}
\newcommand{\conditionNumberUpperBound}{\ensuremath{\mathcal{M}}\xspace}
\newcommand{\consistentSigmaSequencSpace}{\ensuremath{\mathfrak{S}}\xspace}
\newcommand{\convergenceInProbability}[2]{\ifthenelse{\equal{#1}{}}{\ensuremath{o_p^{}\left[#2\right]}}{\ensuremath{o_p^{}\begingroup\edef\x{\endgroup#1[}\x #2 \begingroup\edef\x{\endgroup#1]}\x}}\xspace}
\newcommand{\upperBoundedInProbability}[2]{\ifthenelse{\equal{#1}{}}{\ensuremath{O_p^{}\left[#2\right]}}{\ensuremath{O_p^{}\begingroup\edef\x{\endgroup#1[}\x #2 \begingroup\edef\x{\endgroup#1]}\x}}\xspace}
\newcommand{\exactRateconvergenceInProbability}[2]{\ifthenelse{\equal{#1}{}}{\ensuremath{\theta_p^{}\left[#2\right]}}{\ensuremath{\theta_p^{}\begingroup\edef\x{\endgroup#1[}\x #2 \begingroup\edef\x{\endgroup#1]}\x}}\xspace}
\newcommand{\divergenceInProbability}[2]{\ifthenelse{\equal{#1}{}}{\ensuremath{\omega_p^{}\left[#2\right]}}{\ensuremath{\omega_p^{}\begingroup\edef\x{\endgroup#1[}\x #2 \begingroup\edef\x{\endgroup#1]}\x}}\xspace}
\newcommand{\lowerBoundedInProbability}[2]{\ifthenelse{\equal{#1}{}}{\ensuremath{\Omega_p^{}\left[#2\right]}}{\ensuremath{\Omega_p^{}\begingroup\edef\x{\endgroup#1[}\x #2 \begingroup\edef\x{\endgroup#1]}\x}}\xspace}

\newcommand{\LOBfunctionOfLPS}[2]{\ensuremath{\bandwidth^{#2}_{#1}}\xspace}
\newcommand{\approxLOBfunctionOfLPS}[2]{\ensuremath{\widehat{\bandwidth}^{#2}_{#1}}\xspace}

\newcommand{\normalizedLOBfunction}{\ensuremath{\mathfrak{S}^{n}_{}}\xspace}

\newcommand{\normalizedLOBLPS}[2]{\ensuremath{\mathfrak{S}^{#2}_{#1}}\xspace}
\newcommand{\fOpt}[1]{\ensuremath{\widehat{f}^{#1}_\text{\tiny{Opt}}}\xspace}
\newcommand{\fMT}{\ensuremath{\widehat{f}^{}_\text{\tiny{MT}}}\xspace}
\newcommand{\pMT}{\ensuremath{\widehat{p}^{}_\text{\tiny{MT}}}\xspace}
\newcommand{\fMF}{\ensuremath{\widehat{f}^{}_\text{\tiny{MF}}}\xspace}
\newcommand{\pMF}{\ensuremath{\widehat{p}^{}_\text{\tiny{MF}}}\xspace}
\newcommand{\fBull}{\ensuremath{\widehat{f}^{}_\text{\tiny{Bull}}}\xspace}
\newcommand{\pBull}{\ensuremath{\widehat{p}^{}_\text{\tiny{Bull}}}\xspace}

\newcommand{\fctnComplexityLPS}[2]{\ensuremath{\mathfrak{C}^{#2}_{#1}}\xspace}
\newcommand{\approxFctnComplexityLPS}[2]{\ensuremath{\widehat{\mathfrak{C}}^{#2}_{#1}}\xspace}

\newcommand{\pOptLPS}[2]{\ensuremath{p^{#1,#2}_\text{\tiny{Opt}}}\xspace}
\newcommand{\approxPOptLPS}[2]{\ensuremath{\widehat{p}^{#1,#2}_\text{\tiny{Opt}}}\xspace}
\newcommand{\tildePOptLPS}[2]{\ensuremath{\widetilde{p}^{#1,#2}_\text{\tiny{Opt}}}\xspace}
\newcommand{\pOptFinite}{\ensuremath{p^{n}_\text{\tiny{Opt}}}\xspace}

\newcommand{\sigmaAsympLPS}[1]{\ensuremath{\sigma^{#1,n}_{\text{\tiny{Asymp}}}}\xspace}
\newcommand{\SigmaAsympLPS}[1]{\ensuremath{\Sigma^{#1,n}_{\text{\tiny{Asymp}}}}\xspace}
\newcommand{\vanishingSet}[2]{\ensuremath{\mathcal{T}_{#1}^{#2}}\xspace}
\newcommand{\LPS}{\text{LPS}\xspace}
\newcommand{\LLS}{\text{LLS}\xspace}
\newcommand{\LCS}{\text{LCS}\xspace}
\newcommand{\MSE}{\text{MSE}\xspace}
\newcommand{\RMSE}{\text{RMSE}\xspace}
\newcommand{\MISE}{\text{MISE}\xspace}
\newcommand{\MAE}{\text{MAE}\xspace}
\newcommand{\LOB}{\text{LOB}\xspace}
\newcommand{\LFC}{\text{LFC}\xspace}
\newcommand{\RBF}{\text{RBF}\xspace}
\newcommand{\GPR}{\text{GPR}\xspace}
\newcommand{\randomTestSampling}{\emph{random test sampling}\xspace}
\newcommand{\robust}{\emph{robust}\xspace}
\newcommand{\robustProperty}{\emph{robustness}\xspace}
\newcommand{\optimal}{\emph{optimal}\xspace}
\newcommand{\optimalProperty}{\emph{optimality}\xspace}
\newcommand{\modelagnostic}{\emph{model-agnostic}\xspace}
\newcommand{\modelagnosticProperty}{\emph{model-agnosticity}\xspace}
\newcommand{\stationary}{\emph{stationary}\xspace}
\newcommand{\interpretable}{\emph{interpretable}\xspace}

\newacronym[sort=d01]{LPS}{\LPS}{Local polynomial smoothing}
\newacronym[sort=d02]{LLS}{\LLS}{Local linear smoothing}
\newacronym[sort=d03]{LCS}{\LCS}{Local cubic smoothing}
\newacronym[sort=d04]{MSE}{\MSE}{Mean squared error}
\newacronym[sort=d05]{RMSE}{\RMSE}{Root mean squared error}
\newacronym[sort=d06]{MISE}{\MISE}{Mean integrated squared error}
\newacronym[sort=d07]{MAE}{\MAE}{Maximum absolute error}
\newacronym[sort=d08]{LOB}{\LOB}{Locally optimal bandwidth}
\newacronym[sort=d09]{LFC}{\LFC}{Local function complexity}
\newacronym[sort=d10]{RBF}{\RBF}{Radial basis function}
\newacronym[sort=d11]{GPR}{\GPR}{Gaussian process regression}
\newcommand{\DP}{\text{D.~Panknin}\xspace}
\newcommand{\SN}{\text{S.~Nakajima}\xspace}
\newcommand{\KRM}{\text{K.-R.~M\"uller}\xspace}
\soulregister{\KRM}{7}
\soulregister{\SN}{7}
\soulregister{\eqref}{7}
\soulregister{\secref}{7}
\soulregister{\appref}{7}
\soulregister{\remref}{7}
\soulregister{\cite}{7}
\soulregister{\citep}{7}
\soulregister{\modelagnostic}{7}
\soulregister{\modelagnosticProperty}{7}
\soulregister{\robust}{7}
\soulregister{\robustProperty}{7}
\soulregister{\optimal}{7}
\soulregister{\optimalProperty}{7}

\journal{Neurocomputing}

\begin{document}
\newif\ifshowComments
\showCommentstrue

\begin{frontmatter}



\title{Optimal Sampling Density for Nonparametric Regression}


 
 
\author[1]{Danny~Panknin\texorpdfstring{\corref{cor1}}{}}
\ead{danny.panknin@tu-berlin.de}
\cortext[cor1]{Corresponding authors}

\author[1,2,3,4]{Klaus-Robert~M\"uller\texorpdfstring{\corref{cor1}}{}}
\ead{klaus-robert.mueller@tu-berlin.de}

\author[1,2,5]{Shinichi~Nakajima\texorpdfstring{\corref{cor1}}{}}
\ead{nakajima@tu-berlin.de}

\affiliation[1]{organization={Machine Learning Department},
            addressline={Berlin Institute of Technology}, 
            city={10587 Berlin},
            country={Germany}}

\affiliation[2]{organization={BIFOLD-Berlin Institute for the Foundations of Learning and Data},
            country={Germany}}

\affiliation[3]{organization={Department of Artificial Intelligence},
            addressline={Korea University},
            city={Seoul 136-713},
            country={South Korea}}
 
\affiliation[4]{organization={Max Planck Institute for Informatics},
            postcode={66123}, 
            state={Saarbrücken},
            country={Germany}}
            
\affiliation[5]{organization={RIKEN AIP},
            addressline={1-4-1 Nihonbashi}, 
            city={Chuo-ku},
            state={Tokyo},
            country={Japan}}
            
\begin{abstract}
We propose a novel active learning strategy for regression, which is
model-agnostic, robust against model mismatch, and interpretable. 
Assuming that 
a small number of initial samples are available, we derive the optimal training density that minimizes the generalization error
of \emph{local polynomial smoothing} (\LPS) with its kernel bandwidth tuned locally:
We adopt the \emph{mean integrated squared error} (\MISE) as a generalization criterion,
and use the asymptotic behavior of the \MISE as well as the \emph{locally optimal bandwidths} (\LOB) -- the bandwidth function that minimizes \MISE in the asymptotic limit.
The asymptotic expression of our objective then reveals
the dependence of the \MISE on the training density, enabling analytic minimization.
As a result, we obtain the optimal training density in a closed-form.
The almost model-free nature of our approach thus helps to encode the essential properties of the target problem, providing a robust and model-agnostic active learning strategy.
Furthermore, 
the obtained training density factorizes the influence of local function complexity, noise level and test density in a transparent and interpretable way.
We validate our theory in numerical simulations, and show that the proposed active learning method outperforms the existing state-of-the-art model-agnostic approaches.
\end{abstract}


\begin{highlights}
\item We derive a novel \emph{active learning} framework, based on the \emph{local polynomial smoothing} model class
\item It samples from the \emph{optimal training density} that asymptotically minimizes the generalization error of \emph{local polynomial smoothing}
\item The \emph{optimal training density} is \interpretable as it factorizes the influence of local, problem intrinsic properties such as the \emph{noise level}, \emph{function complexity} and \emph{test relevance}
\item We apply local bandwidth estimates via \emph{Lepski's method} to provide an implementation of our \emph{active learning} framework in the \emph{isotropic} case
\item We provide empirical evidence that our proposed \emph{active learning} framework is \modelagnostic by applying it to a neural network and a random forest model
\end{highlights}


\end{frontmatter}

\section{Introduction}
\label{sec:intro}
Active learning is a powerful tool for inference when acquiring labels is expensive.
Given a fixed budget of labels that may be queried, the basic idea is to construct a training set in a way that minimizes a predefined generalization loss.
Active learning for classification has been applied successfully in learning approaches to text categorization \citep{lewis1994sequential,roy01,goudjil2018novel}, biomedical data analysis \citep{warmuth2003active,pasolli2010,saito2015robust,bressan2019breast}, image classification \citep{sener2018active,beluch2018power,haut2018active} and image retrieval \citep{tong2001support,he2010laplacian}.
In the classification regime these approaches were able to reduce the required amount of training data drastically, where the selection criteria are commonly based on input space geometric arguments.
For example, one may request labels close to the decision boundary, for which the prediction uncertainty is high \citep{tong2001support,warmuth2003active}.
This way, most of the input space that is associated to 
the interior of the class supports can be neglected while training.
More recently, active learning was deployed in regression tasks such as wind speed forecasting \citep{douak2013}, optimal control \citep{wu2020active} and reinforcement learning \citep{teytaud2007active}, quantum chemistry \citep{tang2019prediction,gubaev2018machine} and industrial applications in semiconductor manufacturing \citep{sugiyama2009pool}.

Active learning approaches can be categorized with respect to several properties such as \emph{informativeness}, \emph{representativeness} and \emph{diversity} \citep{settles2010active,wu2019pool} and we refer to the broad overview of literature on active learning reviewed by \cite{settles2010active}.
Each active learning approach induces a \emph{sampling scheme}, by which we mean the process of successively adding new labeled instances to the training dataset.
In this paper we distinguish between \emph{supervised} and \emph{unsupervised} sampling schemes: A supervised sampling scheme is based on a sampling criterion that depends on the so far acquired training labels. Any
sampling scheme that is not \emph{supervised} is hence regarded as unsupervised.
We will refer to i.i.d.~sampling from some test distribution as \randomTestSampling, which is the most simple unsupervised baseline.
Note that in the literature of active learning the notion of unsupervised sampling schemes \citep{liu2021pool} are synonymously referred to as \emph{passive} \citep{yu2010passive,wu2019pool} or \emph{blind} \citep{teytaud2007active}.
Furthermore, we will call a sampling scheme \emph{model-based}, if the derivation of its sampling criterion involves a model of the function to infer.
Whether this particular model is \emph{parametric} or \emph{nonparametric}, it is furthermore reasonable to differentiate between a parametric or nonparametric sampling scheme.
In contrast, we regard any sampling scheme that is not model-based as \emph{model-free}.
\begin{myRemark}
The two properties, whether a sampling scheme is (un)supervised or model-based/free, are complementary in the sense that we find sampling schemes with an arbitrary combination of characteristics of these properties.
\end{myRemark}
We will now address those aspects that are relevant to classify our algorithmic proposal of a novel active learning framework for regression.
Clearly, every category has its share of the domain of learning problems where it performs best (see \secref{subsec:relatedWork_ALregression} for further in-depth discussion).\smallskip

\noindent{\bf Properties of sampling schemes:} We call a sampling scheme
\begin{itemize}
    \item \textbf{optimal}, if it optimizes some risk with respect to a prediction model that is deployed on the learning task.
    \item \textbf{robust}, if its sampling criterion imposes at most mild assumptions on the regularity of the labels.
    \item \textbf{model-agnostic}, if the performance of an arbitrary, but reasonable model with the acquired training set is not worse (ideally better) than using a random test sample.
\end{itemize}
An \optimal sampling scheme allows for the maximal performance gain, as long as we can fix the final prediction model in advance. Note that the \optimal sampling scheme may strongly depend on the prediction model.

In real-world scenarios we often face the situation that domain knowledge is scarce , and thus we are not aware about the regularity of the problem and committing to a final model might be premature.
This means that practitioners will prefer a consistent, yet moderate performance increase provided by a \robust sampling scheme rather than risking to have overestimated the regularity of the problem that an \optimal sampling scheme may assume.
Under violation of such assumptions, the quality of the so acquired training set may deteriorate below \randomTestSampling.
On the other hand, the state-of-the-art for such a scenario is rapidly evolving and so, as also noted by \cite{settles2010active}, it is advantageous if an acquired training set remains meaningful after model adaption. In this regard, it is desirable for a sampling scheme to be \modelagnostic.

In this work, our goal is to find a sampling scheme that is simultaneously \optimal, \robust and \modelagnostic. As we will discuss in \secref{subsec:discussionALProperties}, such a sampling scheme must necessarily be supervised and nonparametric.
We will base our active learning framework on \emph{local polynomial smoothing} (\LPS) (see, for example, \cite{cleveland1988locally}), a nonparametric model class with minimal regularity assumptions on the labels. We consider \LPS as almost model-free.

We will now outline our main contribution. Intuitively, we consider the best achievable \MISE within the \LPS model class as our objective which we then aim to minimize with respect to the training distribution. Ultimately, our active learning strategy is to sample from this training distribution.\smallskip

Let $f$ be the target regression function that we want to infer from noisy observations $y_i^{} = f(x_i^{}) + \varepsilon_i^{}$, 
where $\varepsilon_i^{}$ is independently drawn from a distribution with mean $\E[\varepsilon_i^{}] = 0$ and local noise variance $\Var[\varepsilon_i^{}] = v(x_i^{})$,
and $x_i^{} \sim p$ are i.i.d.~samples according to a training probability density $p$ defined on $\inputSpace\subset\reals{d}$.
Furthermore let $k^\bandwidth(x,x') := \abs{\bandwidth}^{-1} k(\pnorm{\bandwidth^{-1}(x-x')}{})$ be a \emph{radial basis function} (\RBF) kernel, where the positive definite \emph{bandwidth} matrix parameter $\bandwidth \in \posDefSet{d}$ controls the degree of localization.
Here, we denote by $\abs{M}$ the \emph{determinant} of the square matrix $M \in \reals{d\times d}$.

Let $\mathcal{P}_{\LPSorder}(\reals{d})$ be the space of the real polynomial mappings $\fctn{\mathfrak{p}}{\reals{d}}{\R}$ up to order \LPSorder.
For a training set $\bm{X}_n^{} = \left\{x_i^{}\right\}_{i=1}^n \in \inputSpace^n_{}$ of size $n$, the \emph{\LPS predictor} of order \LPSorder is given by

\begin{align}
 \label{eq:lpsPredictionGeneral}
 \predictorLPS{\LPSorder}{\bandwidth}(x) &= \mathfrak{p}_{\LPSorder,\bandwidth,x}^*(0),
 \mbox{ where }\\
 \notag
 \mathfrak{p}_{\LPSorder,\bandwidth,x}^* &= \argmin\limits_{\mathfrak{p} \in \mathcal{P}_{\LPSorder}(\reals{d})} \mySum{i=1}{n}k^{\bandwidth}(x_i^{},x)\left(y_i^{}-\mathfrak{p}(x_i^{}-x)\right)^2.
\end{align}
Here, $\mathfrak{p}_{\LPSorder,\bandwidth,x}^*$ is the optimal local polynomial approximation of $f$ around $x$. 	
In particular, $\mathfrak{p}_{\LPSorder,\bandwidth,x}^*(0)$ gives the approximate function value at $x$.

The simplicity of the \LPS formulation allows for rich analysis as amongst others demonstrated
by \cite{fan1992variable,ruppert1994} in the special case of \emph{local linear smoothing} (\LLS), where $\LPSorder = 1$, and in the general case by \cite{fan1997local,masry1996multivariate,masry1997multivariate}.
Let $\SigmaSpace \subseteq \posDefSet{d}$ be a candidates set of positive definite bandwidth matrices.
Since the prediction \eqref{eq:lpsPredictionGeneral} of \LPS involves solving an individual problem for each test point $x\in\inputSpace$, the bandwidth $\bandwidth \in \SigmaSpace$ can be chosen individually in $x$ without affecting the prediction of any other instance $x' \neq x$.
Thus, given a training set $\bm{X}_{n}$ and a performance measure such as the \emph{conditional mean squared error}
\begin{align*}
\MSE_{\LPSorder}^{}\left(x, \bandwidth | \bm{X}_{n}\right) = \E \left[(\predictorLPS{\LPSorder}{\bandwidth}(x) - f(x))^2 | \bm{X}_n^{}\right],
\end{align*}
we are free to tune $\bandwidth_x \in \SigmaSpace$ in $x$ such that, ideally,
\begin{align}
\label{eq:minimizingBandwidth}
\MSE_{\LPSorder}^{}\left(x, \bandwidth_x | \bm{X}_{n}\right) = \textstyle \inf_{\bandwidth\in\SigmaSpace}\MSE_{\LPSorder}^{}\left(x, \bandwidth | \bm{X}_{n}\right).
\end{align}
Our criterion for optimization of $\bm{X}_{n}$ is then as follows: 
Given a test density $q$ on \inputSpace and a training set $\bm{X}_n^{}$, we may define the optimal \emph{mean integrated squared error} by
\begin{align}
 \label{eq:ALobjective}
 \MISE_{\LPSorder}^{}\left(q| \bm{X}_n^{}\right) = \textstyle \mathop{\mathlarger{\int}}_{\hspace*{-5pt}\inputSpace} \inf_{\bandwidth\in\SigmaSpace} \MSE_{\LPSorder}^{}\left(x, \bandwidth | \bm{X}_{n}\right) q(x) dx.
\end{align}
Note that the infimum over $\Sigma \in \mathcal{S}$ is taken for each $x$ before integration over $x$.
As our ultimate goal, we therefore seek to tune the training set choice $\bm{X}_n^*$ such that
\begin{align}
 \label{eq:optimalTrainingSet}
 \bm{X}_n^* \approx \textstyle \inf_{\bm{X}_n^{} \in \inputSpace^n_{} }\MISE_{\LPSorder}^{}\left(q| \bm{X}_n^{}\right).
\end{align}

The main idea of our theory is to express the objective \eqref{eq:ALobjective} asymptotically as a function of the training density $p$, the local noise level $v$, the training size $n$, the test density $q$ and a measure of \emph{local function complexity} (\LFC) that concentrates local information about $f$ into a scalar value in a natural way.
Since the dependence of the \MISE on $p$ is given explicitly in the form we will derive, we will be able to analytically optimize our objective \eqref{eq:ALobjective} with respect to the training density.
In other words, we will obtain a training density $\pOptLPS{\LPSorder}{n}$ that is asymptotically optimal, and our proposed active learning approach samples training data
$\bm{X}_n^* \sim \pOptLPS{\LPSorder}{n}$.
As we will demonstrate in \secref{subsec:ALPropertiesOfOptimalSampling}, this sampling scheme encompasses all desired properties. That is, it is \optimal, \robust and \modelagnostic. Additionally it turns out to be \stationary and \interpretable as well -- two properties, we will also define and discuss in \secref{sec:discussionRelatedWork}. While stationarity enables batch sampling in a natural way, interpretability gives access to a representation of the sampling scheme that is comprehensible by humans.

In order to be able to express the objective \eqref{eq:ALobjective} in this manner, a fundamental step is to guarantee the existence and understand the asymptotics of $\bandwidth_x$ in the sense of \eqref{eq:minimizingBandwidth}. We will refer to $\bandwidth_x$ as a locally optimal bandwidth (\LOB) in $x\in\inputSpace$ in the following.
In case of the isotropic bandwidth candidate set $\SigmaSpace = \condset{\sigma\idMatrix{d}}{\sigma > 0}$, $\bandwidth_x$ is unique,
and there exist known results on its asymptotic behavior \citep{fan1992variable,ruppert1994,fan1997local,masry1996multivariate,masry1997multivariate} as well as its estimation \citep{zhang2011kernel}.

When considering a non-isotropic bandwidth candidate set instead, $\bandwidth_x$ will typically not be unique, and its asymptotics behaves differently. 
Accordingly, care needs to be taken when trying to generalize the results from the isotropic to the non-isotropic case.
We provide first results for an extension of our theory to the non-isotropic case in \appref{sec:nonIsotropic}.
Yet, as we require an estimate of general, positive definite \LOB in practice -- which is currently an open research question for the non-isotropic case -- we are only able to give an initial theory to the non-isotropic case in \appref{sec:nonIsotropic}; this theory will immediately become applicable once an estimate to non-isotropic \LOB will come to existence. In the remainder of this work, we therefore focus on the isotropic case.

We will begin with an introduction into previous research, followed by the derivation of our main theorem in \secref{sec:theoryIsotropicOptimalSampling}. 
We provide implementation details in \secref{sec:practicalConsiderations} and discuss active learning properties and related work in \secref{sec:discussionRelatedWork}. Then we demonstrate the capabilities of our proposed framework in experiments on toy-data in \secref{sec:experiments}: We show its benefits in settings of inhomogeneous complexity and heteroscedasticity and compare favorably to two related nonparametric active learning approaches that were built to work on the respective dataset.
Finally, we conclude in \secref{sec:conclusion}.

\section{Analyzing Optimal Training in the Isotropic Case}
\label{sec:theoryIsotropicOptimalSampling}
In the following, we denote by $\vect(S)\in\reals{N}$ the arbitrarily, but fixed ordered vectorization of a finite set $S$ of cardinality $N = \abs{S}$. While for our purpose the order of this vectorization does not matter, it should be applied consistently.
For $A\subset\reals{a}, B\subset\reals{b}$ and $c \geq 0$ let $\diffableFunctions{A, B}{c}$ be the set of $c$-times continuously differentiable functions $\fctn{f}{A}{B}$. As a shorthand let $\diffableFunctions{A}{c} := \diffableFunctions{A,\R}{c}$.
For $f \in \diffableFunctions{A, B}{c}$ and $l \leq c$ let $D^l_f(x)$ the tensor of $l$-th order partial derivatives of $f$ in $x\in A$.
Finally, we denote by $X_n = \convergenceInProbability{}{a_n}$ the convergence in probability of a random sequence $X_n$, meaning that for all $\varepsilon > 0$ it is $\Prob(\abs{X_n/a_n} \geq \varepsilon) \rightarrow 0$ as $n\rightarrow\infty$, and by 
$X_n = \upperBoundedInProbability{}{a_n}$ the stochastic boundedness, meaning that for all $\varepsilon > 0$ there exists $M>0$ and $N\in\N$ such that $\Prob(\abs{X_n/a_n} > M) < \varepsilon, \forall n > N$.
A list of frequently used abbreviations and mathematical symbols can be found in \appref{sec:nomenclature}.

If $\bandwidth_x$ in \eqref{eq:minimizingBandwidth} uniquely exists for all $x$, such that for all $\bandwidth \in \SigmaSpace$ with $\bandwidth \neq \bandwidth_x$ it is
$\MSE_{\LPSorder}^{}\left(x, \bandwidth_x | \bm{X}_{n}\right) < \MSE_{\LPSorder}^{}\left(x, \bandwidth | \bm{X}_{n}\right)$, we can define the \emph{locally optimal bandwidth function} (\LOB) as
\begin{align}
 \label{eq:LOBDefinition}
 \LOBfunctionOfLPS{\LPSorder}{n}(x)
  = \textstyle 
  \argmin_{\bandwidth \in \SigmaSpace} \MSE_{\LPSorder}^{}\left(x, \bandwidth | \bm{X}_{n}\right).
\end{align}
In this case, we are also able to define the \emph{oracle local kernel regressor} by
\begin{align}
 \label{def:fOpt}
 \fOpt{\LPSorder}(x) &= \predictorLPS{\LPSorder}{\LOBfunctionOfLPS{\LPSorder}{n}(x)}(x).
\end{align}
\begin{myRemark}
For any regression model $\widehat{f}$ we can define the \MISE as
\begin{align*}
 \textstyle\MISE\left(q,\widehat{f}| \bm{X}_{n}\right) &= \textstyle \mathop{\mathlarger{\int}}_{\hspace*{-5pt}\inputSpace}\E \left[(\widehat{f}(x) - f(x))^2 | \bm{X}_n^{}\right] q(x) dx.
\end{align*}
Accordingly, we can write $\textstyle\MISE_{\LPSorder}^{}\left(q| \bm{X}_n^{}\right) = \MISE\left(q,\fOpt{\LPSorder}| \bm{X}_{n}\right)$.
\end{myRemark}
There are known results that guarantee \LOB in \eqref{eq:LOBDefinition} to be well-defined, which means that the minimizer exists and is unique: We refer, for example, to the work of \cite{masry1996multivariate,masry1997multivariate} in the case of isotropic bandwidth candidates $\SigmaSpace = \condset{\sigma\idMatrix{d}}{\sigma > 0}$, and to \cite{fan1997local} in the general case $\SigmaSpace = \posDefSet{d}$ for $\LPSorder = 1$. 
All these results rely on the antagonizing effect of bias and variance: Generally, we can decompose the conditional \MSE in $x\in\inputSpace$ (see \cite{geman1992neural,bishop2006pattern}) according to
\begin{align}
\label{eq:trueBiasVarianceDecomposition}
\textstyle\MSE_{\LPSorder}^{}\left(x, \bandwidth | \bm{X}_{n}\right) &=\textstyle \Bias_{\LPSorder}^{}\left(x, \bandwidth | \bm{X}_{n}\right)^2 + \Var_{\LPSorder}^{}\left(x, \bandwidth | \bm{X}_{n}\right), \qquad\text{where}\\
\textstyle\Bias_{\LPSorder}^{}\left(x, \bandwidth | \bm{X}_{n}\right) &= f(x) - \E \left[\predictorLPS{\LPSorder}{\bandwidth}(x) | \bm{X}_n^{}\right]\qquad\text{and}\nonumber\\
\textstyle\Var_{\LPSorder}^{}\left(x, \bandwidth | \bm{X}_{n}\right) &= \E \left[(\predictorLPS{\LPSorder}{\bandwidth}(x) - \E \predictorLPS{\LPSorder}{\bandwidth}(x))^2 | \bm{X}_n^{}\right]\nonumber
\end{align}
are the bias- and variance-related error terms.
Now,
as $\pnorm{\bandwidth}{}\!\rightarrow 0$,
$\Bias_{\LPSorder}^{}\left(x, \bandwidth | \bm{X}_{n}\right)$ decreases while
$\Var_{\LPSorder}^{}\left(x, \bandwidth | \bm{X}_{n}\right)$ increases; On the other hand, as $\abs{\bandwidth}\rightarrow \infty$, $\Bias_{\LPSorder}^{}\left(x, \bandwidth | \bm{X}_{n}\right)$ increases while $\Var_{\LPSorder}^{}\left(x, \bandwidth | \bm{X}_{n}\right)$ decreases.
It is known that there is a finite bandwidth that trades off both terms in the optimal way \citep{silverman1986density,wand1994kernel}. After identifying the leading terms of bias and variance, 
an asymptotic closed-form solution to \LOBfunctionOfLPS{\LPSorder}{n} can be constructed explicitly in the isotropic case for arbitrary.

In this paper, we will focus on the case where the bandwidth space is restricted to be isotropic.
In this case, we can elaborate our framework from the theoretical point-of-view by making use of the results of \cite{masry1996multivariate,masry1997multivariate} on the unique existence and asymptotic behavior of \LOB under mild assumptions. Furthermore, we can estimate \LOB \citep{zhang2011kernel} by Lepski's method \citep{lepski1991problem,lepski1997optimal}.

The approach of using the asymptotic closed-form solution to \LOB to prove our theory however cannot be generalized to the non-isotropic case without imposing unrealistically strong assumptions. We will discuss these issues that do arise in the non-isotropic bandwidth case -- when in particular not relying on this explicit construction -- and provide solutions that still make our theory hold under mild conditions in \appref{sec:nonIsotropic}.

\subsection{Preliminaries}
\label{subsec:preliminaries}
Let us begin by making the \LPS predictor from \eqref{eq:lpsPredictionGeneral} explicit:
Define the vector of distinct $j$-th order monomials of $u\in\reals{d}$ by
\begin{align}
\label{eq:distinctMonomialsVector}
m_{j}^{}(u) =\textstyle\vect\left(\condset{\prod_{l=1}^{d} u_l^{\bm{j}_l}}{\bm{j}:\sum_{k=1}^{d}\bm{j}_k = j}\right).
\end{align}
Note that $m_{j}^{}(u) \in \reals{N_j}$ for $N_j = \binom{d-1+j}{d-1}$.
Now, we can write any polynomial $\mathfrak{p} \in \mathcal{P}_{\LPSorder}(\reals{d})$ of order \LPSorder as $\mathfrak{p}(z) = \mySum{j=0}{\LPSorder} \beta_j^\top m_{j}^{}(z)$, where
the mapping between $\mathfrak{p}$ and its monomial coefficients $\beta_j \in \reals{N_j}$ is bijective. For convenience, we will identify a polynomial with its coefficients and use both expressions interchangeably.
In this regard, let $\displaystyle\left(\beta_{\LPSorder,x}^{\bandwidth}\right)^*_0, \ldots, \left(\beta_{\LPSorder,x}^{\bandwidth}\right)^*_{\LPSorder}$ be the coefficients of the optimal polynomial $\mathfrak{p}_{\LPSorder,\bandwidth,x}^*$.
Noting that $\mathfrak{p}(0) = \beta_0 \in \R$, the optimization in \eqref{eq:lpsPredictionGeneral} is expressed as
\begin{align}
 \label{eq:lpsPrediction}
 \predictorLPS{\LPSorder}{\bandwidth}(x) &= \left(\beta_{\LPSorder,x}^{\bandwidth}\right)^*_0,
\mbox{ where }\\
\left(\beta_{\LPSorder,x}^{\bandwidth}\right)^*_0, \ldots, \left(\beta_{\LPSorder,x}^{\bandwidth}\right)^*_{\LPSorder}
&= 
\argmin\limits_{\beta_0, \ldots, \beta_{\LPSorder}}
\hspace*{-1pt}\sum_{i=1}^{n}k^{\bandwidth}(x_i^{},x)\left[y_i^{} - \mySum{j=0}{\LPSorder} \beta_{j}^\top m_{j}^{}(x_i^{}-x)\right]^2.
\notag
\end{align}
Let us aggregate the vector of distinct monomials up to order \LPSorder as
\begin{align}
\label{eq:distinctMonomialsVectorUpToOrderP}
M_{\LPSorder}(u) = [m_{0}^{}(u)^\top, \ldots, m_{\LPSorder}^{}(u)^\top]^\top.
\end{align}
The following closed-form solution for \eqref{eq:lpsPrediction} is known (see, for example, \cite{zhang2011kernel}):
Letting $X_{\LPSorder}^{}(x) = [M_{\LPSorder}(x_1-x),\ldots,M_{\LPSorder}(x_n-x)]^\top$ and $W_x^\bandwidth = \diag\left([k^{\bandwidth}(x_1^{},x),\ldots,k^{\bandwidth}(x_n^{},x)]\right)$, it is
\begin{align*}
 \predictorLPS{\LPSorder}{\bandwidth}(x) &= A_{\LPSorder}^\bandwidth(x) Y_n,
\mbox{ where }\\
A_{\LPSorder}^\bandwidth(x) &= e_1^\top\left(X_{\LPSorder}^{}(x)^\top W_x^\bandwidth X_{\LPSorder}^{}(x)\right)^{-1}X_{\LPSorder}^{}(x)^\top W_x^\bandwidth.
\end{align*}
By substituting the equation above into Eq.~\eqref{eq:trueBiasVarianceDecomposition}, we obtain
\begin{align}
\textstyle\MSE_{\LPSorder}^{}\left(x, \bandwidth | \bm{X}_{n}\right) &=\textstyle \Bias_{\LPSorder}^{}\left(x, \bandwidth | \bm{X}_{n}\right)^2 + \Var_{\LPSorder}^{}\left(x, \bandwidth | \bm{X}_{n}\right), \qquad\text{where}\nonumber\\
\label{eq:trueFiniteBias}
\textstyle\Bias_{\LPSorder}^{}\left(x, \bandwidth | \bm{X}_{n}\right) &= f(x) - A_{\LPSorder}^\bandwidth(x) f(\bm{X}_n^{})\qquad\text{and}\\
\label{eq:trueFiniteVariance}
\textstyle\Var_{\LPSorder}^{}\left(x, \bandwidth | \bm{X}_{n}\right) &= A_{\LPSorder}^\bandwidth(x) \diag(v(\bm{X}_n^{})) A_{\LPSorder}^\bandwidth(x)^\top.
\end{align}
From here on, we assume for simplicity that the input space $\inputSpace \subset \reals{d}$ is compact, which real-world data will typically fulfill\footnote{The theory may be extended to more general cases by imposing a weaker condition than compactness, such as bounded derivatives and a fast enough decay of the kernel.}.
As mentioned above, we restrict ourselves to the case of isotropic bandwidth candidates $\SigmaSpace = \condset{\sigma\idMatrix{d}}{\sigma > 0}$.
Under adequate regularity conditions, explicit asymptotic formulations for \LOB are known (see, for example, \cite{fan1992variable,fan1997local} and \cite{masry1996multivariate,masry1997multivariate}):
Using \eqref{eq:distinctMonomialsVectorUpToOrderP}, define by
\[\textstyle\bm{M}_{\LPSorder}^{} = \mathop{\mathlarger{\int}}_{\hspace*{-5pt}\inputSpace} M_{\LPSorder}(u)M_{\LPSorder}(u)^\top k(u)du,\quad\text{and}\quad\bm{\Gamma}_{\LPSorder}^{} = \mathop{\mathlarger{\int}}_{\hspace*{-5pt}\inputSpace} M_{\LPSorder}(u)M_{\LPSorder}(u)^\top k^2(u)du\]
the first and second moment matrix of the kernel, by
\[\textstyle \bm{D}_{\LPSorder}^{}(x) = \condset{[\prod_{k=1}^{d}\bm{j}_k!]^{-1}\frac{\partial^{\LPSorder+1}}{\prod_{l=1}^{d} \partial^{\bm{j}_l} x_l}f(x)}{\bm{j}:\mySum{k=1}{d}\bm{j}_k =\LPSorder+1},\]
the vector of distinct partial derivatives of $f$ of order $\LPSorder+1$ in $x$, and
\[\bm{B}_{\LPSorder}^{} = \textstyle \mathop{\mathlarger{\int}}_{\hspace*{-5pt}\inputSpace} M_{\LPSorder}(u)m_{\LPSorder+1}^{}(u)^\top k(u)du.
\]

\begin{mythm}[\cite{masry1996multivariate,masry1997multivariate}]
\label{thm:asymptoticBiasVarianceForLPS}
Let $k^\bandwidth(x,x') = \abs{\bandwidth}^{-1}k(\bandwidth^{-1}(x-x'))$ be a not necessarily spherically symmetric kernel.
Let $h_n^{}\rightarrow 0, nh_n^d\rightarrow \infty$ as $n\rightarrow\infty$ and $x$ be a fixed point in the interior of $\condset{x}{p(x) > 0}$.
Furthermore let $f \in \diffableFunctions{\inputSpace}{\LPSorder+1}$, $p \in \diffableFunctions{\inputSpace}{1}$ and $v \in \diffableFunctions{\inputSpace}{0}$ with $\inf_{x\in\inputSpace} p(x) > 0$ and $\inf_{x\in\inputSpace} v(x) > 0$.
The bias and variance of \LPS of order \LPSorder can asymptotically be expressed as
\[\textstyle\Bias_{\LPSorder}^{}\left(x, h_n^{}\idMatrix{d} | \bm{X}_{n}\right) = \biasLPS{h_n^{}\idMatrix{d}}{x}{\LPSorder} + \convergenceInProbability{\noexpand\big}{h_n^{\LPSorder+1}}\] and \[
\textstyle\Var_{\LPSorder}^{}\left(x, h_n^{}\idMatrix{d} | \bm{X}_{n}\right) = \varianceLPS{h_n^{}\idMatrix{d}}{x}{\LPSorder} + \convergenceInProbability{}{n^{-1}h_n^{-d}},\]
where
\begin{align}
 \label{eq:asymptoticLeadingOrderBiasLPS}
\biasLPS{h_n^{}\idMatrix{d}}{x}{\LPSorder} = \textstyle h_n^{\LPSorder+1}e_1^\top\bm{M}_{\LPSorder}^{-1}\bm{B}_{\LPSorder}^{}\bm{D}_{\LPSorder}^{}(x)
\end{align}
is the leading bias-term of order $\LPSorder+1$, and for $\bm{R}_{\LPSorder}^{} = e_1^\top\bm{M}_{\LPSorder}^{-1}\bm{\Gamma}_{\LPSorder}^{}\bm{M}_{\LPSorder}^{-1}e_1$,
\begin{align}
 \label{eq:asymptoticVarianceLPS}
\varianceLPS{h_n^{}\idMatrix{d}}{x}{\LPSorder} = \bm{R}_{\LPSorder}^{}\frac{v(x)}{p(x)nh_n^{d}}.
\end{align}
\end{mythm}
In \thmref{thm:asymptoticBiasVarianceForLPS} and in the following, we call the kernel $k$ \emph{spherically symmetric}, if its evaluation in $x\in\inputSpace$ only depends on $x$ through its norm $\pnorm{x}{}$. That is, we can rewrite $k(x) = \tilde{k}(\pnorm{x}{})$ for some adequate function $\tilde{k}$.
\begin{myRemark}
\label{rem:symmetricKernelsEvenPolyOrder}
If the kernel $k$ is spherically symmetric and \LPSorder is even, then $\biasLPS{h_n^{}\idMatrix{d}}{x}{\LPSorder} = 0$ and $\varianceLPS{h_n^{}\idMatrix{d}}{x}{\LPSorder} = \varianceLPS{h_n^{}\idMatrix{d}}{x}{\LPSorder+1}$, since\linebreak $\bm{R}_{\LPSorder}^{} = \bm{R}_{\LPSorder+1}^{}$.
In this case, if $f \in \diffableFunctions{\inputSpace}{\LPSorder+2}$, the leading bias term will be of the form $b_{\LPSorder}^{}(x, h_n^{}\idMatrix{d}) = \biasLPS{h_n^{}\idMatrix{d}}{x}{\LPSorder+1} + \upperBoundedInProbability{\noexpand\big}{h_n^{\LPSorder+2}}$.
Here, it is common belief that $b_{\LPSorder}^{}(x, h_n^{}\idMatrix{d}) > \biasLPS{h_n^{}\idMatrix{d}}{x}{\LPSorder+1}$ (see, for example, \cite{fan1997local}).
Thus, as we have to require $f \in \diffableFunctions{\inputSpace}{\LPSorder+2}$ anyhow, and the variance does not grow when moving from the \LPS model of order \LPSorder to $\LPSorder+1$, we expect a better performance when using the \LPS model of (odd) order $\LPSorder+1$. Note that the performance increases only by a constant factor and not in convergence rate, when applying \LPS of order $\LPSorder+1$ instead of \LPSorder.
More importantly in the context of our work, the \LPS model of odd order $\LPSorder+1$ is \emph{design-adaptive} \citep{fan1992design}. That is,
the leading bias-term $\biasLPS{h_n^{}\idMatrix{d}}{x}{\LPSorder+1}$ does not depend on the derivatives of the training distribution.
In contrast, for example, the bias of the \emph{Nadaraya-Watson estimator} (\LPS of order $\LPSorder = 0$) is
\[b_{0}^{}(x, h_n^{}\idMatrix{d}) = \biasLPS{h_n^{}\idMatrix{d}}{x}{1} + h_n^2\mu_2\frac{1}{p(x)}\mySum{i=1}{d}\frac{\partial f}{\partial x_i}(x)\frac{\partial p}{\partial x_i}(x),\]
where $\mu_2 = \int u^2k(u)du$. Here, the bias obviously depends on the derivative of $p$.
In this case, as also noted by \cite{fan1992design}, the model has problems to adapt to highly clustered distributions, where $\abs{\frac{1}{p(x)}\frac{\partial p}{\partial x_i}(x)}$ is large.
\end{myRemark}
\begin{myRemark}
Note that we could also use a kernel of higher order $\nu > 2$ to adapt
optimally to functions of higher smoothness $f \in \diffableFunctions{\inputSpace}{\nu}$ without the need to increase the polynomial order $\LPSorder$ of \LPS. However, a higher-order kernel loses consistency as it necessarily takes negative values \citep{gyorfi2002distribution}, and it loses design-adaptivity for $Q < \nu-1$.
\end{myRemark}

In the upcoming derivation of our theory, we rely on design-adaptivity to keep the optimization of the conditional $\MSE$ with respect to the training distribution simple. Design-adaptivity is guaranteed, if the leading order bias term \eqref{eq:asymptoticLeadingOrderBiasLPS} does not vanish. 
In the light of \remref{rem:symmetricKernelsEvenPolyOrder}, it is reasonable to choose a spherically symmetric kernel for odd order $\LPSorder$.
For even order $\LPSorder$, for example, a kernel of higher order $\nu = \LPSorder + 1$ could be applied.

\begin{restatable}{mycor}{isotropicLOBforLPS}
\label{cor:isotropicLOBforLPS}
Let the kernel $k$ be spherically symmetric and $\LPSorder \in \N$ odd.
Under the conditions of \thmref{thm:asymptoticBiasVarianceForLPS}, when searching for \LOB in the space of isotropic candidates $\SigmaSpace = \condset{\sigma\idMatrix{d}}{\sigma > 0}$, and if all, $p(x), v(x), \biasLPS{\idMatrix{d}}{x}{\LPSorder} \neq 0$ do not vanish, then asymptotically it is
\begin{align}
 \nonumber
 &\LOBfunctionOfLPS{\LPSorder}{n}(x) =\textstyle \sigmaAsympLPS{\LPSorder}(x)\idMatrix{d} + \convergenceInProbability{}{n^{-\frac{1}{2(\LPSorder+1)+d}}},\quad\text{for}\\
 \label{eq:asymptoticIsotropicLOBofLPS}
 &\sigmaAsympLPS{\LPSorder}(x) = C_{\LPSorder}^{}\left[\frac{v(x)}{p(x)n}\right]^\frac{1}{2(\LPSorder+1)+d}\biasLPS{\idMatrix{d}}{x}{\LPSorder}^{-\frac{2}{2(\LPSorder+1)+d}},
\end{align}
where $C_{\LPSorder}^{} = \left[d \bm{R}_{\LPSorder}^{}\big/\left(2(\LPSorder+1)\right)\right]^\frac{1}{2(\LPSorder+1)+d}$.
\end{restatable}
\begin{myRemark}
\corollaryref{cor:isotropicLOBforLPS} and \ref{cor:biasVarianceBalance} will also hold for even order $\LPSorder$ by replacing $\biasLPS{\idMatrix{d}}{x}{\LPSorder}$ with $b_{\LPSorder}^{}(x, \idMatrix{d})$ in the sense of \remref{rem:symmetricKernelsEvenPolyOrder} and $\LPSorder$ replaced by $\LPSorder' = \LPSorder + (\LPSorder-1 \bmod 2)$ elsewhere.
\end{myRemark}
The results of this section provide explicit forms of the asymptotic behavior of bias, variance and \LOB that we can make use of in our work.
From now on, we will restrict to the case of an odd \LPS order $\LPSorder$, where we can apply a spherically symmetric kernel $k$. As discussed above, dealing with the case of even order $\LPSorder$ is possible, yet unimportant in practice. In particular, for even order $\LPSorder$ we can instead apply \LPS of order ($\LPSorder+1$), which is only marginally more complex and therefore poses no computational bottleneck. 

\subsection{Isotropic Optimal Sampling} 
\label{subsec:isotropicOptimalSampling}
We begin by simplifying our active learning objective \eqref{eq:ALobjective} by rewriting $\textstyle\MSE_{\LPSorder}^{}\left(x, \LOBfunctionOfLPS{\LPSorder}{n}(x) | \bm{X}_{n}\right)$ solely in terms of $\textstyle\Var_{\LPSorder}^{}\left(x, \LOBfunctionOfLPS{\LPSorder}{n}(x) | \bm{X}_{n}\right)${, which intuitively can be done because the \LOB, as the optimal trade-off between bias and variance, will balance the error contribution of both components to the conditional \MSE.}
We can support this intuition formally as follows (see also \cite{zhang2011kernel}):

\begin{restatable}{mycor}{biasVarianceBalance}
\label{cor:biasVarianceBalance}
Let the kernel $k$ be spherically symmetric and $\LPSorder \in \N$ odd. When applying $\LOBfunctionOfLPS{\LPSorder}{n}(x)$ as the bandwidth for prediction in $x$, the conditional bias- and variance-related error components are asymptotically proportional over \inputSpace. That is, for all $x\in\inputSpace$ it is
\[\textstyle\Bias_{\LPSorder}^{}\left(x, \LOBfunctionOfLPS{\LPSorder}{n}(x) | \bm{X}_{n}\right)^2 = \frac{d}{2(\LPSorder+1)}\varianceLPS{\LOBfunctionOfLPS{\LPSorder}{n}(x)}{x}{\LPSorder} + \convergenceInProbability{}{n^{-\frac{2(\LPSorder+1)}{2(\LPSorder+1)+d}}}.\]
Hence, the conditional \MSE can asymptotically be expressed as
\begin{align*}
\textstyle\MSE_{\LPSorder}^{}\left(x, \LOBfunctionOfLPS{\LPSorder}{n}(x) | \bm{X}_{n}\right) = \frac{2(\LPSorder+1)+d}{2(\LPSorder+1)}\varianceLPS{\LOBfunctionOfLPS{\LPSorder}{n}(x)}{x}{\LPSorder} + \convergenceInProbability{}{n^{-\frac{2(\LPSorder+1)}{2(\LPSorder+1)+d}}}.
\end{align*}
\end{restatable}
\begin{proof}
Using \thmref{thm:asymptoticBiasVarianceForLPS} and \corollaryref{cor:isotropicLOBforLPS}, it is
\[\textstyle\Bias_{\LPSorder}^{}\left(x, \LOBfunctionOfLPS{\LPSorder}{n}(x) | \bm{X}_{n}\right) = \sigmaAsympLPS{\LPSorder}(x)^{\LPSorder+1}\biasLPS{\idMatrix{d}}{x}{\LPSorder} + \convergenceInProbability{}{n^{-\frac{\LPSorder+1}{2(\LPSorder+1)+d}}},\]
where it was $\sigmaAsympLPS{\LPSorder}(x) = C_{\LPSorder}^{}\left[v(x)\big/(p(x)n)\right]^\frac{1}{2(\LPSorder+1)+d}\biasLPS{\idMatrix{d}}{x}{\LPSorder}^{-\frac{2}{2(\LPSorder+1)+d}}$.
Furthermore, with \eqref{eq:asymptoticVarianceLPS},
\begin{align*}
\varianceLPS{\LOBfunctionOfLPS{\LPSorder}{n}(x)}{x}{\LPSorder} =  \bm{R}_{\LPSorder}^{}\frac{v(x)}{p(x)n\sigmaAsympLPS{\LPSorder}(x)^{d}} + \convergenceInProbability{}{n^{-\frac{2(\LPSorder+1)}{2(\LPSorder+1)+d}}}.
\end{align*}
Therefore we can rewrite
\begin{align*}
&\textstyle\Bias_{\LPSorder}^{}\left(x, \LOBfunctionOfLPS{\LPSorder}{n}(x) | \bm{X}_{n}\right)^2 + \convergenceInProbability{}{n^{-\frac{2(\LPSorder+1)}{2(\LPSorder+1)+d}}} = \sigmaAsympLPS{\LPSorder}(x)^{2(\LPSorder+1)}\biasLPS{\idMatrix{d}}{x}{\LPSorder}^2\\
&= C_{\LPSorder}^{2(\LPSorder+1)}\left[\frac{v(x)}{p(x)n}\right]^\frac{2(\LPSorder+1)}{2(\LPSorder+1)+d}\biasLPS{\idMatrix{d}}{x}{\LPSorder}^\frac{2d}{2(\LPSorder+1)+d}\\
&= C_{\LPSorder}^{2(\LPSorder+1)+d}\left[\frac{v(x)}{p(x)n}\right]\sigmaAsympLPS{\LPSorder}(x)^{-d}
= \frac{d}{2(\LPSorder+1)} \varianceLPS{\LOBfunctionOfLPS{\LPSorder}{n}(x)}{x}{\LPSorder}.
\end{align*}

Using \thmref{thm:asymptoticBiasVarianceForLPS}, it is
\begin{align*}
 \MSE_{\LPSorder}^{}&\textstyle\left(x, \LOBfunctionOfLPS{\LPSorder}{n}(x) | \bm{X}_{n}\right) = \Bias_{\LPSorder}^{}\left(x, \LOBfunctionOfLPS{\LPSorder}{n}(x) | \bm{X}_{n}\right)^2 + \Var_{\LPSorder}^{}\left(x, \LOBfunctionOfLPS{\LPSorder}{n}(x) | \bm{X}_{n}\right)\\
 &= \frac{2(\LPSorder+1)+d}{2(\LPSorder+1)}\varianceLPS{\LOBfunctionOfLPS{\LPSorder}{n}(x)}{x}{\LPSorder} + \convergenceInProbability{}{n^{-\frac{2(\LPSorder+1)}{2(\LPSorder+1)+d}}}.
 \qedAtBottomLine
\end{align*}
\end{proof}
This result will become quite handy later on, as we can get rid of the cumbersome bias term, when estimating the conditional \MSE at $\LOBfunctionOfLPS{\LPSorder}{n}(x)$ in $x$.

Next, the asymptotic form in \eqref{eq:asymptoticIsotropicLOBofLPS} reveals that \LOB factorizes into a global scaling with respect to sample size $n$ and local scaling components with respect to noise level $v(x)$ and training density $p(x)$. For \LPSorder odd, the last component $\biasLPS{\idMatrix{d}}{x}{\LPSorder}$ solely contains information about the function $f$ to be learnt near $x$. Intuitively, the optimal bandwidth will be locally smaller where the structure of $f$ is locally more complex.
Based on this observation, we define the following:
\begin{restatable}[Isotropic complexity of LPS]{mydef}{isotropicComplexityForLPS}
\label{def:isotropicComplexityForLPS}
For $\LPSorder \in\N$ odd, let $\LOBfunctionOfLPS{\LPSorder}{n}$ be the optimal bandwidth function as defined in \eqref{eq:LOBDefinition}. With the adjusted bandwidth function
\begin{align}
 \label{eq:normalizedLOBLPS}
 \normalizedLOBLPS{\LPSorder}{n}(x) = C_{\LPSorder}^{-1}\left[\frac{v(x)}{p(x)n}\right]^{-\frac{1}{2(\LPSorder+1)+d}}\LOBfunctionOfLPS{\LPSorder}{n}(x),
\end{align}
we define by
\begin{align}
 \label{eq:isotropicComplexityLPS}
 \fctnComplexityLPS{\LPSorder}{n}(x) = \abs{\normalizedLOBLPS{\LPSorder}{n}(x)}^{-1}
\end{align}
the local function complexity (\LFC) of $f$ in $x$ with respect to the \LPS model of order \LPSorder.
\end{restatable}

Essentially, as the reciprocal of \LOB, $\fctnComplexityLPS{\LPSorder}{n}$ grows with increasing complexity of $f$, locally at $x$.
But most importantly, $\fctnComplexityLPS{\LPSorder}{n}$ is asymptotically independent of the global scaling with respect to training size $n$, as well as the local scaling with respect to the training density $p$ and noise level $v$:
Indeed, asymptotically we can write
\begin{align}
 \label{eq:asymptoticIsotropicComplexityLPS}
 \fctnComplexityLPS{\LPSorder}{n}(x) = \fctnComplexityLPS{\LPSorder}{\infty}(x)(1 + \convergenceInProbability{}{1}),\;\;\text{where}\;
 \fctnComplexityLPS{\LPSorder}{\infty}(x) = \biasLPS{\idMatrix{d}}{x}{\LPSorder}^\frac{2d}{2(\LPSorder+1)+d}.
\end{align}
We observe that $\fctnComplexityLPS{\LPSorder}{n}$ is asymptotically continuous, since \fctnComplexityLPS{\LPSorder}{\infty} is continuous via construction.

Combining the balancing property from \corollaryref{cor:biasVarianceBalance} and the asymptotic results on \LOB in \corollaryref{cor:isotropicLOBforLPS} and the variance \eqref{eq:asymptoticVarianceLPS}, we can express $\MSE_{\LPSorder}^{}\left(x, \LOBfunctionOfLPS{\LPSorder}{n}(x) | \bm{X}_{n}\right)$ in terms of the training density $p(x)$, the noise variance $v(x)$, the training size $n$ and the \LFC $\fctnComplexityLPS{\LPSorder}{n}(x)$.
With these preparations, we are now able to state our main result:

Let $q \in \diffableFunctions{\inputSpace, \nonnegativeReal{}}{0}$ be a test density such that
$\mathop{\mathlarger{\int}}_{\hspace*{-5pt}\inputSpace}q(x)dx = 1$.
Since $\LOBfunctionOfLPS{\LPSorder}{n}$ is well-defined, recall from \eqref{def:fOpt} that our active
active learning objective is given by
\[\MISE_{\LPSorder}^{}\left(q| \bm{X}_n^{}\right) = \textstyle \mathop{\mathlarger{\int}}_{\hspace*{-5pt}\inputSpace}  \MSE_{\LPSorder}^{}\left(x, \LOBfunctionOfLPS{\LPSorder}{n}(x) | \bm{X}_{n}\right) q(x) dx.\]

Our goal is now to minimize 
$\MISE_{\LPSorder}^{}\left(q| \bm{X}_n^{}\right)$ with respect to the training set $\bm{X}_{n}$.
As we will show in the following theorem, asymptotically, the optimal $\bm{X}_{n}$ can be expressed as a random sample from the optimal density, which we denote by $\pOptLPS{\LPSorder}{n}$. That is, $\bm{X}_{n} \sim \pOptLPS{\LPSorder}{n}$.
\begin{mythm}
 \label{thm:OptimalSampling}
 Let $v, q \in \diffableFunctions{\inputSpace, \nonnegativeReal{}}{0}$ for a compact input space \inputSpace, where $q$ is a test density such that $\textstyle \mathop{\mathlarger{\int}}_{\hspace*{-5pt}\inputSpace}q(x)dx = 1$. 
 Additionally, assume that $v$ and $q$ are bounded away from zero. That is, $v,q \geq \epsilon$ for some $\epsilon > 0$.
 Let $k$ be a \RBF-kernel with bandwidth parameter space $\SigmaSpace = \condset{\sigma\idMatrix{d}}{\sigma > 0}$. Let $\LPSorder \in \N$ be odd and $f \in \diffableFunctions{\inputSpace}{\LPSorder+1}$ such that $e_1^\top\bm{M}_{\LPSorder}^{-1}\bm{B}_{\LPSorder}^{}\bm{D}_{\LPSorder}^{}(x) \neq 0$, almost everywhere.
 Then the optimal training density for \LPS of order \LPSorder is asymptotically given by
 \begin{align}
 \label{eq:optimalSamplingFiniteLPS}
 \pOptLPS{\LPSorder}{n}(x) \propto \textstyle \left[\fctnComplexityLPS{\LPSorder}{n}(x) q(x)\right]^{\frac{2(\LPSorder+1)+d}{4(\LPSorder+1)+d}}v(x)^{\frac{2(\LPSorder+1)}{4(\LPSorder+1)+d}}(1 + o(1)).
 \end{align}
\end{mythm}
\begin{proof}
To begin with, recall from \corollaryref{cor:isotropicLOBforLPS} that in the isotropic case, \LOB is given by 
\[\LOBfunctionOfLPS{\LPSorder}{n}(x) = \sigmaAsympLPS{\LPSorder}(x)\idMatrix{d} + \convergenceInProbability{}{n^{-\frac{1}{2(\LPSorder+1)+d}}},\]
where $\sigmaAsympLPS{\LPSorder}(x) = C_{\LPSorder}^{}\left[v(x)\big/(p(x)n)\right]^\frac{1}{2(\LPSorder+1)+d}\biasLPS{\idMatrix{d}}{x}{\LPSorder}^{-\frac{2}{2(\LPSorder+1)+d}}$.
According to \corollaryref{cor:biasVarianceBalance}, it is
\begin{align*}
\MSE_{\LPSorder}^{}\left(x, \LOBfunctionOfLPS{\LPSorder}{n}(x) | \bm{X}_{n}\right) = \frac{2(\LPSorder+1)+d}{2(\LPSorder+1)}\varianceLPS{\LOBfunctionOfLPS{\LPSorder}{n}(x)}{x}{\LPSorder} + \convergenceInProbability{\noexpand\big}{n^{-\frac{2(\LPSorder+1)}{2(\LPSorder+1)+d}}}.
\end{align*}
Noting that $\textstyle\abs{\LOBfunctionOfLPS{\LPSorder}{n}(x)} = \sigmaAsympLPS{\LPSorder}(x)^d$, we know from \eqref{eq:asymptoticVarianceLPS} in \thmref{thm:asymptoticBiasVarianceForLPS} that
\begin{align*}
\varianceLPS{\LOBfunctionOfLPS{\LPSorder}{n}(x)}{x}{\LPSorder} = \bm{R}_{\LPSorder}^{}\frac{v(x)}{p(x)n}\abs{\LOBfunctionOfLPS{\LPSorder}{n}(x)}^{-1}.
\end{align*}
Using \defref{def:isotropicComplexityForLPS} and \eqref{eq:asymptoticIsotropicComplexityLPS}, we can therefore write
\begin{align*}
\MSE&_{\LPSorder}^{}\left(x, \LOBfunctionOfLPS{\LPSorder}{n}(x) | \bm{X}_{n}\right)\\
&=
\frac{2(\LPSorder+1)+d}{2(\LPSorder+1)}\bm{R}_{\LPSorder}^{} C_{\LPSorder}^{-d} \left[\frac{v(x)}{p(x)n}\right]^\frac{2(\LPSorder+1)}{2(\LPSorder+1)+d}\abs{\normalizedLOBLPS{\LPSorder}{n}(x)}^{-1} + \convergenceInProbability{}{n^{-\frac{2(\LPSorder+1)}{2(\LPSorder+1)+d}}}\\
&= \bar{C}_{\LPSorder}^{} \left[\frac{v(x)}{p(x)n}\right]^\frac{2(\LPSorder+1)}{2(\LPSorder+1)+d}\fctnComplexityLPS{\LPSorder}{n}(x) + \convergenceInProbability{}{n^{-\frac{2(\LPSorder+1)}{2(\LPSorder+1)+d}}}\\
&= \bar{C}_{\LPSorder}^{} \left[\frac{v(x)}{p(x)n}\right]^\frac{2(\LPSorder+1)}{2(\LPSorder+1)+d}\fctnComplexityLPS{\LPSorder}{\infty}(x) + \convergenceInProbability{}{n^{-\frac{2(\LPSorder+1)}{2(\LPSorder+1)+d}}},
\end{align*}
where we have set $\bar{C}_{\LPSorder}^{} = \frac{2(\LPSorder+1)+d}{2(\LPSorder+1)}\bm{R}_{\LPSorder}^{} C_{\LPSorder}^{-d}$.
Putting this into \eqref{def:fOpt}, we obtain
\begin{align*}
  \MISE_{\LPSorder}^{}\left(q| \bm{X}_n^{}\right) = \textstyle \mathop{\mathlarger{\int}}_{\hspace*{-5pt}\inputSpace}\bar{C}_{\LPSorder}^{} \left[\frac{v(x)}{p(x)n}\right]^\frac{2(\LPSorder+1)}{2(\LPSorder+1)+d}\fctnComplexityLPS{\LPSorder}{\infty}(x)q(x)dx + \convergenceInProbability{}{n^{-\frac{2(\LPSorder+1)}{2(\LPSorder+1)+d}}}.
\end{align*}
Since $v$, $\fctnComplexityLPS{\LPSorder}{\infty}$ and $q$ are continuous on the compact input space \inputSpace, the \emph{Stone–Weierstrass theorem} guarantees
a sequence $g_{\LPSorder}^{n} \in \diffableFunctions{\inputSpace}{1}$ that satisfies
\[\textstyle\sup_{x\in\inputSpace}\abs{g_{\LPSorder}^{n}(x) - \left[\frac{v(x)}{n}\right]^\frac{2(\LPSorder+1)}{2(\LPSorder+1)+d}\fctnComplexityLPS{\LPSorder}{\infty}(x)q(x)} < \frac{1}{n}.\]
Obviously, $g_{\LPSorder}^{n}$ can be chosen independently of $p$. Hence we can write
\begin{align*}
  \MISE_{\LPSorder}^{}\left(q| \bm{X}_n^{}\right) = \textstyle \mathop{\mathlarger{\int}}_{\hspace*{-5pt}\inputSpace}\bar{C}_{\LPSorder}^{} p(x)^{-\frac{2(\LPSorder+1)}{2(\LPSorder+1)+d}}g_{\LPSorder}^{n}(x)dx + \convergenceInProbability{}{n^{-\frac{2(\LPSorder+1)}{2(\LPSorder+1)+d}}}.
\end{align*}

Define the Lagrangian for 
minimizing $\MISE_{\LPSorder}^{}\left(q| \bm{X}_n^{}\right)$ with respect to $p \in \diffableFunctions{\inputSpace,\nonnegativeReal{}}{1}$
under the constraint $\int_\inputSpace p(x)dx = 1$, that is,
\[L(p,\lambda) = \textstyle\mathop{\mathlarger{\int}}_{\hspace*{-5pt}\inputSpace} \underbrace{\bar{C}_{\LPSorder}^{} p(x)^{-\frac{2(\LPSorder+1)}{2(\LPSorder+1)+d}}g_{\LPSorder}^{n}(x) + \lambda p(x)dx - \lambda}_{F(x,p,\lambda)}.\]
We find the optimal training density to be $\pOptLPS{\LPSorder}{n} := p^*$ for the stationary point $(p^*, \lambda^*)$ of $L$ that exhibits minimal objective.
Note that since $g_{\LPSorder}^{n}$ is continuously differentiable, so is $F$, such that we can apply calculus of variations to solve for extremes with respect to $p$:
\begin{align*}
 &0 =\textstyle \frac{dF}{dp}(x) = -\frac{2(\LPSorder+1)\bar{C}_{\LPSorder}^{}}{2(\LPSorder+1)+d}p(x)^{-\frac{4(\LPSorder+1)+d}{2(\LPSorder+1)+d}}g_{\LPSorder}^{n}(x) + \lambda\\
 \Leftrightarrow\; &p^*(x) =\textstyle \left\{\frac{2(\LPSorder+1)\bar{C}_{\LPSorder}^{}}{\lambda(2(\LPSorder+1)+d)}g_{\LPSorder}^{n}(x)\right\}^{\frac{2(\LPSorder+1)+d}{4(\LPSorder+1)+d}}
\end{align*}
The concrete value of $\lambda^*$ is not of interest. It is enough to guarantee $\frac{dF}{d\lambda} = 0$ such that
\begin{align*}
\pOptLPS{\LPSorder}{n}(x) =\textstyle \hat{C}_{\LPSorder}^{-1}g_{\LPSorder}^{n}(x)^\frac{2(\LPSorder+1)+d}{4(\LPSorder+1)+d}(1+o(1)),
\end{align*}
where the normalization
$\hat{C}_{\LPSorder} =\textstyle \mathop{\mathlarger{\int}}_{\hspace*{-5pt}\inputSpace}g_{\LPSorder}^{n}(x)^\frac{2(\LPSorder+1)+d}{4(\LPSorder+1)+d}dx < \infty$
is bounded:

First of all let $\hat{D}_{\LPSorder} = \mathop{\mathlarger{\int}}_{\hspace*{-5pt}\inputSpace}v(x)^\frac{2(\LPSorder+1)}{4(\LPSorder+1)+d}\fctnComplexityLPS{\LPSorder}{\infty}(x)^\frac{2(\LPSorder+1)+d}{4(\LPSorder+1)+d}q(x)^\frac{2(\LPSorder+1)+d}{4(\LPSorder+1)+d}dx$.
Note that $\hat{D}_{\LPSorder} < \infty$,
since $v$ and $\fctnComplexityLPS{\LPSorder}{\infty}$ are bounded over \inputSpace, and $q \in L^\frac{2(\LPSorder+1)+d}{4(\LPSorder+1)+d}(\inputSpace)$ since $q \in L^1(\inputSpace) \subset L^\frac{2(\LPSorder+1)+d}{4(\LPSorder+1)+d}(\inputSpace)$.
Thus,
\begin{align*}
\textstyle\mathop{\mathlarger{\int}}_{\hspace*{-5pt}\inputSpace}\abs{g_{\LPSorder}^{n}(x)}^\frac{2(\LPSorder+1)+d}{4(\LPSorder+1)+d}dx &\textstyle\leq \mathop{\mathlarger{\int}}_{\hspace*{-5pt}\inputSpace}\bigg[\left[v(x)\big/n\right]^\frac{2(\LPSorder+1)}{2(\LPSorder+1)+d}\fctnComplexityLPS{\LPSorder}{\infty}(x)q(x)\bigg]^\frac{2(\LPSorder+1)+d}{4(\LPSorder+1)+d}dx\\
&+ \vol(\inputSpace)n^{-\frac{2(\LPSorder+1)+d}{4(\LPSorder+1)+d}}
= \hat{D}_{\LPSorder} n^{-\frac{2(\LPSorder+1)}{4(\LPSorder+1)+d}}(1 + o(1)) < \infty.
\end{align*}
Here, we defined by $\vol(\inputSpace) = \textstyle\mathop{\mathlarger{\int}}_{\hspace*{-5pt}\inputSpace}dx$ the volume of \inputSpace, which is finite, since \inputSpace in compact.
To summarize, asymptotically we can write
\begin{align*}
 \pOptLPS{\LPSorder}{n}(x) &=\textstyle \hat{C}_{\LPSorder}^{-1}g_{\LPSorder}^{n}(x)^\frac{2(\LPSorder+1)+d}{4(\LPSorder+1)+d}(1+o(1))\\
 &=\textstyle \hat{D}_{\LPSorder}^{-1}v(x)^\frac{2(\LPSorder+1)}{4(\LPSorder+1)+d}\left[\fctnComplexityLPS{\LPSorder}{\infty}(x)q(x)\right]^\frac{2(\LPSorder+1)+d}{4(\LPSorder+1)+d}(1 + o(1))\\
 &=\textstyle \hat{D}_{\LPSorder}^{-1}v(x)^\frac{2(\LPSorder+1)}{4(\LPSorder+1)+d}\left[\fctnComplexityLPS{\LPSorder}{n}(x)q(x)\right]^\frac{2(\LPSorder+1)+d}{4(\LPSorder+1)+d}(1 + o(1)). \qedAtBottomLine
\end{align*}
\end{proof}

Note that, if $\LOBfunctionOfLPS{\LPSorder}{n}$ is estimated based on $\bm{X}_{n}\sim \pOptLPS{\LPSorder}{n}$, then
\begin{align*}
 &\pOptLPS{\LPSorder}{n}(x) \propto \textstyle \left\{v(x)^{\frac{2(\LPSorder+1)}{2(\LPSorder+1)+d}}q(x)\abs{\normalizedLOBfunction(x)}^{-1}\right\}^\frac{2(\LPSorder+1)+d}{4(\LPSorder+1)+d}\\
 &\propto \textstyle \left\{v(x)q(x)\pOptFinite(x)^{-\frac{d}{2(\LPSorder+1)+d}}|\LOBfunctionOfLPS{\LPSorder}{n}(x)|^{-1}\right\}^\frac{2(\LPSorder+1)+d}{4(\LPSorder+1)+d}
 \propto \left\{v(x)q(x)|\LOBfunctionOfLPS{\LPSorder}{n}(x)|^{-1}\right\}^{\frac{1}{2}}.
\end{align*}

The factorization \eqref{eq:optimalSamplingFiniteLPS} reveals the influence of the local function complexity, noise variance and test relevance of an input, which reflects human intuition:
On the one hand, $q(x)$ tells us about the relevance of an accurate prediction in $x$. Accordingly, we benefit from reinforcing the training set where $q(x)$ is large.
On the other hand, we require more samples where the noise variance $v(x)$ or the local function complexity $\fctnComplexityLPS{\LPSorder}{n}(x)$ is large.
In addition, the factorization provides an exact quantitative result on how to account for each of these three factors in an optimal way.
\begin{myRemark}
When considering $f\in\diffableFunctions{\inputSpace}{L+1}$ with $L\in\N$ odd, note that $f$ is also a function of $\diffableFunctions{\inputSpace}{2}, \ldots, \diffableFunctions{\inputSpace}{L-1}$ such that any \LPS model of odd order $1, 3, \ldots, L$ could be applied to calculate \LFC \eqref{eq:isotropicComplexityLPS} and the optimal training density \eqref{eq:optimalSamplingFiniteLPS}.
In practice, we will stick to low-order \LPS models for computational tractability. However, due to the almost model-free nature of \LPS, we
consider these \LFC and optimal training density of $f$ that were obtained for the maximal applicable order as the `true' \LFC and optimal training density of $f$. That is, they represent properties that are intrinsic to $f$ rather than the \LPS model.
\end{myRemark}

\subsection{The Active Learning Procedure}
\label{subsec:ALprocedure}
For now let us assume that we are given $q$, \inputSpace and reasonable estimates $\approxLOBfunctionOfLPS{\LPSorder}{n}$ and $\widehat{v}$ of \LOB and the local noise level for any labeled training set $\bm{X}_n, Y_n$ of size $n$.
Then we can formulate an online sampling procedure that approaches $\bm{X}_{n}\sim \pOptLPS{\LPSorder}{n}$ as $n\rightarrow\infty$.

Let $n_0^{}$ be a small, but reasonable initial training size and $p_0^{} \equiv \uniformDist{\inputSpace}$ the initial training density according to which we sample the initial training inputs $\bm{X}_{n_0}$ with labels $Y_{n_0}$.
We then iterate over $k\in\N_{0}^{}$, beginning with $k=0$, to grow the training set as follows:
Given the current training set $\bm{X}_{n_{k}}, Y_{n_{k}}$ we estimate $\widehat{v}$ and $\approxLOBfunctionOfLPS{\LPSorder}{n_{k}}$.
Using \eqref{eq:normalizedLOBLPS}, \eqref{eq:isotropicComplexityLPS} and \eqref{eq:optimalSamplingFiniteLPS}, it is
\begin{align*}
&\textstyle\approxFctnComplexityLPS{\LPSorder}{n_k^{}}(x) \propto \left[\widehat{v}(x)\big/p_k^{}(x)\right]^{\frac{1}{2(\LPSorder+1)+d}}\abs{\approxLOBfunctionOfLPS{\LPSorder}{n_k^{}}(x)}^{-1},\;\;\text{and}\\
&\textstyle\approxPOptLPS{\LPSorder}{n_k^{}}(x) \propto \left[\approxFctnComplexityLPS{\LPSorder}{n_k^{}}(x) q(x)\right]^{\frac{2(\LPSorder+1)+d}{4(\LPSorder+1)+d}}\widehat{v}(x)^\frac{2(\LPSorder+1)}{4(\LPSorder+1)+d}.
\end{align*}
Setting $n_{k+1}^{} = 2n_k^{}$, we aim to draw $X_{n_k+1},\ldots,X_{n_{k+1}} \sim \widetilde{p}_{k+1}$ such that 
the new training set $\bm{X}_{n_{k+1}} \sim p_{k+1}$ is as close to the proposed optimal training density estimate $\approxPOptLPS{\LPSorder}{n_k^{}}$ in distribution as possible:

Writing $p_{k+1} := \gamma_2^{} p_k + (1 - \gamma_2^{}) \approxPOptLPS{\LPSorder}{n_k^{}}$, we therefore aim to minimize $\gamma_2^{} \in [0,1]$.
Note that we could simply set $\widetilde{p}_{k+1} \equiv \approxPOptLPS{\LPSorder}{n_k^{}}$, that is, sampling the new batch according to the current proposed optimal training density estimate. In this case we get $\gamma_2^{} = 0.5$.
However, a stronger similarity of $p_{k+1}$ to $\approxPOptLPS{\LPSorder}{n_k^{}}$ is always possible:
Define $p_{k+1} := \gamma_2^{} p_k+(1-\gamma_2^{})\approxPOptLPS{\LPSorder}{n_k^{}}$, where
\[\gamma_1^{} = \max_{x\in\inputSpace} \frac{p_k(x)}{\approxPOptLPS{\LPSorder}{n_k^{}}(x)} \in [1,\infty)\quad\text{and}\qquad\gamma_2^{} = \max\left\{0, \frac{0.5-\gamma_1^{-1}}{1-\gamma_1^{-1}}\right\} \in [0,0.5).\]
Then we obtain $\bm{X}_{n_{k+1}} \sim p_{k+1}$ by drawing $X_{n_k+1},\ldots,X_{n_{k+1}} \sim \widetilde{p}_{k+1}$ for $\widetilde{p}_{k+1} = 2p_{k+1} - p_k$.
Note that $\widetilde{p}_{k+1}$ is a valid probability density, because
\[\int_\inputSpace \widetilde{p}_{k+1}(x) dx = 2\int_\inputSpace p_{k+1}(x) dx - \int_\inputSpace p_k(x) dx = 2 - 1 = 1,\]
and $\widetilde{p}_{k+1} \geq 0$: Indeed, for any $x\in\inputSpace$, it is
\begin{align*}
 p_{k+1}(x)\big/p_k(x) &= \gamma_2^{} +  (1-\gamma_2^{})\approxPOptLPS{\LPSorder}{n_k^{}}(x)\big/p_k(x)
 \geq \gamma_2^{} + (1-\gamma_2^{})\gamma_1^{-1}\\
 &= \gamma_2^{} (1 - \gamma_1^{-1}) + \gamma_1^{-1} \geq \frac{0.5-\gamma_1^{-1}}{1-\gamma_1^{-1}}(1 - \gamma_1^{-1}) + \gamma_1^{-1} = 0.5.
\end{align*}
Therefore $2p_{k+1}(x) \geq p_k(x)$ such that $\widetilde{p}_{k+1}(x) = 2p_{k+1}(x) - p_k(x) \geq 0$.

\section{Practical Considerations}
\label{sec:practicalConsiderations}
While the theoretical result above is an insightful contribution on its own, one may ask about its relevance in practice, since on first glance we require a lot of information about the true data distribution. This information is especially scarce in the early stages of active learning, where we only have access to a small set of labeled instances.

First of all, we would like to emphasize that, while we rely on the explicit formulation of bias in our theoretical analysis, our results will apply to any consistent estimate of $\LOBfunctionOfLPS{\LPSorder}{n}(x)$.

Recalling \eqref{eq:asymptoticIsotropicLOBofLPS}, it is possible to construct \LOB via the bias which involves the estimation of $D^{\LPSorder+1}_f(x)$.
Yet, this approach is of very limited relevance: Even though we search over the restricted space of isotropic bandwidths $\SigmaSpace = \{\sigma\idMatrix{d}, \sigma > 0\}$, which has only one degree of freedom, the number of components of the derivative to fit grows as $O(d_{}^{\LPSorder+1})$. This quickly becomes computationally intractable as $d$ increases: For example, when considering its estimation via \LPS \citep{zhang2011kernel}, it involves solving a linear system of the size of the number of the derivative components, which is $O(d_{}^{3(\LPSorder+1)})$.
Moreover, even estimates of lower order derivatives like the Hessian require a large amount of training data.
In our experiments, we use Lepski's method \citep{lepski1991problem,lepski1997optimal}, which is a direct estimate to \LOB that avoids these difficulties. We detail this approach in \secref{subsec:LOBestimation}.

With respect to the other two quantities, namely local noise variance $v(x)$ and test density $q(x)$, we would like to emphasize that we do not necessarily require full estimates for them:
In the case where homoscedasticity is (or can be) assumed, we only have to estimate a constant $v(x) \equiv \text{v}$.
Whether we need a local estimate $v(x)$ or the global estimate $\text{v}$, we discuss a simple approximation in \secref{subsec:noiseVarianceEstimation}, which then serves as an input to Lepski's method. 
Finally regarding $q(x)$, there are two major scenarios: We might
know $\inputSpace$ and the test density $q(x)$ is externally specified -- therefore requires no estimation. Here, the canonical candidate is the uniform distribution $q \sim \uniformDist{\inputSpace}$, where we aim to optimize the prediction performance uniformly well over the input space.

In the other scenario, we observe inputs $x \sim p_\inputSpace^{}$ from an unknown data generating process. Here, the natural candidate for the test density $q \equiv  p_\inputSpace^{}$, and we require an estimate $\widehat{p}_\inputSpace$. In this scenario -- also known as \emph{pool-based} active learning -- it is a common assumption that unlabeled input instances are cheaply obtainable in contrast to labeled samples. Therefore we can estimate $\widehat{p}_\inputSpace$ in advance, prior to the actual active learning procedure. We refer to the broad literature on density estimation \cite{silverman1986density,wand1994kernel} in this case.

\subsection{Estimation of the local Noise Variance}
\label{subsec:noiseVarianceEstimation}
Regarding the estimation of the local noise variance $v(x)$, we first consider the homoscedastic scenario for d=1, where $v(x) \equiv \text{v}$. Here, a robust estimate to $\text{v}$ based on the \emph{median absolute deviation} (MAD) is known (see, for example, \cite{katkovnik2006signalProcessing}):
\begin{align}
 \label{eq:robustGlobalNoiseStandardDeviation}
 \sqrt{\widehat{\text{v}}} = (\sqrt{2}\cdot0.6745)^{-1}\text{median}(\condset{\abs{y_{\pi(i)} - y_{\pi(i+1)}}}{0 < i < n}),
\end{align}
where $\pi$ is an ordering index permutation such that $x_{\pi(i)} < x_{\pi(i+1)}$ for $0 < i < n$.
The MAD estimate \eqref{eq:robustGlobalNoiseStandardDeviation} relies on the fact that $y_{\pi(i)} - y_{\pi(i+1)} \rightarrow \varepsilon_{\pi(i)} - \varepsilon_{\pi(i+1)}$ as $n\rightarrow\infty$.
The idea may fail at small sample sizes where $\abs{f(x_{\pi(i)}) - f(x_{\pi(i+1)})} \gg \epsilon\pnorm{f}{}$ for some considerable constant $\epsilon > 0$. This is critical as it influences the subsequent sampling process in a negative way.
We tackle this issue by replacing $y_i^{}$ above with the residuals $r_i = y_i^{} - \widehat{f}^\sigma(x_i^{})$ for a global predictor $\widehat{f}^\sigma(x) =  \predictorLPS{\LPSorder}{\sigma\idMatrix{d}}(x)$ according to the \LPS model \eqref{eq:lpsPrediction}, for which we cross-validate the constant bandwidth $\sigma$ over the current training set.

As a generalization of \eqref{eq:robustGlobalNoiseStandardDeviation} to both, higher dimensions $d > 1$ and heteroscedasticity, let $\mathcal{J}^i_m \subseteq \bm{X}_n^{}$ be the set of indices of the m-nearest neighbors of $x_i^{}$ in $\bm{X}_n^{}$, and for an arbitrary $x\in\inputSpace$ let $\mathcal{I}_l(x) \subseteq \bm{X}_n^{}$ be the indices of the l-nearest neighbors of $x$ in $\bm{X}_n^{}$.
Then we estimate
\begin{align}
 \label{eq:robustLocalNoiseStandardDeviationAdaptedMultivariate}
 \sqrt{\widehat{v}(x)} = (\sqrt{2}\cdot0.6745)^{-1}\text{median}\sideset{}{_{i \in \mathcal{I}_l(x)}^{}}\bigcup\sideset{}{_{j\in\mathcal{J}^i_m}^{}}\bigcup \abs{r_i - r_j}.
\end{align}
While $m$ is a free parameter, we set $m = 2d$ in the following.
For asymptotic consistency of the estimate, the number of neighbors $l := l_n$ should increase with n, but $l_n \in o(n)$ such that the expected diameter of the neighborhoods $\condset{x_j}{j\in\mathcal{J}^i_m}$ decreases for all $i$.
For an optimal trade-off, note that according to \cite{foi2007pointwise} (page 161), it is $\E(\sqrt{\widehat{v}(x)}) = \sqrt{v(x)}$ and $\Var(\sqrt{\widehat{v}(x)}) \approx 1.35 v(x)/\left((l-1) + 1.5\right)$.
Furthermore, according to \cite{evans2002asymptotic}, $\E\abs{x_i-x_j} \lesssim \left(m/n\right)^\frac{1}{d}$ for $j\in\mathcal{J}^i_m$, and analogously $\E\abs{x_i-x_j} \lesssim \left(l/n\right)^\frac{1}{d}$ for $i \in \mathcal{I}_l(x)$.
Therefore
\begin{align*}
 \E\left[(\sqrt{\widehat{v}(x)} - \sqrt{v(x)})^2|\bm{X}_n\right] \leq \underbrace{v(x)\frac{1.35}{l+0.5}}_{O(l^{-1})} + \underbrace{\pnorm{D_f}{\infty}^2\left(\frac{m}{n}\right)^\frac{2}{d} + \pnorm{D_v}{\infty}^2\Big(\frac{l}{n}\Big)^\frac{2}{d}}_{O(\left(\frac{l}{n}\right)^\frac{2}{d})}.
\end{align*}
For fastest convergence, we need to balance both error components, giving the optimal relation $l_n = \lceil C_{v} n^\frac{2}{2+d}\rceil$ in the heteroscedastic case for some reasonably chosen constant $C_{v} > 0$.
When we have expert knowledge about homoscedasticity, we apply \eqref{eq:robustLocalNoiseStandardDeviationAdaptedMultivariate} with $l_n = n$, forcing it to become a global estimate again.

\subsection{Estimation of the locally Optimal Bandwidths}
\label{subsec:LOBestimation}
Lepski et al.~\citep{lepski1991problem,lepski1997optimal} considered optimal pointwise adaptation in the broader context of nonparametric estimation.
We will follow the work of \cite{zhang2011kernel}, who implemented Lepski's method for \LPS.

Consider a set of logarithmically spaced bandwidth candidates, that is, $\bandwidth_0,\ldots,\bandwidth_L$ with $\bandwidth_{j} = \sigma_j\idMatrix{d}$ and $\sigma_j = \underline{\sigma}\cdot s^j$ for a step size $s > 1$ and a lower bound $\underline{\sigma}$.
For an $x\in\inputSpace$ we choose $\approxLOBfunctionOfLPS{\LPSorder}{n}(x)$ according to the \emph{intersection of confidence interval} (ICI) rule:
Let
\[\textstyle\mathcal{C}_{\LPSorder}^j(x) = \condset{c\in\R}{\!|c - m_{\LPSorder}^{\bandwidth_{j}}(x)| \leq \kappa(1 + 2/(s^{\frac{2(\LPSorder+1)+d}{2}} - 1))\sqrt{\Var_{\LPSorder}^{}\left(x, \bandwidth_{j} | \bm{X}_{n}\right)}}\]
be the confidence interval constructed such that $\E m_{\LPSorder}^{\bandwidth_{j}}(x) \in \mathcal{C}_{\LPSorder}^j(x)$ with high probability, where we can calculate the prediction variance $\Var_{\LPSorder}^{}\left(x, \bandwidth_{j} | \bm{X}_{n}\right)$ according to \eqref{eq:trueFiniteVariance}.
For example, for $\kappa = 1.96$ and $\kappa = 2.58$ it is $\Prob(m_{\LPSorder}^{\bandwidth_{j}}(x) \in \mathcal{C}_{\LPSorder}^j(x)) = 0.95$ and $0.99$, respectively.
Furthermore let
\[j^*(x) = \max\condset{0 \leq j \leq L}{\sideset{}{_{i\leq j}}\bigcap \mathcal{C}_{\LPSorder}^i(x) \neq \emptyset},\]
meaning that for $j > j^*(x)$ the confidence intervals do not intersect anymore.
We then set 
\begin{align}
 \label{eq:LOBlepskiEstimate}
 \approxLOBfunctionOfLPS{\LPSorder}{n}(x) = \bandwidth_{j^*(x)}.
\end{align}
Commonly, the bandwidth candidate parameters are set to $s = 2$, $\underline{\sigma} > 0$ fixed to a small value, and $L$ such that $\sigma_L \leq n\big/\log(n)$.
We will however choose $\underline{\sigma}_n$ and $s_n$ adaptively: For some reasonable constant $C_{\sigma} > 0$ we set $\underline{\sigma}_n = C_{\sigma} n^{-\frac{2}{2(\LPSorder+1)+d}}$, which decays twice as fast as the usual bandwidth decay rate, enabling us to reproduce fine structure that may first reveal at larger training sizes.
Furthermore, noting that the confidence interval range is proportional to $n^{-\frac{\LPSorder+1}{2(\LPSorder+1)+d}}\big/\Big[s_n^{\frac{2(\LPSorder+1)+d}{2}}-1\Big]$
, we set $s_n = \left[1 + C_{s}n^{-\frac{\LPSorder+1}{2(\LPSorder+1)+d}}\right]^\frac{2}{2(\LPSorder+1)+d}$ such that $s_n$ decays, but slow enough to not blow up the confidence intervals as $n$ increases. Here, $C_{s} > 0$ is a reasonable constant. For example, with $C_{s} = n_0^{\frac{\LPSorder+1}{2(\LPSorder+1)+d}}(2^\frac{2(\LPSorder+1)+d}{2}-1)$ we obtain $s_{n_0} = 2$.

\subsection{Stabilization of local Estimates}
\label{subsec:localStabilization}
At small training size $n$, the effective number of samples that are involved in the estimation of $\approxLOBfunctionOfLPS{\LPSorder}{n}$ and $\widehat{v}$ can be marginal which may result in quite unstable estimates.
In an online procedure, where we add new samples according to our proposed optimal training density that relies on these estimates, such instability is critical: Given a faulty estimate of small noise or complexity (a large bandwidth) in $x$, we might end up in a singular case, where we add no further sample in the vicinity of $x$ for a long time.

In order to prevent this, we suggest to replace the pointwise estimates of the noise level and \LOB most conservatively as follows:
Let $B_{\delta_n}(x) = \condset{x'\in\inputSpace}{\pnorm{x-x'}{} \leq \delta_n}$ be the ball around $x$ of radius $\delta_n$.
We replace the estimate $\widehat{v}(x)$ from \eqref{eq:robustLocalNoiseStandardDeviationAdaptedMultivariate} by
\begin{align}
 \label{eq:stableNoiseEstimate}
 \textstyle\widetilde{v}(x) = \max\condset{\widehat{v}(x')}{x'\in B_{\delta_n}(x)}.
\end{align}
Similarly, we replace $\approxLOBfunctionOfLPS{\LPSorder}{n}(x)$ in \eqref{eq:LOBlepskiEstimate} by 
\begin{align}
 \label{eq:stableLepskiEstimate}
\widetilde{\bandwidth}^{n}_{Q}(x) &= \bandwidth_{\tilde{j}^*(x)},\quad\text{where}\\
\notag\tilde{j}^*(x) &= \min\condset{0 \leq j \leq L}{\exists x'\in B_{\delta_n}(x): \approxLOBfunctionOfLPS{\LPSorder}{n}(x) = \bandwidth_j}.
\end{align}
This is conservative in the sense that we tend to overestimate the local noise level and local function complexity, which results in more sample mass in the subsequent sampling step.

For asymptotic optimality it must hold $\delta_n\rightarrow 0$ as $n \rightarrow \infty$. On the other hand, $\delta_n$ should decay slow enough such that the expected number of training samples in $B_{\delta_n}(x)$ increases. Thus, it must hold $\delta_n = \divergenceInProbability{\noexpand\big}{n^{-\frac{1}{d}}}$, and we choose $\delta_n = C_{\delta}n^{-\frac{1}{d(d+1)}}$ for some reasonable constant $C_{\delta} > 0$.

\subsection{Boundary Correction of the Optimal Training Density}
\label{subsec:boundaryCorrection}
Even though \thmref{thm:OptimalSampling} holds asymptotically almost everywhere, for finite training size $n$ there will be an undesired behavior at the support boundary: While for odd-order $\LPSorder$ the \MSE convergence law at the support boundary is consistent with the law in the interior of the support, the conditional bias and variance behave differently to equations \eqref{eq:asymptoticLeadingOrderBiasLPS} and \eqref{eq:asymptoticVarianceLPS} in a non-trivial way.
Asymptotically, this can be ignored as the support boundary makes up a set of measure zero. Yet, for finite $n$, the support boundary has substantial measure and we suggest to perform a correction of the proposed optimal training density estimate:

Let $\supportInteriorStdDevs > 0$ be some reasonable factor of the standard deviation to the kernel $k$ such that
\[\int_{B_\supportInteriorStdDevs(0)} k(u) du = 1 - \varepsilon,\]
for some small $\varepsilon > 0$. For example, for the Gaussian kernel we may apply a value of $1 \leq \supportInteriorStdDevs \leq 3$.
At training size $n$, we define the effective support interior \effectiveInputSpaceInterior of \inputSpace as
\begin{align}
\label{def:effectiveSupportInterior}
\textstyle\effectiveInputSpaceInterior = \condset{x\in\inputSpace}{\condset{x'\in\reals{d}}{\pnorm{[\LOBfunctionOfLPS{\LPSorder}{n}(x)]^{-1}(x-x')}{} \leq \supportInteriorStdDevs} \subset \inputSpace}.
\end{align}
Note that for any $\supportInteriorStdDevs > 0$ , $\effectiveInputSpaceInterior \rightarrow \inputSpaceInterior$ as $n\rightarrow\infty$.
As a correction, we suggest to set
\begin{align}
\label{eq:boundaryCorrectedOptimalSamplingFiniteLPS}
\tildePOptLPS{\LPSorder}{n}(x) = \begin{cases} \approxPOptLPS{\LPSorder}{n}(x) &, x \in \effectiveInputSpaceInterior,\\
\approxPOptLPS{\LPSorder}{n}(x^\circ_n)&,\text{else, where }x^\circ_n = \argmin_{x'\in\effectiveInputSpaceInterior}\pnorm{x-x'}{}.
\end{cases}
\end{align}

\subsection{Algorithmic Summary}
\label{subsec:algoSummary}
Let us recapitulate that we require reasonable constants $C_{v}$ for the localization of the noise estimate, $C_{\sigma}$ and $C_{s}$ for the lower bandwidth candidate bound and the step size of the bandwidth candidates, and the confidence interval size factor $\kappa$ in Lepski's method.
Furthermore, we require the constants $C_{\delta}$ for stabilization of the local property estimates, and 
the standard deviations factor \supportInteriorStdDevs of the kernel to identify regions of the input space that will suffer from boundary effects at finite training sizes.
Given these constants, we have summarized a full active learning step of our proposed framework in \algoref{fig:algorithmALIteration}.




\begin{algorithm}[H]
  \caption{Construction of $(\bm{X}_{n_{k+1}}, Y_{n_{k+1}})$ from $(\bm{X}_{n_{k}}, Y_{n_{k}})$}
  \label{fig:algorithmALIteration}
  \renewcommand\thealgorithm{}
  \color{black}
  \relsize{-1}
\vspace{1.5mm}
  \begin{algorithmic}[1]
  \Algphase{Input}
  \State Current training inputs $\bm{X}_{n_{k}}$ with labels $Y_{n_{k}}$ and training density $\bm{X}_{n_{k}} \sim p_k$
  \State The test density $q$
  \State The order \LPSorder of the underlying \LPS model
  \State Constants $C_{v}$ for the localization of the noise estimate, $C_{s}$ and $C_{\sigma}$ for Lepski's bandwidth candidate step size factor $s_n$ and the smallest candidate $\underline{\sigma}_n$ and
  $C_{\delta}$ for choosing the local estimates stabilizer $\delta_n$
  \State The confidence interval size factor $\kappa$ of Lepski's method and
  the standard deviations factor \supportInteriorStdDevs of the kernel for boundary correction of the optimal training density estimate
  \State A Boolean \isHomoscedastic, if homoscedastic noise can be assumed
  \Algphase{Output}
  \State New training inputs $\bm{X}_{n_{k+1}}$ with labels $Y_{n_{k+1}}$ and training density $\bm{X}_{n_{k+1}} \sim p_{k+1}$
  \Algphase{Procedure}
  \State \textbf{if} \isHomoscedastic \textbf{then} \Comment{Set training size dependent parameters}\newline $l_{n_k} = n_k$ \textbf{else} $l_{n_k} = C_{v}n_k^\frac{2}{2+d}$
  \State $s_{n_k} = \left[1 + C_{s} n_k^{-\frac{\LPSorder+1}{2(\LPSorder+1)+d}}\right]^\frac{2}{2(\LPSorder+1)+d}$, $\underline{\sigma}_{n_k} = C_{\sigma}n_k^{-\frac{2}{2(\LPSorder+1)+d}}$, $\delta_{n_k} = C_{\delta}n_k^{-\frac{1}{d(d+1)}}$
  \State Estimate $\widehat{v}$ according to  \eqref{eq:robustLocalNoiseStandardDeviationAdaptedMultivariate}, using $\bm{X}_{n_{k}}, Y_{n_{k}}$ and neighborhood size $l_{n_k}$ \Comment{see \secref{subsec:noiseVarianceEstimation}}
  \State Obtain the stabilized noise estimate $\widetilde{v}$ according to \eqref{eq:stableNoiseEstimate}, using $\delta_{n_k}$ \Comment{see \secref{subsec:localStabilization}}
  \State Estimate \LOB $\approxLOBfunctionOfLPS{\LPSorder}{n_k}$ via Lepski's method \eqref{eq:LOBlepskiEstimate}, \Comment{see \secref{subsec:LOBestimation}}\newline using $\bm{X}_{n_{k}}, Y_{n_{k}}, \widetilde{v}, \kappa, s_{n_k}$ and $\underline{\sigma}_{n_k}$
  \State Obtain the stabilized \LOB estimate $\widetilde{\bandwidth}^{n_k}_{Q}$ according to \eqref{eq:stableLepskiEstimate}, using $\delta_{n_k}$ \Comment{see \secref{subsec:localStabilization}}
  \State Estimate $\approxPOptLPS{\LPSorder}{n_k}$ according to \eqref{eq:normalizedLOBLPS}, \eqref{eq:isotropicComplexityLPS} and \eqref{eq:optimalSamplingFiniteLPS} \Comment{see \secref{subsec:ALprocedure}}
  \State Perform boundary correction \eqref{eq:boundaryCorrectedOptimalSamplingFiniteLPS} to obtain $\tildePOptLPS{\LPSorder}{n_k}$ \Comment{see \secref{subsec:boundaryCorrection}}
  \State Set $\gamma_1^{} = \max_{x\in\inputSpace}\left. p_k(x) \middle/ \tildePOptLPS{\LPSorder}{n_k}(x) \right.$ and $\gamma_2^{} = \max\left\{0, \left.(0.5-\gamma_1^{-1})\middle/(1-\gamma_1^{-1})\right.\right\}$
  \State Set $p_{k+1} = \gamma_2^{} p_k + (1 - \gamma_2^{}) \tildePOptLPS{\LPSorder}{n_k}$ \Comment{see \secref{subsec:ALprocedure}}
  \State Sample $X_{n_k+1},\ldots,X_{n_{k+1}} \sim \widetilde{p}_{k+1}$, where $\widetilde{p}_{k+1} = 2 p_{k+1} - p_k$, to obtain $\bm{X}_{n_{k+1}} \sim p_{k+1}$
  \State Query the labels $y_{n_k+1},\ldots,y_{n_{k+1}}$ of $X_{n_k+1},\ldots,X_{n_{k+1}}$ to obtain $Y_{n_{k+1}}$
  \end{algorithmic}
\end{algorithm}

\section{Discussion and Related Work}
\label{sec:discussionRelatedWork}
We will now elaborate the active learning properties mentioned in the introduction in more detail and show that our proposed methodology encompasses all of them. Additionally, we will discuss related work in the light of these properties.
\subsection{The Active Learning Properties}
\label{subsec:discussionALProperties}


Let us recall from \secref{sec:intro} that -- besides the fundamental categories of (un)supervised and model-free/based sampling schemes -- the most relevant properties of sampling schemes in the scope of this work are \optimalProperty, \robustProperty and \modelagnosticProperty.
We have also motivated why it is preferable to have a sampling scheme that possesses these three properties at the same time. We will now discuss these three properties and their relation in more detail. In particular, we will conclude that a sampling scheme has to be necessarily supervised and nonparametric in order to fulfill them simultaneously.

First of all note that the difference between a \robust and a \modelagnostic sampling scheme is subtle: Especially in the parametric regime, both properties never hold, making them coincide trivially.
In contrast, in the nonparametric regime, let us consider \emph{uncertainty sampling} for standard \emph{Gaussian process regression} (\GPR) \citep{seo2000gaussian}, where new samples are drawn so as to minimize the integrated predictive variance.
By standard \GPR we mean a Gaussian process that is based on a global bandwidth parameter that allows for no local adaption.
Here, the bias is assumed to be negligible, compared to the variance, and thus can be completely ignored.
For those regression problems where the assumption of a negligible bias is justified, \GPR uncertainty sampling is an \optimal sampling scheme for \GPR and superior to \randomTestSampling across model classes, making it a \modelagnostic sampling scheme.
Ignoring the bias is however incorrect for regression problems of inhomogeneous complexity structure, where the bias varies over the input space.
Since standard \GPR uncertainty sampling does not account for this inhomogeneity, it is not an appropriate sampling scheme  for \GPR or other model classes. As in practice inhomogeneously complex regression problems occur frequently \citep{donoho1994ideal,benesty2013adaptive} the assumption of a negligible bias is not a mild one. In this light, standard \GPR uncertainty is not \robust.


As we have already discussed in the introduction, it is somewhat contradictory for a sampling scheme to be \optimal on the one hand, and \robust and \modelagnostic on the other hand. In fact, in the unsupervised regime \robustProperty and \optimalProperty are mutually exclusive:
\begin{mylem}
\label{lem:incompatiblity}
A sampling scheme cannot be optimal, robust and unsupervised, simultaneously.
\end{mylem}
\begin{proof}
Assume by contradiction that there exists an \optimal, \robust and unsupervised sampling scheme.
Since it is \optimal, it must be model-based. This model $\widehat{f}$ must be nonparametric, due to \robustProperty, and it must treat bias and variance. Now assume some arbitrary input space \inputSpace. Since the sampling scheme is unsupervised, ($\widehat{f}$, \inputSpace) already uniquely determine the associated optimal training distribution $\bm{X}_n \sim p^{\text{train}}_{n}$. Since $\widehat{f}$ is nonparametric, itself as well as $p^{\text{train}}_{n}$ are local. That is, for subsets $\mathcal{A} \subset \inputSpace$ the optimal training density of the restriction $\widehat{f}\raisebox{-3pt}{|}_\mathcal{A}$ is given by $p^{\text{train}}_{n}\raisebox{-3pt}{|}_\mathcal{A}$.
Now, \robustProperty and \optimalProperty imply that $p^{\text{train}}_{n}$ is optimal for 
a substantial space of label distribution $p(y|x)$, whereas it is not optimal for a negligible space of labelings.
Optimal means that $\bm{X}_n$ minimizes, for example, the \MISE for such a labeling $p(y|x)$.

Without loss of generality we assume that $p^{\text{train}}_{n} \not\sim \uniformDist{\inputSpace}$, because $\uniformDist{\inputSpace}$ is an optimal solution for a set of labelings of measure zero in the space of heteroscedastic and inhomogeneously complex problems. Then there necessarily exists a reflection $R_{\bar x}$ in $\bar x \in \inputSpaceInterior$ with $B_\varepsilon(\bar x) \subset \inputSpace$ for some $\varepsilon > 0$ such that $p^{\text{train}}_{n} \circ R_{\bar x}\raisebox{-3pt}{|}_{B_\varepsilon(\bar x)} \not\equiv p^{\text{train}}_{n}\raisebox{-3pt}{|}_{B_\varepsilon(\bar x)}$.
We now construct the new labeling $p^*(y|x) = \begin{cases} p(y|R_{\bar x}(x)),& x\in B_\varepsilon(\bar x)\\ \widetilde{p}(y|x),&\text{else}\end{cases}
$ over \inputSpace, where we can choose $\widetilde{p}(y|x)$ as a continuation of $p(y|R_{\bar x}(x))$ outside $B_\varepsilon(\bar x)$ that preserves the regularity of $p(y|x)$.
By construction, $p^{\text{train}}_{n}$ is not the optimal training distribution for $p^*(y|x)$, because of the locality of $\widehat{f}$ and $p^{\text{train}}_{n}$.
Now that the choice of $p(y|x)$ was arbitrary, we have a substantial set of labelings $p^*(y|x)$ on \inputSpace for which
$p^{\text{train}}_{n}$ is not \optimal, in contradiction to the assumption.
\end{proof}
For example, there exist model-based, \optimal sampling schemes that are also unsupervised -- that is, they do not depend on labels even if some labeled samples are available, like in optimum experimental design \citep{kiefer1959optimum,mackay1992information,he2010laplacian} and in uncertainty sampling for \GPR \citep{seo2000gaussian}.
In the light of \lemref{lem:incompatiblity}, these sampling schemes are necessarily not \robust, which means that they must impose strong assumptions on the regularity of the labels. In fact, particularly these strong assumptions are necessary to make such sampling schemes unsupervised.

Second, sampling schemes based on parametric models are quite specific to the model and are therefore never \modelagnostic. As a compromise between model-free and parametric approaches, nonparemetric models impose rather mild conditions on the labels. Thus, a sampling scheme based on a nonparametric model is more promising, though not guaranteed, to be \modelagnostic.
Finally, advanced model-free sampling schemes \citep{teytaud2007active} are inherently \modelagnostic and \robust, because they were not derived via a model that might assume any regularity of the labels. However also by definition, they are necessarily not \optimal, because of the absence of such a model.

To summarize, we sketch the relation between \optimal, \robust and \modelagnostic sampling schemes, as well as (un)supervised and (non)parametric sampling schemes in \figref{fig:AL_properties_withReferences}, where we also exemplarily classify the discussed active learning approaches from related work.
The goal of our work was to propose a sampling scheme that combines
the amenities of both sides, model-free and \optimal approaches.
Namely, it should lead to a consistent performance increase across several model classes, while being truly adapted to the regression task at the same time.
Therefore it has to be simultaneously \optimal, \robust and \modelagnostic.

With the above arguments we can narrow down the set of candidate sampling schemes as claimed: Namely, such a sampling scheme must be model-based, since model-free sampling schemes are never \optimal. Next, since parametric sampling schemes are never \modelagnostic, the candidate must be necessarily nonparametric. Finally, in order to additionally match \robustProperty, we need a supervised sampling scheme since \robustProperty and \optimalProperty are incompatible in the unsupervised regime.

\begin{figure}[tb]
\centering

\begin{minipage}[tb]{0.64\linewidth}
\vspace{0pt}
\includegraphics[width=1.\linewidth]{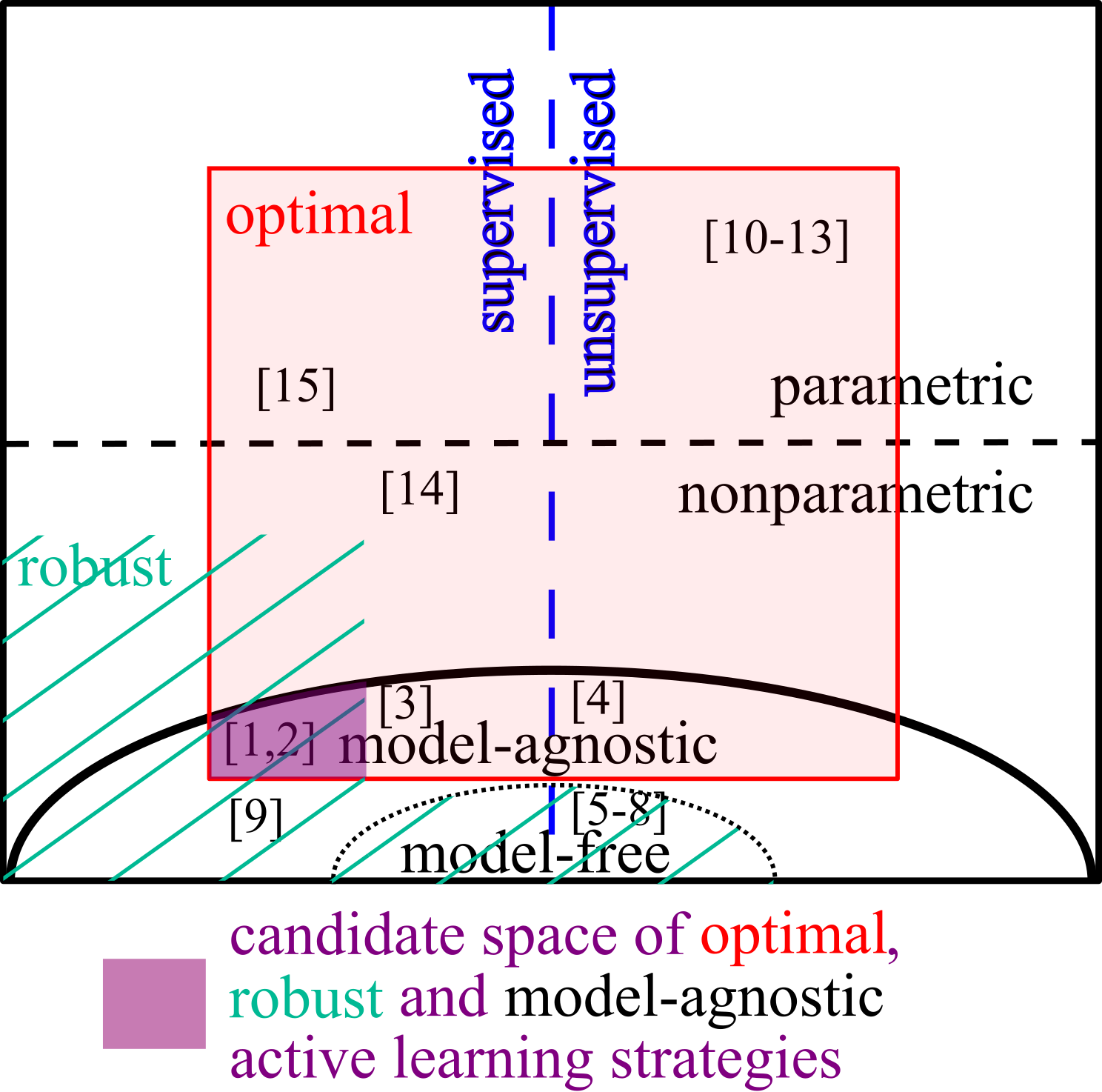}
\label{fig:AL_properties_withReferences}
\end{minipage}
\hfill
\begin{minipage}[tb]{0.32\linewidth}
\vspace{-40pt}
\footnotesize
\begin{enumerate}[{[}1{]}]
 \itemsep0em
 \item Our novel approach
 \item \cite{goetz2018active}
 \item \cite{bull2013spatially}
 \item \cite{seo2000gaussian}
 \item \cite{teytaud2007active}
 \item \cite{yu2010passive}
 \item \cite{wu2019pool}
 \item \cite{liu2021pool}
 \item \cite{kiefer1959optimum}
 \item \cite{mackay1992information}
 \item \cite{he2010laplacian}
 \item \cite{sugiyama2009pool}
 \item \cite{douak2013}
 \item \cite{cohn1997minimizing}
 \item \cite{seung1992query}
\end{enumerate}
\end{minipage}
\vspace*{-15pt}
\caption{Relation between the main properties of active learning approaches we focus on in this work. The filled, purple area contains the candidates for our purpose.}
\end{figure}

In addition, our sampling scheme is \stationary and \interpretable:\smallskip

\noindent{\bf Further properties of sampling schemes:} We call a sampling scheme
\begin{itemize}
    \item \textbf{stationary}, if the training inputs can be formulated as an independently and identically distributed random sample of a fixed distribution.
    \item \textbf{interpretable}, if the decision making on which labels to query can be visualized to and understood by a domain expert.
\end{itemize}
A \stationary sampling scheme shares the benefits of unsupervised sampling schemes, when the desired terminal training size is large, or when the annotator of the label queries, such as a human expert, is not always available \citep{settles2010active}.
For a \stationary sampling scheme, the data acquisition can be proceeded independently from data annotation -- even if supervised.
In contrast, for example, information-based sampling schemes \citep{mackay1992information} require re-estimation of the information measure after the acquisition of each new sample.

The advantage of an \interpretable sampling scheme is that we can give an explanation for what reason a proposed query seems informative, which a human expert is able to comprehend.
In this case, a domain expert is able to monitor the healthiness of the sampling process, as opposed to a black box sampling process where a potential faulty behavior will reveal in hindsight \citep{lapuschkin2019unmasking,Samek2021explaining}.
Therefore, such transparency makes active learning more appealing in real-world environments, where data acquisition and annotation is expensive, as we can reduce the risk of wasting costs.

\subsection{Properties of the Proposed Sampling Method}
\label{subsec:ALPropertiesOfOptimalSampling}
Our proposed sampling scheme fulfills all five introduced properties:\smallskip

\noindent{(\textbf{optimal})}
\thmref{thm:OptimalSampling} guarantees that, asymptotically, our proposed sampling scheme is the solution to \eqref{eq:optimalTrainingSet}, the minimal \MISE of the \LPS model, which proves optimality.\smallskip

\noindent{(\textbf{robust})} In order to asymptotically minimize \MISE, we rely on the expressions of the leading-order bias- and variance-terms as given in \thmref{thm:asymptoticBiasVarianceForLPS}.
Therefore, we assume certain smoothness of the noise level, the test density and the function to learn, as well as non-vanishing leading terms, almost everywhere. All these assumptions are mild (as opposed to assuming homoscedasticity or a negligible bias), which makes our sampling scheme \robust.\smallskip

\noindent{(\textbf{model-agnostic})} As already mentioned, our sampling scheme is likely to provide the \modelagnosticProperty, since the \LPS model as the base of our objective is almost model-free. We will show in the experiments in \secref{sec:experiments} that this property indeed holds.\smallskip

\noindent{(\textbf{stationary})} Since $\pOptLPS{\LPSorder}{n}$ converges in probability to an asymptotic density, our sampling scheme $\bm{X}_n^* \sim \pOptLPS{\LPSorder}{n}$ is asymptotically stationary. Hence, it features the stationarity property.\smallskip

\noindent{(\textbf{interpretable})}
Instead of having to rely on a black box information score, the closed-form solution of $\pOptLPS{\LPSorder}{n}$ reveals influence of \LFC, noise variance and test density on the optimal sample. These three scalar-valued properties can intuitively be understood by a human expert, which makes our sampling scheme \interpretable.

\subsection{Related Work}
\label{subsec:relatedWork_ALregression}
Recall that our proposed active learning scheme is \optimal, \robust and \modelagnostic, among other properties.
It is therefore especially well designed for the mid- to long-term construction of a meaningful training set in regression tasks where there is scarce domain knowledge, such that we only know few about the regularity of the problem and/or have not determined the final regression model to solve the problem optimally.
However, as we have deduced in \secref{subsec:discussionALProperties}, an approach with the specification of our proposed sampling scheme is necessarily supervised.
And like any supervised sampling scheme, we thus require a small, but sufficient amount of training labels for the initial estimation of the sampling criterion.

This initialization could be done by \randomTestSampling, but also by any reasonable unsupervised sampling scheme -- as suggested by \cite{liu2021pool}.
Especially, for the reasons of a consistent performance increase and model flexibility, unsupervised, input space geometric sampling schemes enjoy great popularity \citep{teytaud2007active,yu2010passive,wu2019pool,liu2021pool} in practice. As they are model-free, they are inherently \modelagnostic and \robust by definition.
However also by definition, they are necessarily not \optimal, and thus become inferior to \optimal sampling schemes in the long run:
For example \cite{teytaud2007active} aim to make the training set as diverse as possible.
While such approaches show advantages in the early stages of the training set construction, they become inferior to \optimal sampling schemes as soon as the input space is well represented.
In summary, unsupervised, input space geometric sampling schemes and sampling schemes with the same specification as our proposed framework play a complementary role:
While both categories are \robust and \modelagnostic, the prior one works from scratch, whereas the latter one is \optimal.

Let us also take a look at unsupervised, model-based sampling schemes that could be used, among other things, for the initialization of a supervised sampling scheme.
Model-based, unsupervised approaches eliminate the dependence on the labels by imposing strong model assumptions. For example, 
in linear parametric regression a correct model specification is assumed such that bias can be considered negligible:
Here, the minimization of the expected generalization loss can be translated to maximizing information gain. Some approaches encode information via the variance of the model parameters at the current training state (see, for example, \cite{sugiyama2009pool}). Also the theory of optimum experimental design
(see, for example, \cite{kiefer1959optimum}) follows this approach, using the \emph{Fisher information} as a measure to minimize the variance. These methods have
since been extended beyond simple linear methods to kernel methods and modified to take into account regularization (see \cite{he2010laplacian}).
\cite{mackay1992information} follow this idea in a Bayesian setting.
In the nonparametric regime such as Gaussian process regression, \cite{seo2000gaussian} assume strong correlation between the shape of the applied kernel and the predictive variance.

A parametric, unsupervised approach is not \modelagnostic, making it no good candidate for the initialization of a supervised approach. Additionally, both, parametric and nonparametric, unsupervised approaches are not \robust in our definition. Therefore also a nonparametric, unsupervised should be applied with caution when there is almost no domain knowledge that may justify the model assumptions.
On the other hand, when the model assumptions are justified the respective models can also serve reasonably as the final prediction model. Especially a parametric approach has a \MISE that decays at the rate $O(n^{-1})$, which is superior to the convergence of any nonparametric model.

Unsupervised sampling schemes often do not require intensive recalculations of the sampling criterion after each query. Additionally, as they do not rely on the labels to query, they will not suffer from a potential bottleneck at the label annotation, which could, for example, be a human that is not always available.
Therefore they share the essential benefits of a stationarity sampling scheme , like our approach.

Another advantage of input space geometric sampling schemes is that they are typically \interpretable. In contrast, the sampling criterion based on a parametric model tells us which new sample candidate is currently regarded as most informative. When the basis-functions are not trivial, the reason for this rating is a black box that cannot be understood by human.

In the domain of supervised sampling schemes several approaches are model-based but not strictly \optimal.
For example, \cite{cohn1997minimizing} discussed how for a nonparametric approach such as \emph{locally weighted regression} one can sample either to minimize predictive bias or variance. While any combination of both will result in a \robust sampling scheme, the question was left open how to combine both in order to achieve the true minimization of the joint error components.
Another example for this are \emph{Query by committee} approaches, where candidates are scored depending on the disagreement between several models maintained in parallel (see, for example, \cite{seung1992query}).

While supervised sampling schemes are typically at least weakly \optimal (through a heuristic approximation), they are not necessarily \robust. For example, \cite{douak2013} queries labels where the prediction error is considered largest, which is reasonable under homoscedastic noise assumptions. 
When accidentally applied in a heteroscedastic scenario, as soon as the true function is coarsely rendered, the sampling process becomes degenerate, as it collapses to the point of largest noise level.

In the following, we will set our focus of the discussion on the one hand to the category of approaches that are \optimal, \robust and nonparametric, like our proposed active learning framework. In this category, \cite{goetz2018active} most recently proposed an active learning strategy that is based on purely random trees.
On the other hand, we will discuss the wavelet-based approach of \cite{bull2013spatially},
since it provides theoretical guarantees to exceed the minimax-convergence rate of nonparametric active learning approaches that operate on a general function space like $\diffableFunctions{\inputSpace}{\LPSorder+1}$ \citep{willett2006ALrates}. They achieve this by making use of a more sophisticated segmentation of this very function space, where the segmentation is done with respect to local complexity of the function.
Note that the approach of \cite{bull2013spatially} is not \robust, as it assumes homoscedastic noise.

\subsubsection{A random tree/forest active learning approach}
\label{subsubsec:goetz}
\cite{goetz2018active} choose a Mondrian tree \fMT as the underlying model of their active learning approach.
The tree is first set up by partitioning the input space into cuboids.
Then, \fMT is a constant mean prediction over each such cuboid.
Goetz et al.~found the following law for the lifetime hyperparameter $\lambda(n^{1/(2+d)}-1)$ to be optimal. It controls the expected number of splits of the cuboids. 
As soon as the random tree is set up, there is a model induced bias which is unaffected by the training sample. Accordingly, Goetz et al.~minimize the remaining variance for their active learning scheme, which gives an optimal training density \pMT that is proportional to the root \MSE times test input marginal $q$ on each respective cuboid.
Given an initial, randomly chosen training set $\sim q$, the criterion can be estimated simply, using the cuboid-wise empirical variances and sample counts.
Here, Goetz et al.~suggest to set the initial sample size to half the terminal sample size.
In the experiments, we will compare to this supervised, nonparemetric active learning approach, since it fulfills the requirements we have imposed in this work: It is \optimal and \robust, and will likely provide \modelagnosticProperty as well.
A single tree is quite a rudimentary model in terms of prediction performance.
By setting up a Mondrian forest, that is, an ensemble of random trees where we average their responses, the performance greatly improves.
The main shortcoming of a single Mondrian tree is the default prediction, whenever a cuboid contains no training data at all. When combining several trees, we only need to fall back to the default, when for a test input the associated cuboids of all respective trees are simultaneously empty.
Therefore, by increasing the number of trees, we can deploy a larger $\lambda$ in the law of the lifetime hyperparameter.
Regarding active learning, Goetz et al.~suggest to sample from \pMF, the average of the optimal densities of each individual tree.
Note that this heuristic sampling scheme for the Mondrian forest does not preserve theoretical guarantees of optimality. 

\subsubsection{A wavelet-based approach}
\label{subsubsec:bull}
It is well known that a nonparametric active learning approach such as ours that operates on a general function space like $\diffableFunctions{\inputSpace}{\LPSorder+1}$
is subject to the nonparametric minimax-convergence rate, given by $\upperBoundedInProbability{\noexpand\Big}{n^{-\frac{2(\LPSorder+1)}{2(\LPSorder+1)+d}}}$, which is already obtained by $\randomTestSampling$. In particular, it is impossible for nonparametric active learning to increase in rate beyond $\randomTestSampling$ over this function space.
Despite this, \cite{bull2013spatially} have shown that for a more sophisticated segmentation of the function space faster learning rates can be theoretically obtained on adequate subsets of $\diffableFunctions{\inputSpace}{\LPSorder+1}$:
Indeed, the model with best, uniform performance over such a subset can exceed the performance of the model that is best over the whole space $\diffableFunctions{\inputSpace}{\LPSorder+1}$.
This segmentation is done with respect to local complexity of the function.
While they do not come up with an estimate for that enhanced learning rate, they provide an active learning scheme for a wavelet-based prediction model that will asymptotically converge at this rate:
The amplitude of wavelet coefficients are used to rank sub-intervals of the input space. Then the training set is refined by deterministically adding samples proportional to the reciprocal of these ranks.
Furthermore \cite{bull2013spatially} state that the subspace of functions for which no increased rate can be obtained is negligible.

Bull et al.~assume a homoscedastic noise structure (making it a non-robust approach) and a uniform test density (leaving it sub-optimal for other test distributions). Also the generalization beyond 1-dimensional input data remains unclear.
Yet, the improvement of Bull's approach beyond the minimax-convergence rate that our proposed active learning framework underlies can be a strong advantage over our sampling scheme.
We will therefore also compare the work of Bull et al.~to our framework in \secref{subsec:doppler} on a dataset of inhomogeneous local function complexity that we adopted from their paper.
Our experiment shows that the theoretical rate increase -- even though appealing in the asymptotic limit -- does not manifest itself in this example at moderate training sizes, which active learning is concerned with. And so, our active learning framework compares favorably Bull's approach in this case.

\section{Experiments}
\label{sec:experiments}
In the following experiments we will visualize our theoretical results, compare to approaches from related work and underpin our claim that our active learning scheme preserves its meaningfulness across models. In particular, we compare to a wavelet-based and a random tree based active learning approach, where we adopt for each approach a toy-regression problem on which they are designed to work. By means of a fair comparison, and to demonstrate the \modelagnosticProperty of our approach, we will furthermore apply the actively chosen training sets to an \emph{\RBF network}, respectively a \emph{random forest} model.

\subsection{Measuring the Active Learning Performance}
\label{subsec:evalAL}

As already mentioned in \secref{subsec:relatedWork_ALregression} it is impossible for our active learning scheme to exceed the convergence rate of \randomTestSampling\footnote{Recall that we have defined \randomTestSampling as i.i.d.~sampling from the test distribution in \secref{sec:intro}.}, since \LPS is based on the general function space $\diffableFunctions{\inputSpace}{\LPSorder+1}$.
More precisely, for any fixed training density $p$ with $\bm{X}_n^{} \sim p$, we know that, according to \thmref{thm:asymptoticBiasVarianceForLPS}, there is a constant $C_{\LPSorder}^{p}$ with respect to $n$ such that
\begin{align*}
 \textstyle\MISE_{\LPSorder}^{}\left(q| \bm{X}_n^{}\right) &=\textstyle \mathop{\mathlarger{\int}}_{\hspace*{-5pt}\inputSpace} \MSE_{\LPSorder}^{}\left(x, \LOBfunctionOfLPS{\LPSorder}{n}(x) | \bm{X}_{n}\right) q(x) dx\\
 &=\textstyle C_{\LPSorder}^{p} n^{-\frac{2(\LPSorder+1)}{2(\LPSorder+1)+d}}(1 + \convergenceInProbability{}{1}),
\end{align*}
which holds, in particular, for $p \equiv q$.

Yet, we still can benefit from active learning:
We can reduce the required sample amount to reach a certain prediction performance by a constant percentage, over \randomTestSampling.
For a regression model $\widehat{f}$ that cannot increase in rate, and a training density $p$,
let us define by $\relSampleSize{\widehat{f}}{p} > 0$ the relative required sample size such that for all large n and $n' = \relSampleSize{\widehat{f}}{p} n$ we obtain
$\textstyle\MISE\left(q,\widehat{f}| \bm{X}_{n}\right) \equiv \MISE\left(q,\widehat{f}| \bm{X}_{n'}'\right)$,
where $\bm{X}_{n'}' \sim p$ and $\bm{X}_{n} \sim q$.
The smaller $\relSampleSizeSymbol$ is, the better is the training density $p$. That means, for a reasonable active learning scheme $\relSampleSizeSymbol \leq 1$ should hold.

In the case of \LPS, it is $\textstyle\relSampleSize{\fOpt{\LPSorder}}{p} = \big[C_{\LPSorder}^{p}\big/C_{\LPSorder}^{q}\big]^\frac{2(\LPSorder+1)+d}{2(\LPSorder+1)}$.
Asymptotically, we can express the ratio as 
$\textstyle\left.C_{\LPSorder}^{p} \big/ C_{\LPSorder}^{q}\right. = \left.\MISE_{\LPSorder}^{}\left(q| \bm{X}_n'\right)\big/\MISE_{\LPSorder}^{}\left(q| \bm{X}_n^{}\right)\right.$
, such that we can estimate it as the mean error ratio between $\bm{X}_{n}' \sim p$ and $\bm{X}_{n} \sim q$ at large, equal training sizes.
We then average this estimate over several repetitions.
\begin{myRemark}
The predictions of the Mondrian tree and forest models are locally constant, and so they share the $\relSampleSizeSymbol$ of \LPS of order $Q = 0$. That is,
\[\textstyle \relSampleSize{\fMT}{p} = \left[\frac{\MISE\left(q, \fMT| \bm{X}_n'\right)}{\MISE\left(q, \fMT| \bm{X}_n^{}\right)}\right]^\frac{2+d}{2},\;\;\text{and}\quad\relSampleSize{\fMF}{p} = \left[\frac{\MISE\left(q, \fMF| \bm{X}_n'\right)}{\MISE\left(q, \fMF| \bm{X}_n^{}\right)}\right]^\frac{2+d}{2}.\]
\end{myRemark}

\subsection{Two Dimensional Heteroscedastic Toy-Data}
\label{subsec:heteroscedastic}
We start with a two-dimensional toy-example, following the experiments of \cite{goetz2018active} in order to show that our approach generalizes to multivariate problems. Furthermore, we compare to the active learning performance of Goetz et al.
Note that the approach of \cite{bull2013spatially} does not naturally generalize to higher dimensions, such that we do not compare with their approach in this first experiment.

All reported values of relative required sample sizes $\relSampleSizeSymbol$ with respect to \randomTestSampling are calculated as described in \secref{subsec:evalAL}.
Let
\begin{align*}
p(y | x) &=  \Gauss{y}{f(x)}{v(x)},\quad\text{where}\\
f(x) &= C\sin\Big(2\pi\frac{\pnorm{x}{}}{\sqrt{d}}\Big), \quad\text{and}\quad v(x) = \left\{\begin{array}{lr}
        25, &\forall 1\leq i\leq d: x_i^{} > 1/d,\\
        1, & \text{else},
        \end{array}\right.
\end{align*}
for $C = 100$, $x\in\inputSpace = [0,1]^d$ and $d = 2$. The test inputs $x \sim q = \uniformDist{\inputSpace}$ are uniformly distributed.
An example of the dataset is given to the left in \figref{fig:heteroscedastic_DatasetBandwidthsNoiseOptimalDensities}.
\fig{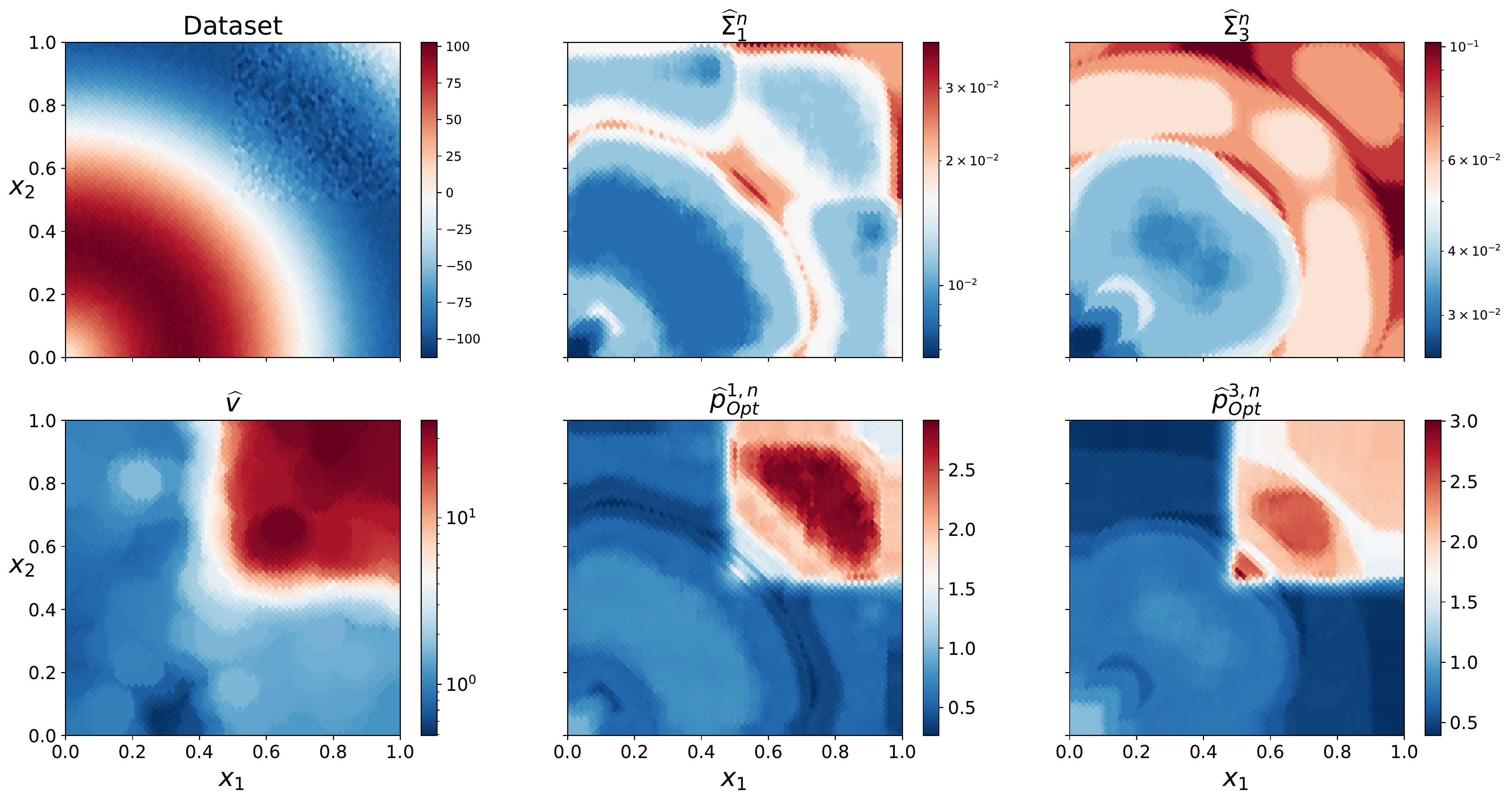}{1}{The heteroscedastic experiment: (Left) Exemplary dataset and initial noise estimate at $n_0=2^{10}$, using \eqref{eq:robustLocalNoiseStandardDeviationAdaptedMultivariate}. Our proposed terminal \LOB and optimal training density estimates based on \LLS (middle) and \LCS (right), using equations \eqref{eq:LOBlepskiEstimate} and \eqref{eq:optimalSamplingFiniteLPS}.}{0.8cm}{0.0cm}

First, note that $f\in\diffableFunctions{(0,1)}{\infty}$, such that we could apply our approach with arbitrary degree \LPSorder. We will delimit our discussion to the cases $\LPSorder = 1, 3$, where we refer to the respective \LPS model as local linear smoothing (\LLS) for $\LPSorder=1$ and local cubic smoothing (\LCS) for $\LPSorder=3$.
The experimental results are based on $20$ repetitions.
Applying the Gaussian kernel $k$, we implement our proposed active learning procedure as described in \secref{subsec:algoSummary}.
We start with $n_0^{} = 2_{}^{10}$ equidistantly spaced samples
and choose the hyperparameters $\kappa = 2.58, \supportInteriorStdDevs = 1$, and the constants $C_{v}, C_{s}, C_{\sigma}, C_{\delta}$ in order to obtain $l_{n_0} = 2^6$, $s_{n_0} = 2^\frac{2}{3}$, $\underline{\sigma}_{n_0} = 5\times 10^{-2}$ and $\delta_{n_0} = 0.1$.
To the left in \figref{fig:heteroscedastic_DatasetBandwidthsNoiseOptimalDensities} we show an example of our noise estimate over the initial training set.

The proposed optimal sampling densities for \LPS of order $1$ and $3$ are shown in the middle, respectively to the right panel in \figref{fig:heteroscedastic_DatasetBandwidthsNoiseOptimalDensities}. The prediction performance over increasing training size can be seen to the left in \figref{fig:heteroscedastic_errors_LPS_MF}, where we compare our active sampling scheme to \randomTestSampling.
As this figure suggests, we find consistent sample savings within the \LPS model class, which we can quantify via estimation of the relative required sample size with respect to \randomTestSampling, as described in \secref{subsec:evalAL}.
We obtain the values $\textstyle\relSampleSize{\fOpt{1}}{\approxPOptLPS{1}{n}} = 0.77\pm0.05$ and $\textstyle\relSampleSize{\fOpt{3}}{\approxPOptLPS{3}{n}} = 0.66\pm0.07$, which means that we can save about $23$, respectively $34$ percent of samples, using the \LPS model of order $Q=1$ and $Q=3$, when sampling according to our proposed optimal density estimate $\approxPOptLPS{\LPSorder}{n}$ instead of sampling from the test distribution.
Hence, in accordance with \thmref{thm:OptimalSampling}, this shows the superiority of our proposed active learning framework.

\fig{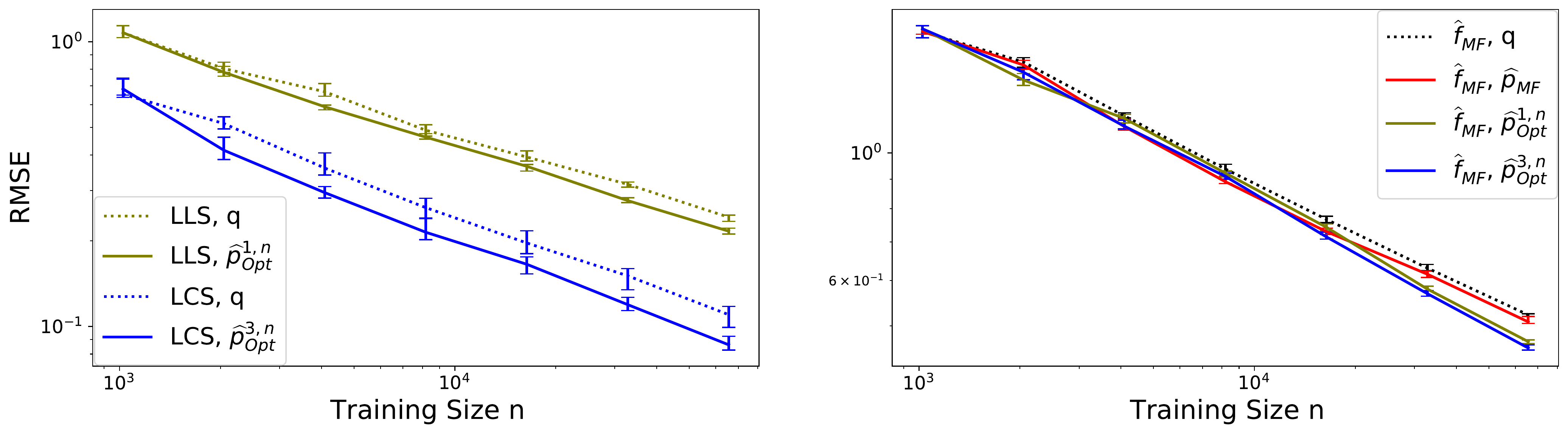}{1}{The heteroscedastic experiment: (Left) The median \RMSE obtained by \LLS and \LCS at several training sizes, comparing \randomTestSampling to the respectively proposed optimal sampling. (Right) The \RMSE of Mondrian forests under \randomTestSampling, according to the active learning approach \pMF of \cite{goetz2018active} and according to the optimal sampling scheme of \LLS and \LCS. The error bars show the $95\%$ confidence interval.}{0.8cm}{0.0cm}

In order to analyze the transferability of our sampling scheme, we now combine our proposed sampling scheme with the Mondrian forest model.
As a preconsideration, let us take an isolated look at the Mondrian tree and forest model when applying the active learning approach of Goetz et al.: 
We found $\lambda = 2.5$ in the law of the lifetime hyperparameter of the Mondrian tree to work well. As described in \secref{subsubsec:goetz}, we draw half of the terminal training size at random from $q$, from which the optimal density \pMT for the tree is estimated. The remaining samples are drawn subsequently in a way that approaches this density.
The average estimated optimal density \pMT of the Mondrian tree model can be seen to the left in \figref{fig:heteroscedastic_OptimalDensities_MT_MF}.
\pMT is larger, where the noise level is higher, but also in the steep regions of the function $f$, resulting from the locally constant modelling of the Mondrian tree.

\fig{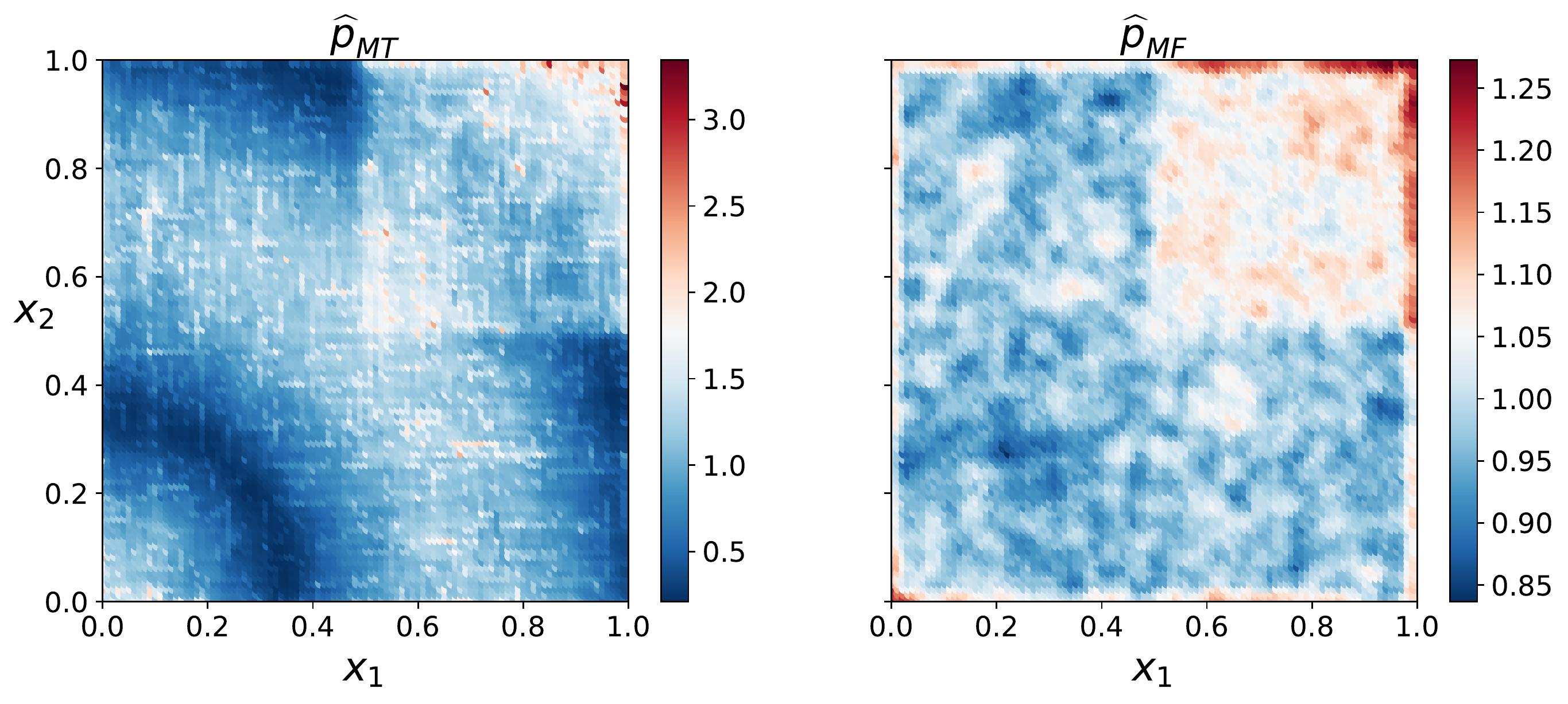}{0.65}{The heteroscedastic experiment: The terminal training densities by the active learning approach of \cite{goetz2018active}, for the Mondrian tree (left) and Mondrian forest (right).}{0.4cm}{0.0cm}

We also implement a Mondrian forest model \fMF as described in \secref{subsubsec:goetz} with $\lambda = 7$ and $100$ Mondrian trees.
Noting that the prediction performance of the forest increases with the number of trees, the performance has almost converged at $100$ trees and computation starts to become intractable by increasing the number further.
Goetz et al.~propose to use \pMF, which is the average \pMT over the individual trees of the forest, as the active learning density. Recall that other than \pMT for the Mondrian tree being provably \optimal, applying \pMF for the Mondrian forest is a heuristic.

To the right of \figref{fig:heteroscedastic_OptimalDensities_MT_MF} we show the average estimated training density \pMF that is associated to the Mondrian forest model. In contrast to the Mondrian tree model, the Mondrian forest shows much lower bias due to the larger lifetime hyperparameter $\lambda$. Therefore the obtained density is mostly driven by the local noise level.

In \figref{fig:heteroscedastic_MondrianTree_vs_MondrianForest} we see that the Mondrian tree performance increases significantly with the optimal density \pMT of Goetz et al., but the absolute level of the performance of \fMT is much lower compared to more sophisticated prediction models, as can be seen in  \figref{fig:heteroscedastic_errors_LPS_MF}.

\fig{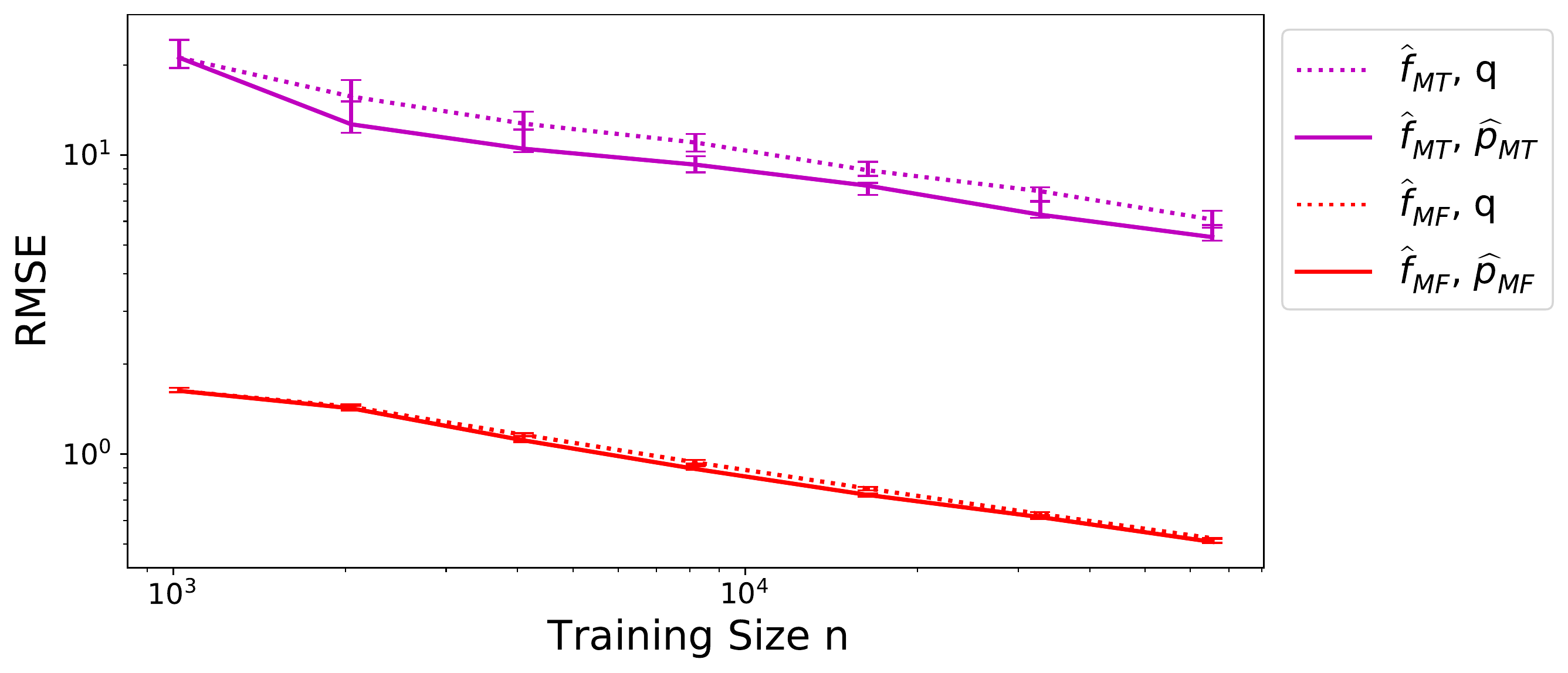}{0.6}{The heteroscedastic experiment: The median \RMSE obtained by Mondrian trees \fMT and forests \fMF at several training sizes, comparing \randomTestSampling to the active learning strategies \pMT and \pMF of \cite{goetz2018active} for \fMT, respectively \fMF. The error bars show the $95\%$ confidence interval.}{0.4cm}{0.0cm}

When estimating the relative required sample size with respect to \randomTestSampling, we obtain $\textstyle\relSampleSize{\fMT}{\pMT} = 0.61 \pm 0.05$ and $\textstyle\relSampleSize{\fMF}{\pMF} = 0.90 \pm 0.02$. The sample savings for the Mondrian tree are substantial, as expected, because \pMT is provably \optimal for this model. In contrast, the sample savings for the Mondrian forest, although being significant, are much weaker.
Recall that \pMF was heuristically designed for the Mondrian forest by Goetz et al. And while this heuristic is successful in terms of model transferabilty, we expect room for further improvement. 

Finally, we combine our sampling scheme with the Mondrian forest model.
The error curves for the Mondrian forest model in combination with different sampling schemes are shown to the right in \figref{fig:heteroscedastic_errors_LPS_MF}.

We observe significant sample savings over \randomTestSampling for the Mondrian forest in combination with our proposed active sampling scheme, with $\textstyle\relSampleSize{\fMF}{\approxPOptLPS{1}{n}} = 0.68 \pm 0.02$ and $\textstyle\relSampleSize{\fMF}{\approxPOptLPS{3}{n}} = 0.62 \pm 0.02$ for our proposed optimal training density estimates with $Q=1$, respectively $Q=3$.
This first of all provides evidence that our proposed active sampling scheme is \modelagnostic.

Furthermore, recalling that it was $\textstyle\relSampleSize{\fMF}{\pMF} = 0.90 \pm 0.02$, both values beat the active learning performance of \pMF by far. In fact, we can save about $24$, respectively $31$ percent of samples when applying our active learning framework with $Q=1$ and $Q=3$ instead of \pMF, which was specifically crafted for the Mondrian forest.

\subsection{Doppler Function}
\label{subsec:doppler}
As described in \secref{subsubsec:bull}, 
the wavelet-based approach by \cite{bull2013spatially} is designed to adapt to inhomogeneous complexity under homoscedastic noise $v(x) \equiv \text{v}$, assuming a uniform test density $q\sim\uniformDist{\inputSpace}$. In theory, their approach is asymptotically capable of exceeding the convergence rate of our proposed active learning framework, which is why we will compare both, qualitatively and quantitatively.
We will also include the Mondrian forest approach of \cite{goetz2018active} from the first experiment, noting that it is not a promising candidate for a dataset of inhomogeneous complexity:
The Mondrian forest model has no adaption parameter in the sense of local bandwidth, as opposed to the wavelet and our approach.

Bull et al.~used the Doppler function (see, for example, \cite{donoho1994ideal}) as a prototype where they expect such an increase in rate due to the strong inhomogeneous complexity of the function to learn. We adopt their experimental specification: For $x \in \inputSpace = [0,1]$, let
\begin{align*}
&p(y | x) = \Gauss{y}{f(x)}{1},
\quad f(x) = C\sqrt{x(1-x)}\sin\left(2\pi(1 + \epsilon)\middle/(x + \epsilon)\right),
\end{align*}
where $\epsilon = 0.05$, $C$ is chosen such that $\pnorm{f}{2} = 7$ and $\Gauss{\cdot}{\mu}{\sigma^2}$ denotes the Gaussian distribution with mean $\mu$ and variance $\sigma^2$.
\figref{fig:dopplerExperiment_dataset_bandwidths} (Left) shows an example dataset.
\fig{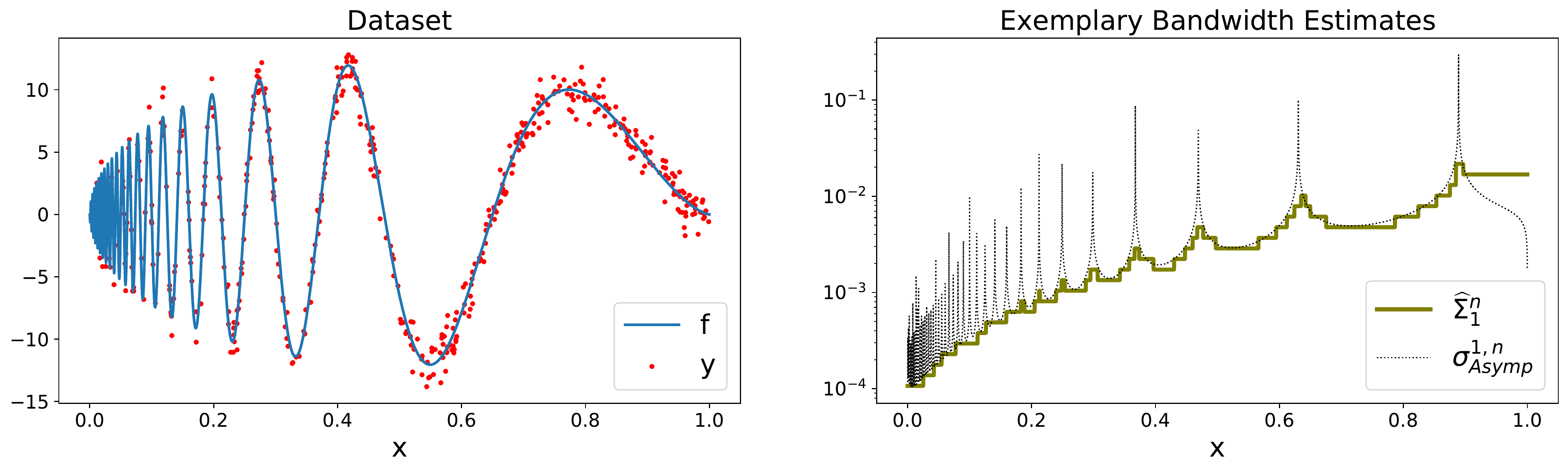}{1}{The Doppler experiment: (Left) Exemplary dataset. (Right) Asymptotic and estimated locally optimal bandwidths, using equations \eqref{eq:asymptoticIsotropicLOBofLPS} and \eqref{eq:LOBlepskiEstimate}.}{0.8cm}{0.0cm}
In this experiment we know that $q \sim \uniformDist{\inputSpace}$ and the problem is homoscedastic. Applying the Gaussian kernel $k$, we implement our proposed active learning procedure as described in \secref{subsec:algoSummary}.
We start with $n_0^{} = 2_{}^{9}$ equidistantly spaced samples
and choose the hyperparameters $\kappa = 1.96, \supportInteriorStdDevs = 2$, and the constants $C_{s}, C_{\sigma}, C_{\delta}$ in order to obtain $s_{n_0} = 2^\frac{2}{3}$, $\underline{\sigma}_{n_0} = 1.5\times10^{-3}$ and $\delta_{n_0} = 0.1$.
An example of our \LOB estimate for $\LPSorder = 1$ is given to the right in \figref{fig:dopplerExperiment_dataset_bandwidths}.

While again $f\in\diffableFunctions{(0,1)}{\infty}$, we will delimit our discussion to the cases $\LPSorder = 1, 3$.
The experimental results are based on $30$ repetitions.
In \figref{fig:dopplerExperiment_LPSOpt_P13} we show the achieved performance of both cases, when either sampling from the test distribution $\sim q$, or when sampling according to the respective optimal training density estimate, which is plotted to the left.

\fig{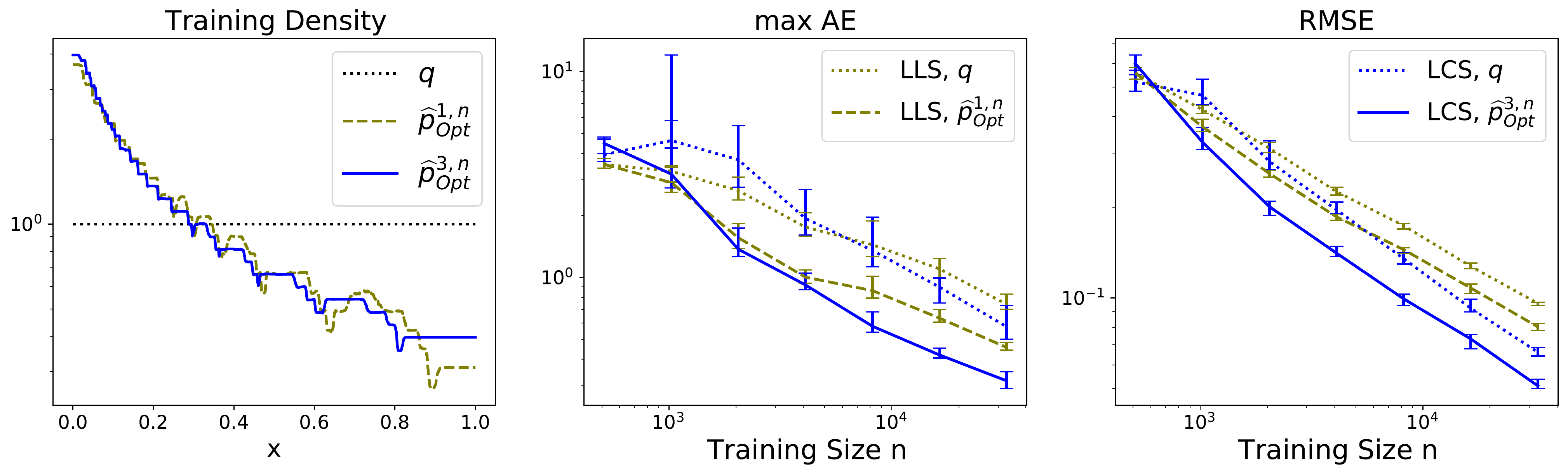}{1}{The Doppler experiment: (Left) The optimal training density estimates of our approach for the polynomial degrees $\LPSorder = 1, 3$ at training size $n = 2^{15}$. The median maximal absolute error (middle) and \RMSE (right) obtained by \LPS with the respective active sampling scheme and \randomTestSampling ($\sim q$) at various training sizes. The error bars show the $95\%$ confidence interval.}{0.8cm}{0.0cm}

Like in the first experiment, we calculate the relative required sample size with respect to \randomTestSampling, as described in \secref{subsec:evalAL}. Again, confirming our result in \thmref{thm:OptimalSampling}, we observe significant sample savings of $\textstyle\relSampleSize{\fOpt{1}}{\approxPOptLPS{1}{n}} = 0.64\pm0.02$ and $\textstyle\relSampleSize{\fOpt{3}}{\approxPOptLPS{3}{n}} = 0.53\pm0.03$.

\fig{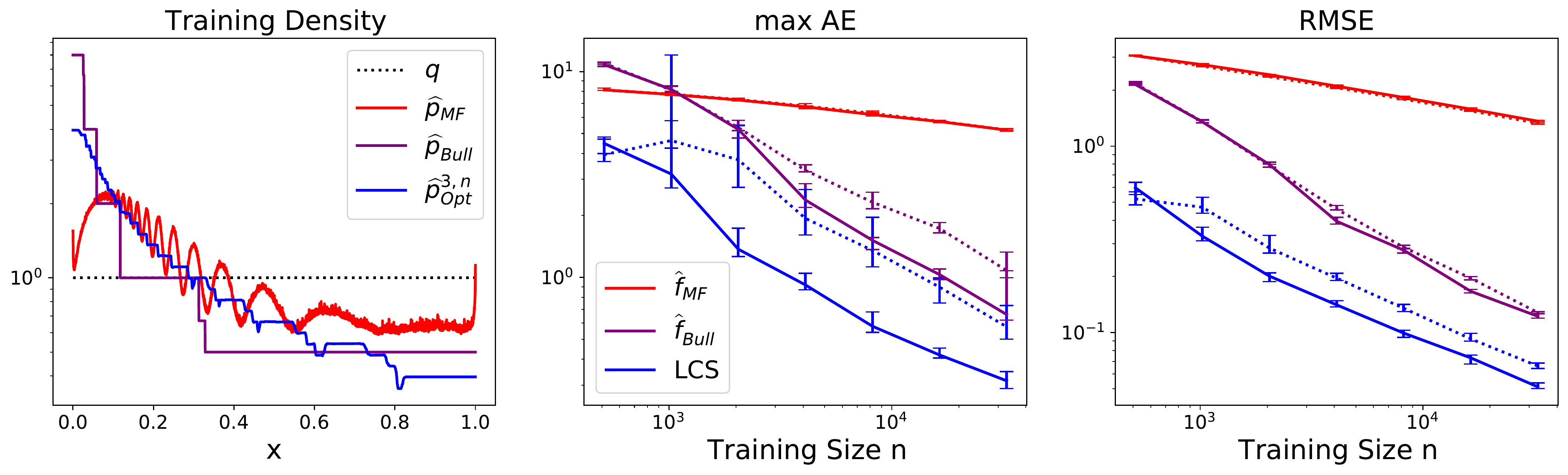}{1}{The Doppler experiment: (Left) The terminal training density of Goetz' (red), Bull's (purple) and our approach (blue). The median maximal absolute error (middle), respectively \RMSE (right) obtained by Goetz', Bull's and our model with the respective active sampling scheme and \randomTestSampling ($\sim q$) at various training sizes. The error bars show the $95\%$ confidence interval.}{0.8cm}{0.0cm}

For a comparison to the active sampling approach of Bull et al., we implement their method in \emph{python}, based on the \emph{pywt} package, and using the Daubechies-wavelets of filter length 8 (DB8).
The resulting training density can be seen to the left in \figref{fig:dopplerExperiment_LCSOpt_vs_BullOpt_vs_GoetzOpt}, together with the densities of Goetz' and our approach. We observe that all approaches spend more samples to the left -- as expected -- where Bull's and our approach concentrate the more samples, the higher the local function complexity becomes. Here, the density of Bull's approach increases steeper to the left of the input space.

Furthermore, we observe to the middle and right of \figref{fig:dopplerExperiment_LCSOpt_vs_BullOpt_vs_GoetzOpt}
that the \LPS model class shows better performance than the wavelet-based approach, especially at smaller samples size.
At sample size $2^{12}$ the wavelet-based approach allows for enough flexibility to adapt locally, which leads to the sudden dissociation of the learning curves of the wavelet-based approach under Bull's sampling scheme compared to \randomTestSampling in \figref{fig:dopplerExperiment_LCSOpt_vs_BullOpt_vs_GoetzOpt}.
After that, both approaches follow almost parallel learning curves.
This indicates that, in practice, there are regression problems where the theoretically achievable enhancement of the learning rate (with a function space segmentation such as in Bull et al.) is negligible. While it is theoretically appealing to achieve a better learning rate in the asymptotic limit, active learning is usually concerned with small to moderate sample sizes. In this regime, a constant percentage of sample savings, as achieved by our approach, can be of greater benefit.

Note that the actual \RMSE decay law of the wavelet-based model is unknown. Yet, since the \RMSE of the wavelet approach decays at least as fast as for \LCS, we can upper bound the active learning performance of Bull's approach by calculating the relative required sample size with respect to \randomTestSampling analogously to $\textstyle\relSampleSize{\fOpt{3}}{\approxPOptLPS{3}{n}}$, which gives $\textstyle\relSampleSize{\fBull}{\pBull} = 0.82\pm0.03$.
Thus, we can say that we save about $47$ percent of samples with our active learning framework compared to \randomTestSampling on the \LCS model, whereas we save at most $18$ percent with Bull's active learning approach on the wavelet-based model. 

Since both models are not directly comparable, we cannot deduce from these numbers which sampling scheme is better.
In this regard, we will adopt a radial basis function network as a regression model from the domain of neural network learning, for which both active learning approaches are not optimized.
We will train the \RBF-network, using \randomTestSampling, as well as the actively chosen training sets of Bull's, Goetz' and our approach. In addition, we assume a small validation dataset of size $2^{10}$ to be given.
A successful outcome first of all underpins the claimed \modelagnosticProperty of our proposed active learning framework. Second, it allows for a fair comparison of the three approaches.

The \RBF-network \citep{moody1989fast} is implemented in \emph{PyTorch} \citep{paszke2019pytorch} with a few hyperparameters:
Its \RBF-layer consist of $N$ Gaussian basis function nodes with bandwidths $\sigma_i$ and centers $\mu_i$, followed by a linear layer with weights $w$. Formally,
\[\widehat{f}_\text{NN}^{}(x) = w_0 + \mySum{i=1}{N}w_i k^{\sigma_i}(\mu_i,x).\]
The \RBF-nodes are initialized with $\sigma_i = \sigma_0\left(\frac{n}{n_0}\right)^{-\frac{1}{5}}$, and centers $\mu_i$ (sub-)sampled from the current training dataset.

We apply the training \MSE as the training objective. In addition, inspired by Lepski's method, we favor larger local bandwidths over smaller ones when they perform similarly. Therefore we add the term $-\lambda \mySum{i=1}{N}\log\{\sigma_i\}$ to the objective to penalize small bandwidth choices, where we set the penalty factor $\lambda = \lambda_0\frac{N}{2^9}\left(\frac{n}{n_0}\right)^{-\frac{1}{2}}$.
The training is then done, using the \emph{AdamW}-optimizer \citep{loshchilov2018decoupled} with a weight decay of $\nu_0\left(\frac{n}{n_0}\right)^{-\frac{1}{2}}$ and mini-batches of the training data of size $B = \lceil B_0\left(\frac{n}{n_0}\right)^\frac{1}{2}\rceil$.
We initialize a shared learning rate factor of $l = 10^{-2}$ that we gradually decrease towards $10^{-7}$ whenever the validation error gets stuck.
Using this shared learning rate factor, we apply individual initial learning rates for the linear weights $l_w = l$, the bandwidths $l_\sigma = 0.5 l$ and the centers $l_\mu = 0.1 l$.

For all sampling schemes we apply the same set of hyperparameters
$B_0 = 4$, $\sigma_0 = 7\times 10^{-4}$, $\lambda_0 = 1.6$ and $\nu_0 = 8\times 10^{-3}$, where we only choose the number of \RBF-nodes $N$ individually: 
For Bull's as well as our approach, $N = 2^{9}$ works best, whereas for all other approaches $N = 2^{11}$ works best. The heavy local complexity towards zero is only recognized properly by Bull's and our approach such that the number of \RBF-node centers near zero is large enough at this smaller total number of nodes.

\fig{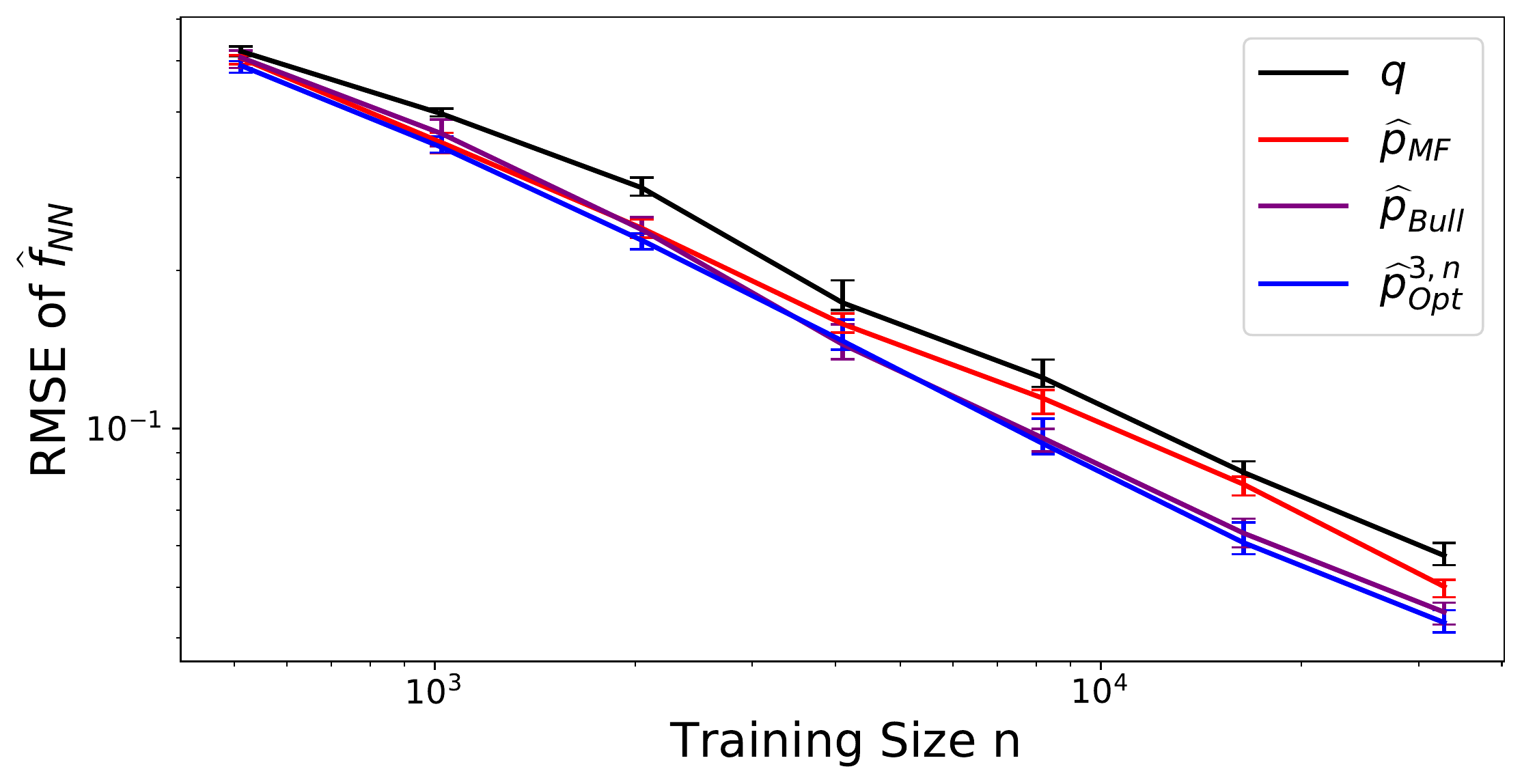}{0.5}{The Doppler experiment: The median \RMSE obtained by a radial basis function network $\widehat{f}_\text{NN}^{}$ at several training sizes, when trained with \randomTestSampling, respectively Bull's, Goetz' and our proposed active sampling scheme. The error bars show the $95\%$ confidence interval.}{0.4cm}{0.0cm}
In \figref{fig:dopplerExperiment_NNmodel_unif_vs_LCSOpt_vs_BullOpt} we observe that all approaches, Bull's, Goetz' and ours, outperform \randomTestSampling for the \RBF-network model $\widehat{f}_\text{NN}^{}$.
This first of all underpins the transferability of all these sampling schemes.
To the left in \figref{fig:dopplerExperiment_LCSOpt_vs_BullOpt_vs_GoetzOpt} we have seen that Bull's and our sampling scheme act qualitatively similar on this dataset.
Therefore, as expected, both approaches also behave quantitatively similar.
For an exact quantitative analysis, note that the \MISE of the \RBF-network model follows the decay law $\sim n^{-1}$. Accordingly, we can calculate the required sample size relative to \randomTestSampling, for which we obtain the values
{$\textstyle\relSampleSize{\widehat{f}_\text{NN}^{}}{\pMF} = 0.79\pm0.05$},
{$\textstyle\relSampleSize{\widehat{f}_\text{NN}^{}}{p_\text{Bull}^{}} = 0.62\pm0.05$} and
{$\textstyle\relSampleSize{\widehat{f}_\text{NN}^{}}{\approxPOptLPS{3}{n}} = 0.61\pm0.04$}.
Thus, Bull's and our approach are significantly better than Goetz' approach in terms of transferability on this dataset. Furthermore, our sampling scheme compares favorably but not significantly to Bull's sampling scheme.

\subsection{Discussion}
\label{subsec:expEval}
We have chosen the toy-examples above from the domain of regression problems, for which the approach of Goetz et al., respectively Bull et al.~is designed to work.
Since these approaches and our active learning framework are \modelagnostic, it is not surprising that all resulting sampling schemes behave qualitatively similar. Yet, we obtained equal or better across-model performance with our proposed sampling scheme.
Additionally, our approach is more flexible, regarding the applicable class of regression problems: The approach of Bull et al.~assumes a uniform test distribution and homoscedastic noise, and it lacks a straight-forward extension to multivariate problems $(d>1)$. Goetz' approach has degrading performance for inhomogeneously complex regression problems, since its underlying model features no local adaptivity parameter such as \LOB. In contrast, our active learning framework does not suffer from these limitations, but incorporates these properties in the optimal sampling scheme instead.

As we have already discussed in \secref{subsubsec:bull}, Bull's sampling scheme may feature an \MISE decay law superior to our approach. Hence, asymptotically it should exceed the performance of our approach. But the experiment suggests that this will not occur at reasonable training sizes.
The advantage of Goetz' approach is that the model class of random trees is better suited for high-dimensional multivariate problems  -- a property that the implementation of our theory lacks: In this paper, we construct the optimal training density from pointwise estimators of \LOB and the noise level. In future work, this can be remedied by modelling these components as functions.

\section{Conclusion and Outlook}
\label{sec:conclusion}
The goal of our work was to reconcile the advantages of model-free active learning approaches in regression such as \robustProperty and \modelagnosticProperty, with the advantages of model-based active learning approaches such as \optimalProperty.
An active sampling scheme with these properties is ideal to construct a larger training set, when we face a regression problem for which the state-of-the-art is still evolving due to, for example, scarce domain knowledge.

As an ansatz to achieve this goal, we consider local polynomial smoothing (\LPS), a nonparametric model class with minimal assumptions on the labels, which can be regarded as almost model-free. In terms of locally optimal bandwidths, we chose the mean integrated squared error of \LPS as our objective, which we aim to minimize with respect to the training dataset.
Making use of the asymptotic behavior of the objective, as well as the isotropic optimal bandwidths, the optimization could be shown to be analytically solvable.
The result is obtained in closed-form in terms of the optimal training density $\pOptLPS{\LPSorder}{n}$, which nicely factorizes the influence of problem intrinsic properties on the optimal sample demand, that is, local function complexity, noise variance and test relevance.
This makes our sampling scheme transparent and \interpretable, a desired property in critical real-world applications.
Additionally, the sampling process $\bm{X}_n^* \sim \pOptLPS{\LPSorder}{n}$ is \stationary, which enables batch sampling and is advantageous when on demand label annotation is a bottleneck.

Using Lepski's method for the estimation of isotropic, locally optimal bandwidths, we derived a practical implementation of our theory.
In experiments, we then compared to related work. Furthermore, we provided evidence that our proposed sampling scheme is \modelagnostic by 
applying our actively sampled training data to other model classes. In particular, we observed a consistent performance increase over \randomTestSampling for a radial basis function network and a Mondrian forest model. Moreover our active learning framework compared favorably to state-of-the-art nonparametric active learning approaches.

One possible way of generalizing our theory is to consider a non-isotropic candidate set $\SigmaSpace \subseteq \posDefSet{d}$, over which we build our objective \eqref{eq:ALobjective}.
A straightforward extension of the proof of our theory from the isotropic case would require the existence of an explicit asymptotic form of \LOB.
While this existence can be guaranteed under mild assumptions in the isotropic case, it can not in the non-isotropic case -- as we have indicated in the beginning of \secref{sec:theoryIsotropicOptimalSampling}:
In particular, the crucial assumption is that the leading terms of bias and variance -- as, for example, given in \thmref{thm:asymptoticBiasVarianceForLPS} -- do not vanish, almost everywhere. In the general bandwidth case with $\LPSorder=1$ however, this only holds if we require $f$ to be indefinite at most on a set of measure zero, which we will call the \emph{definiteness assumption}.
This definiteness assumption is a tremendous restriction on $f$, which most multivariate functions will not fulfill.

When dropping the definiteness assumption, $\LOBfunctionOfLPS{\LPSorder}{n}$ from \eqref{eq:LOBDefinition} is not well-defined and, even if we find a minimizing bandwidth $\bandwidth_x$ as in \eqref{eq:minimizingBandwidth}, the theory on its asymptotics is not elaborated, yet.
We discuss these issues that do arise in the general bandwidth case, when in particular not relying on the definiteness assumption, and provide solutions that still make our theory hold in \appref{sec:nonIsotropic}.

In particular, we prove the existence of $\bandwidth_x$ as in \eqref{eq:minimizingBandwidth} under mild conditions, and analyze its asymptotic scaling behavior, which depends on the smoothness of $f$ in $x$.
Here, we also constructed a minimal, controlled 2-dimensional toy-example with local \emph{anisotropic} bandwidths to substantiate our theory on the asymptotic behavior of non-isotropic \LOB.
While we can not guarantee the uniqueness of such an optimal bandwidth, we will show that $\abs{h_n^{-1}\bandwidth_x}^{-1}$ is asymptotically unique, where $h_n^{}$ is the appropriate bandwidth decay rate in $x$, as the training size $n$ grows.
From this point, a straightforward generalization of \defref{def:isotropicComplexityForLPS} to a definition of non-isotropic \LFC as in \defref{def:isotropicComplexityForLPS} emerges, and an again straightforward generalization of the optimal training density \eqref{eq:optimalSamplingFiniteLPS} in \thmref{thm:OptimalSampling} becomes apparent.

Unfortunately, in lack of an estimate to \LOB in the non-isotropic bandwidth case, we cannot apply our proposed active learning framework in practice at this point.
Yet, we would like to emphasize that our framework can readily be applied, once such an estimate becomes available.

Other future work will dedicate a practical application of our novel framework for domains like quantum chemistry or materials properties, where data is extremely expensive (e.g. \cite{butler2018machine,von2020exploring,keith2021combining}).

\section*{Acknowledgments}
The authors would like to thank Vladimir Spokoiny from the Weierstrass Institute for Applied Analysis and Stochastics (WIAS Berlin) for his helpful suggestions, and Andreas Ziehe from the Berlin Institute of Technology for improving the readability of this work.

\DP was funded by the BMBF project ALICE III, Autonomous Learning in Complex Environments (01IS18049B).

\SN and \KRM were funded by the German Ministry for Education and Research as BIFOLD - Berlin Institute for the Foundations of Learning and Data (ref. 01IS18025A and ref 01IS18037A).
\KRM was also supported by the BMBF Grants 01GQ1115 and 01GQ0850,
under the Grants 01IS14013A-E, 031L0207A-D; DFG under Grant Math+, EXC 2046/1, Project ID 390685689 and  by the Institute of Information \& Communications Technology Planning \& Evaluation (IITP) grants funded by the Korea Government (No. 2017-0-00451, Development of BCI based Brain and Cognitive Computing Technology for Recognizing User’s Intentions using Deep Learning) and funded by the Korea Government (No. 2019-0-00079,  Artificial Intelligence Graduate School Program, Korea University).

All funding sources were not involved in the process of writing and submitting this work.

\bibliographystyle{plainnat}
\bibliography{nonparametricOptimalSampling}

\newpage
\renewcommand{\theHsection}{A\arabic{section}}
\appendix
\newacronym[sort=b01]{ABS}{$\abs{M}, \abs{A}$}{The determinant of a square matrix $M \in \reals{d\times d}$ or the cardinality of a set $A$}
\newacronym[sort=b02]{INT}{$A^\circ$}{The interior of a metric set $A$}
\newacronym[sort=b03]{TRANS}{$M^\top$}{The transpose of a matrix $M \in \reals{m\times n}$}
\newacronym[sort=b04]{TRACE}{$\trace(M)$}{The trace of a matrix $M \in \reals{m\times n}$}
\newacronym[sort=b05]{VEC}{$\vect(A)$}{The arbitrarily, but fixed ordered vectorization of a finite set $A$}
\newacronym[sort=b06]{ONES}{$\ones{d}$}{The vector of ones in \reals{d}}
\newacronym[sort=b07]{DIAG}{$\diag(v)$}{The diagonal matrix with the entries of the vector $v$ on its diagonal}
\newacronym[sort=b08]{ID}{$\idMatrix{d}$}{The identity matrix $\diag(\ones{d}) \in \reals{d\times d}$}
\newacronym[sort=b09]{EXP}{$\E Z$}{The expectation of a random variable $Z$}
\newacronym[sort=b10]{BIAS}{$\Bias$}{The bias function}
\newacronym[sort=b11]{VAR}{$\Var$}{The variance function}
\newacronym[sort=b12]{PUNIF}{$\uniformDist{A}$}{The uniform distribution over a set $A$ of finite measure}
\newacronym[sort=b13]{PNORM}{$\Gauss{\cdot}{\mu}{\sigma^2}$}{The Gaussian distribution with mean $\mu$ and variance $\sigma^2$}
\newacronym[sort=b14]{TDIF}{$D^l_f(x)$}{The tensor of $l$-th order partial derivatives of $f$ in $x\in A$ for $f \in \diffableFunctions{A}{l}$}
\newacronym[sort=b15]{IND}{$\indicatorFunction{A}{x}$}{The indicator function, returning $1$ for $x\in A$ and $0$, else}
\newacronym[sort=b16]{SGN}{$\sign(x)$}{The sign function $2\cdot\indicatorFunction{\nonnegativeReal{}}{x}-1$}
\newacronym[sort=a01]{XSPACE}{\inputSpace}{The input space, being a subset of $\reals{d}$}
\newacronym[sort=a02]{BCS}{\SigmaSpace}{The bandwidth candidate space with $\SigmaSpace \subseteq \posDefSet{d}$}
\newacronym[sort=a03]{XDIM}{$d$}{The dimension of the input space \inputSpace}
\newacronym[sort=a04]{LPSorder}{\LPSorder}{The polynomial order of the \LPS model}
\newacronym[sort=a05]{KER}{$k$}{The RBF-kernel function}
\newacronym[sort=a06]{FUN}{$f$}{The regression function to infer}
\newacronym[sort=a07]{NOISE}{$v$}{The local noise variance function}
\newacronym[sort=a08]{PTR}{$p$}{The density of the training input distribution}
\newacronym[sort=a09]{PTE}{$q$}{The density of the test input distribution}
\newacronym[sort=a10]{ALPHA}{$\alpha = L + \beta$}{The Hölder exponent with $L\in\N$ and $\beta\in(0,1]$}
\newacronym[sort=a11]{RHO}{$\relSampleSizeSymbol$}{The relative required sample size with respect to \randomTestSampling}
\newacronym[sort=c01]{NAT}{$\N$}{The natural numbers}
\newacronym[sort=c02]{REAL}{$\R,\nonnegativeReal{},\positiveReal{}$}{The real, nonnegative real and positive real numbers}
\newacronym[sort=c03]{SYM}{$\symmetricSet{d}$}{The space of symmetric matrices $M \in \reals{d\times d}$ with $M^\top = M$}
\newacronym[sort=c04]{PD}{$\posDefSet{d}$}{The space of positive definite matrices $M \in \symmetricSet{d}$ with $a^\top M a > 0,\forall a\in\reals{d}\setminus\{0\}$}
\newacronym[sort=c05]{POLY}{$\mathcal{P}_{\LPSorder}(\reals{d})$}{The space of real polynomial mappings $\fctn{\mathfrak{p}}{\reals{d}}{\R}$ up to order \LPSorder}
\newacronym[sort=c06]{DIF}{$\diffableFunctions{A,B}{c}$}{The space of $c$-times differentiable mappings $\fctn{f}{A}{B}$}
\newacronym[sort=c07]{DIFreal}{$\diffableFunctions{A}{c}$}{The shorthand for $\diffableFunctions{A,\R}{c}$}
\newacronym[sort=c08]{PSTEP}{$\mathfrak{P}(A)$}{The space of probability step-functions on the set $A$}
\newacronym[sort=c09]{CPROB}{$o_p[a_n]$, $\omega_p[a_n]$, $\theta_p[a_n]$}{The space of random sequences $X_n$ that converge strictly faster, respectively strictly slower or at equal rate in probability to $0$, compared to $a_n$}
\newacronym[sort=c10]{BPROB}{$O_p[a_n]$, $\Omega_p[a_n]$}{The space of random sequences $X_n$ that are upper, respectively lower bounded in probability by $a_n$}
\newacronym[sort=e1]{CV}{$C_{v}$}{The factor of the noise estimate localization}
\newacronym[sort=e2]{CS}{$C_{s}$}{The constant in the adaptive bandwidth step size in Lepski's method}
\newacronym[sort=e3]{CSigma}{$C_{\sigma}$}{The constant in the adaptive smallest bandwidth in Lepski's method}
\newacronym[sort=e4]{CDelta}{$C_{\delta}$}{The constant in the adaptive radius for the stabilization of local estimates}
\newacronym[sort=e5]{Kappa}{$\kappa$}{The confidence interval size factor in Lepski's method}
\newacronym[sort=e6]{SDBC}{\supportInteriorStdDevs}{The standard deviations factor of the kernel for boundary correction}

\glsaddall
\printglossary[style=mystyle,type=\acronymtype]

\setcounter{figure}{0}

\cleartheorem{mythm}
\cleartheorem{mycor}
\cleartheorem{mylem}
\cleartheorem{myprop}
\cleartheorem{mydef}
\cleartheorem{mycond}
\cleartheorem{myconjecture}
\cleartheorem{myRemark}
\newtheorem{mythm}{Theorem}[section]
\newtheorem{mycor}[mythm]{Corollary}
\newtheorem{mylem}[mythm]{Lemma}
\newtheorem{myprop}[mythm]{Proposition}
\newtheorem{mydef}[mythm]{Definition}
\newtheorem{mycond}[mythm]{Condition}
\newtheorem{myconjecture}[mythm]{Conjecture}
\newtheorem{myRemark}[mythm]{Remark}

\renewcommand{\themythm}{\Alph{section}\arabic{mythm}}
\renewcommand{\themycor}{\Alph{section}\arabic{mythm}}
\renewcommand{\themylem}{\Alph{section}\arabic{mythm}}
\renewcommand{\themyprop}{\Alph{section}\arabic{mythm}}
\renewcommand{\themydef}{\Alph{section}\arabic{mythm}}
\renewcommand{\themycond}{\Alph{section}\arabic{mythm}}
\renewcommand{\themyconjecture}{\Alph{section}\arabic{mythm}}
\renewcommand{\themyRemark}{\Alph{section}\arabic{mythm}}

\section{Analysis of the Non-isotropic Case}
\label{sec:nonIsotropic}
In the first part of this paper we discussed the space of isotropic bandwidth, where we could derive our optimal sampling theory in a simplified way: Due to \LOB being well-defined in this case, it exhibits an asymptotic closed-form from which the balancing property as well as its asymptotic scaling immediately followed.
Also the literature on \LOB mostly concentrates on the isotropic bandwidth case for practical reasons: It is not straight-forward to come up with an estimate to non-isotropic \LOB.
Yet, in case that future research comes up with estimates to non-isotropic \LOB, we would like to analyze in advance to what extent the asymptotic theory on \LOB, \LFC and optimal sampling will generalize.

In summary, we will find out that under reasonable assumptions \LOB also exists in the non-isotropic case, but the resulting \MSE will generally decay at a substantially different law. Here, \LOB is not necessarily unique, and hence may not exhibit an asymptotic closed-form. Fortunately, the determinant of \LOB is unique, which enables a straight-forward generalization of isotropic \LFC, as given in \defref{def:isotropicComplexityForLPS}, under some restrictions on the training distribution. Finally, due to the restrictions on the training distribution, we are not able to solve for the optimal training density analytically. However, we can formulate a closed-form heuristic, which is the straight-forward looking generalization of the optimal training density in the isotropic case.

From here on, we assume a non-isotropic bandwidth candidate space $\SigmaSpace \subseteq \posDefSet{d}$, where $d > 1$ as the notion of non-isotropy makes no sense in the 1-dimensional case. Furthermore, we assume a symmetric kernel $k$, and, for convenience, we restrict ourselves to the well-known local linear smoothing (\LLS) model, that is, \LPS of order $\LPSorder=1$, noting that the extension to an arbitrary order of the \LPS model is straight-forward.
\begin{myRemark}
If not explicitly stated differently, we discuss the case $\SigmaSpace = \posDefSet{d}$, noting that several results may hold true, when considering true subsets $\SigmaSpace \neq \posDefSet{d}$ that are conic.
In particular $\SigmaSpace$ is conic, if $\SigmaSpace = \bigcup_{\bandwidth\in \SigmaSpace}\bigcup_{a\in\nonnegativeReal{}}\{a\bandwidth\}$.
\end{myRemark}

\subsection{Non-isotropic Locally Optimal Bandwidths}
\label{subsec:discussionLOB}
The key difference we have to treat is, whether or not the leading bias term of order $(\LPSorder+1)$ can be canceled a.e.~over the input space for an appropriate local choice of bandwidth matrix.
In the case of \LLS and $\SigmaSpace = \posDefSet{d}$ the question above simplifies to whether $f$ fulfills the definiteness assumption, mentioned in the introduction.
This can be understood by taking a look at the bias-variance-decomposition:
Let $\mu_2 = \int u^2k(u)du$ and $R(k) = \int k^2(u)du$. When holding $\bandwidth \in \posDefSet{d}$ fixed, the asymptotic form of the leading bias- and variance-terms is known:
\begin{mythm}[\cite{ruppert1994}]
\label{thm:llsMSE_v2}
Let $h_n^{}\rightarrow 0, nh_n^d\rightarrow \infty$ as $n\rightarrow\infty$ and $x\in\inputSpaceInterior$ with $p(x) > 0$. Furthermore let $f \in \diffableFunctions{\inputSpace}{2}$ with Hessian $D^2_f(x)$, $p \in \diffableFunctions{\inputSpace}{1}$ and $v \in \diffableFunctions{\inputSpace}{0}$.
Then, for a fixed $\bandwidth \in \posDefSet{d}$, it is
\[\textstyle\Bias_{1}^{}\left[x, h_n^{}\bandwidth | \bm{X}_{n}\right] = \biasLLS{h_n^{}\bandwidth}{x}{2} + \convergenceInProbability{}{h_n^2}\]
and
\[\textstyle\Var_{1}^{}\left[x, h_n^{}\bandwidth | \bm{X}_{n}\right] = \varianceLPS{h_n^{}\bandwidth}{x}{1} + \convergenceInProbability{}{n^{-1}h_n^{-d}},\]
where for $Z \in \posDefSet{d}$ we define
\begin{align}
 \label{eq:asymptoticScndOrderBiasLLS}
 \biasLLS{Z}{x}{2} = \frac{1}{2}\mu_2\trace(D^2_f(x)Z^2),\\
 \label{eq:asymptoticVarianceLLS}
 \varianceLPS{Z}{x}{1} = \frac{R(k)v(x)}{\abs{Z} p(x)n}.
\end{align}
Therefore the conditional mean squared error in $x$ can be expressed as
 \begin{align*}
  &\MSE_{1}^{}\left[x, h_n^{}\bandwidth | \bm{X}_{n}\right]
  = \biasLLS{h_n^{}\bandwidth}{x}{2}^2 \!+ \varianceLPS{h_n^{}\bandwidth}{x}{1}
  + \convergenceInProbability{\noexpand\big}{h_n^4+n^{-1}h_n^{-d}}.
 \end{align*}
\end{mythm}
\begin{myRemark}
For $\bandwidth = \idMatrix{d}$, both formulations, \eqref{eq:asymptoticLeadingOrderBiasLPS}, \eqref{eq:asymptoticVarianceLPS} and \eqref{eq:asymptoticScndOrderBiasLLS}, \eqref{eq:asymptoticVarianceLLS} of asymptotic bias and variance coincide. In this sense, \thmref{thm:llsMSE_v2} is a generalization of \thmref{thm:asymptoticBiasVarianceForLPS} for $(\LPSorder=1)$.
\end{myRemark}
Now, if $f$ fulfills the definiteness assumption, then \LOB exists uniquely and exhibits a closed-form asymptotic solution:
\begin{mycor}[\cite{fan1997local}]
\label{cor:LOBforLLSandDefiniteF}
May the conditions of \thmref{thm:llsMSE_v2} hold and assume that $p(x), v(x) \neq 0$ and $f$ is definite in $x$.
When searching for \LOB in the space of positive definite bandwidth candidates $\SigmaSpace = \posDefSet{d}$, asymptotically it is
\begin{align}
 \nonumber
 &\LOBfunctionOfLPS{1, \posDefSet{d}}{n}(x) =\textstyle \SigmaAsympLPS{1}(x) + \convergenceInProbability{}{n^{-\frac{1}{4+d}}},\quad\text{for}\\
 \label{eq:asymptoticgeneralizedLOBofLLS}
 &\SigmaAsympLPS{1}(x) = \left[\frac{R(k)v(x)\abs{\sqrt{D^2_f(x)^+}}}{p(x)n\mu_2^2d}\right]^\frac{1}{4+d}\sqrt{D^2_f(x)^+}^{-1},
\end{align}
where $D^2_f(x)^+ = \begin{cases} D^2_f(x),&\text{if } D^2_f(x)\text{ is positive definite} \\ -D^2_f(x),&\text{if } D^2_f(x)\text{ is negative definite} \end{cases}$,\newline
and $\sqrt{\Sigma}\in\posDefSet{d}$ such that $\sqrt{\Sigma}\cdot \sqrt{\Sigma} = \Sigma$ for $\Sigma \in \posDefSet{d}$.
\end{mycor}
Therefore, in the case of $(\LPSorder=1)$ and almost everywhere definite $f$, \corollaryref{cor:LOBforLLSandDefiniteF} generalizes \corollaryref{cor:isotropicLOBforLPS} to the positive definite bandwidth candidate space. It is noteworthy that in this case we do not benefit from $f$ being smoother than $C^2$ up to a constant factor. In particular there is no rate increase, using non-isotropic over isotropic \LOB.

In contrast, if $f$ does not fulfill the definiteness assumption, 
we can show that \LOB still exists under weaker conditions. In this case,
the asymptotics of \LOB and the associated \MSE behaves surprisingly different. We refer to \appref{sec:nonIsotropicLOB} for a detailed derivation.

In preparation for this, we need to notion of consistent sequences
\begin{align}
\label{def:consistentSequenceSpace}
\consistentSigmaSequencSpace\!=\!\condset{\!(\bandwidth^n)_{n\in\N}\!\subset \SigmaSpace\!}{\!\abs{\bandwidth^n}^{-1}\!n^{-1}\!\!\rightarrow 0, \pnorm{\bandwidth^n}{}\!\rightarrow 0, \pnorm{\bandwidth^n}{}\!\cdot\!\pnorm{[\bandwidth^n]^{-1}}{}\!\leq \!\conditionNumberUpperBound},
\end{align}
where $\conditionNumberUpperBound < \infty$ is an arbitrary constant. We further need to define the pointwise Hölder smoothness of $f$:
\begin{mydef}
\label{def:pointwiseHölder}
The function $f$ belongs to the pointwise Hölder space $\Lambda^\alpha(x)$, if for $\alpha = Q + \beta$ with $Q\in\N$ and $\beta\in(0,1]$ there exists a constant $C_{x} > 0$, a
closed ball $\bar B_{\delta_x}(x) = \condset{x'\in\reals{d}}{\pnorm{x-x'}{}\leq \delta_x}\subset\inputSpace$ for some $\delta_x > 0$ and a polynomial $P\in\mathcal{P}_{Q}(\reals{d})$ such that
\[
\sup\condset{\abs{f(x') - P(x)} \leq C_{x}\pnorm{x-x'}{}^\alpha}{{x'\in \bar B_{\delta_x}(x) }}.
\]
The pointwise Hölder exponent of $f$ in $x$ is then given by
\[
\alpha(f,x) = \sup\condset{a > 0}{f\in\Lambda^a(x)}.
\]
\end{mydef}
\begin{myRemark}
If $f\in\diffableFunctions{\inputSpace}{\alpha}$, then $f \in \Lambda^\alpha(x)$  such that $\alpha(f,x) \geq \alpha$ for all $x \in \inputSpace$.
\end{myRemark}
With these two definitions, we are able to state the following result (see \appref{sec:nonIsotropicLOB}):
\begin{restatable}[Non-isotropic \LOB and \MSE of \LLS for indefinite functions]{mythm}{generalizedLOBandMSE}~\linebreak
\label{thm:generalizedLOBandMSE}
Let $\SigmaSpace=\posDefSet{d}$ with $d \geq 2$.
Let $x\in\inputSpaceInterior$ where for $\alpha := \alpha(f,x)$ it holds that $2 < \alpha < \infty$ and write $\alpha = L + \beta$ with $L\in\N, L \geq 2$ and $\beta\in(0,1]$. Furthermore assume that $k \in \diffableFunctions{\nonnegativeReal{}, \nonnegativeReal{}}{\lfloor L/2\rfloor}$, $v$ is continuous in $x$ and $p\in\Lambda^{\alpha-2}(x)$ with $p(x), v(x) > 0$.
If $f$ is indefinite in $x$,
then there exists a sequence $(\bandwidth_x^n)_{n\in\N} \in \consistentSigmaSequencSpace$ such that with $h_n^{} = n^{-\frac{1}{2\alpha+d}}$,
\begin{align*}
 &\textstyle\MSE_{1}^{}\left(x, \bandwidth_x^n | \bm{X}_{n}\right)
 = \exactRateconvergenceInProbability{}{h_{n}^{2\alpha}} = \exactRateconvergenceInProbability{\noexpand\big}{n^{-\frac{2\alpha}{2\alpha+d}}},
\end{align*}
and $N_x \in\N$ such that for all $(\bandwidth^n)_{n\in\N} \in \consistentSigmaSequencSpace$, if $n \geq N_x$, it holds that
\[\textstyle\MSE_{1}^{}\left(x, \bandwidth_x^n| \bm{X}_{n}\right) \leq \MSE_{1}^{}\left(x, \bandwidth^n | \bm{X}_{n}\right).\]
For such a sequence, $\bandwidth_x^n = \exactRateconvergenceInProbability{\noexpand\big}{n^{-\frac{1}{2\alpha+d}}\ones{d}\onesT{d}}$ and $\inf_{\bandwidth\in\vanishingSet{1}{x}}\pnorm{h_n^{-1}\bandwidth_x^n - \bandwidth}{} = \upperBoundedInProbability{}{h_n^2}$ must hold.
\end{restatable}
The intuition behind the above theorem is as follows:
By assumption, we are able to strictly eliminate the bias of (the slowest) second order. Accordingly let $\bandwidth\in\SigmaSpace$ be such a root.
Due to the $\alpha$-smoothness of $f$ in $x$, we can explicitly expand the bias to the order of the next smaller integer (see \appref{sec:higherOrderBias}). Now we can perturb $\bandwidth$ in such a way that the second-order bias-term equals minus the sum of the remaining higher-order bias-terms. It remains a bias of the order $\upperBoundedInProbability{}{\pnorm{\bandwidth}{}^\alpha}$ that cannot be eliminated.
In \appref{sec:nonIsotropicIndefiniteExample}, we demonstrate this phenomenon in a controlled 2-dimensional toy-example.

While we now know about the existence of \LOB in the indefinite regime, we know nothing about its uniqueness. In particular, we have no explicit asymptotic form, like in the definite case. We may still define the set-valued function of minimizers
\begin{align}
 \label{eq:generalizedLOBDefinition}
 \LOBfunctionOfLPS{1, \SigmaSpace}{n}(x)
  = \textstyle 
  \condset{\bandwidth \in S_{n}}{\MSE_{1}^{}\left(x, \bandwidth | \bm{X}_{n}\right) = \min\limits_{\bandwidth'\in S_{n}}\MSE_{1}^{}\left(x, \bandwidth' | \bm{X}_{n}\right)},
\end{align}
where we delimit the bandwidth search space reasonably, as in the proof of \lemref{lem:minimizerExistence}, by 
\[
\textstyle S_{n} = \condset{\bandwidth\in\SigmaSpace}{\pnorm{\bandwidth}{} \leq n^{-\gamma}, \abs{\bandwidth} \geq n^{-1}},
\]
using $\gamma = \frac{1}{2\alpha+d+1}$ to exclude inconsistent solutions.

\subsection{Non-isotropic Local Function Complexity}
\label{subsec:discussionLFC}
Let us first take a look at the definite case: Here, we know from \corollaryref{cor:LOBforLLSandDefiniteF} about the asymptotic scaling of \LOB. Hence, a straight-forward generalization of \LFC from the isotropic case in \defref{def:isotropicComplexityForLPS} is given as follows:

\begin{restatable}[Non-Isotropic \LFC of \LLS in the definite case]{mydef}{generalizedComplexityLLSdefinite}~\newline
\label{def:generalizedComplexityLLSdefinite}
For $\SigmaSpace=\posDefSet{d}$, let $\LOBfunctionOfLPS{1, \posDefSet{d}}{n}$ be the optimal bandwidth function as defined in \eqref{eq:generalizedLOBDefinition}.
For $f$ definite in $x$, we define by
\begin{align}
 \label{eq:generalizedComplexityLLSdefinite}
 \fctnComplexityLPS{1, \posDefSet{d}}{n}(x) = C_{1}^{d}\Big[\frac{v(x)}{p(x)n}\Big]^{\frac{d}{4+d}}\abs{\LOBfunctionOfLPS{1, \posDefSet{d}}{n}(x)}^{-1}
\end{align}
the \LFC of $f$ in $x$ with respect to the \LLS model.
\end{restatable}
Like in the isotropic case, 
$\fctnComplexityLPS{1, \posDefSet{d}}{n}$ is asymptotically continuous and independent of the global scaling with respect to training size $n$, as well as the local scaling with respect to the training density $p$ and noise level $v$, which can be seen by writing
\begin{align}
 \notag
 \fctnComplexityLPS{1, \posDefSet{d}}{n}(x) &= \fctnComplexityLPS{1, \posDefSet{d}}{\infty}(x)(1 + \convergenceInProbability{}{1}),\;\;\text{where}\\
 \label{eq:asymptoticGeneralizedComplexityLLS}
 \fctnComplexityLPS{1, \posDefSet{d}}{\infty}(x) &= \Big[\frac{\mu_2 d}{2}\Big]^\frac{2d}{4+d}\abs{\sqrt{D^2_f(x)^+}}^\frac{4}{4+d}.
\end{align}
\begin{myRemark}
The isotropic and the generalized definitions of \LFC in \eqref{eq:isotropicComplexityLPS} and \eqref{eq:generalizedComplexityLLSdefinite} coincide asymptotically, if an isotropic bandwidth is, in fact, the optimal solution to the generalized case. That is, if we can write $D^2_f(x) = g(x)\idMatrix{d}$ for some scalar-valued function $g$, then
\begin{align*}
 \fctnComplexityLPS{1, \posDefSet{d}}{\infty}(x) &\overset{\eqref{eq:asymptoticGeneralizedComplexityLLS}}{=} \Big[\frac{\mu_2 d}{2}\Big]^\frac{2d}{4+d}\abs{g(x)}^\frac{2d}{4+d} = \left[\frac{\mu_2}{2}\trace\{D^2_f(x)\}\right]^\frac{2d}{4+d}\\ &\overset{\eqref{eq:asymptoticScndOrderBiasLLS}}{=} \biasLPS{\idMatrix{d}}{x}{1}^\frac{2d}{4+d} \overset{\eqref{eq:asymptoticIsotropicComplexityLPS}}{=} \fctnComplexityLPS{1}{\infty}(x).
\end{align*}
\end{myRemark}

When moving on to the indefinite case, we struggle to define \LFC as a function of \LOBfunctionOfLPS{1, \posDefSet{d}}{n} on first glance, because \LOB is not necessarily a proper function over the input space.
Fortunately, we can show the following (see \appref{sec:nonIsotropicLOBdeterminant}):
\begin{restatable}{mythm}{asymptoticLOBdeterminant}
\label{thm:asymptoticLOBdeterminant}
Let $f\in\diffableFunctions{\inputSpace}{\alpha}$ and may the assumptions of \thmref{thm:generalizedLOBandMSE} hold uniformly with $\alpha(f,x) \equiv \alpha$ for all $x\in\inputSpaceInterior$.
Then there exist $D_n\in\diffableFunctions{\inputSpace,\positiveReal{}}{0}$ such that
\[\abs{h_n^{-1}\LOBfunctionOfLPS{1, \posDefSet{d}}{n}(x)}^{-1} = D_n(x) + \convergenceInProbability{}{1},\]
almost everywhere in \inputSpace.
Furthermore there exists a limiting function $D$ such that
\[D_n(x) = D(x) + \convergenceInProbability{}{1}.\]
\end{restatable}
In particular this means that, if $\bandwidth_n,\bandwidth_n' \in \LOBfunctionOfLPS{1, \posDefSet{d}}{n}(x)$ are both optimal for prediction in $x$, then they asymptotically share their reciprocal determinant.
For finite $n$ the reciprocal determinant is close to a continuous function over the input space. Finally, the reciprocal determinant is itself pointwise convergent as $n\rightarrow\infty$, which enables an asymptotic analysis:

We know that, whether $f$ is definite or indefinite has no influence on the asymptotic variance as defined in \eqref{eq:asymptoticVarianceLLS}. 
On the other hand, the bias will behave substantially different. And even though we have no access to an explicit form, we expect it to depend on $p$ and its derivatives, since the explicit higher-order bias-terms do so. All we can say is, that the bias will not depend on $v$ and $n$.

We can partially remedy the above problem by enforcing $p$ to have vanishing derivatives, almost everywhere. In particular, let us define the set of probability step-functions:
\begin{mydef}
 We call $p\in\mathfrak{P}(\inputSpace)$ a probability step-function if $p$ is a probability density over \inputSpace, and there exists a finite partition $\inputSpace = X_1 \uplus \ldots \uplus X_S$ of the input space with constants $P_1,\ldots,P_S > 0$ such that
 \[p(x) = \mySum{s=1}{S} \indicatorFunction{X_s}{x} P_s.\]
\end{mydef}
Then we obtain the following result of \LFC in the indefinite case:
\begin{restatable}[Non-isotropic \LFC of \LLS in the indefinite case]{mydef}{generalizedComplexityLLSindefinite}~\linebreak
\label{def:generalizedComplexityLLSindefinite}
May the assumptions of \thmref{thm:generalizedLOBandMSE} hold with $\alpha(f,x) \equiv \alpha$ and let $p\in\mathfrak{P}(\inputSpace)$ be a probability step-function. Then we define the \LFC of \LLS in the non-isotropic, indefinite case as
\[\fctnComplexityLPS{1, \posDefSet{d}}{n}(x) = \left[\frac{v(x)}{p(x)n}\right]^\frac{d}{2\alpha+d}\abs{\LOBfunctionOfLPS{1, \posDefSet{d}}{n}(x)}^{-1} = \left[\frac{v(x)}{p(x)}\right]^\frac{d}{2\alpha+d}D_n(x) + \convergenceInProbability{}{1}.\]
\end{restatable}
Since we can show that $D_n$ is a continuous function, almost everywhere, so is $\fctnComplexityLPS{1, \posDefSet{d}}{n}$.

\subsection{Non-isotropic Optimal Sampling}
\label{subsec:discussionAL}
From here on, we assume the local properties to hold globally over the input space. That is, $f$ is either definite or indefinite, almost everywhere. And if $f$ is indefinite, then we assume $\alpha(f,x) \equiv \alpha$ for some shared $\alpha > 2$.

When taking a look at a definite function $f$, the result of \thmref{thm:OptimalSampling} generalizes straight-forward the non-isotropic case:  
\begin{mycor}
 \label{cor:definiteOptimalSampling}
 Let $v, q \in \diffableFunctions{\inputSpace, \nonnegativeReal{}}{0}$ for a compact input space \inputSpace, where $q$ is a test density such that $\textstyle \mathop{\mathlarger{\int}}_{\hspace*{-5pt}\inputSpace}q(x)dx = 1$. 
 Additionally, assume that $v$ and $q$ are bounded away from zero. That is, $v,q \geq \epsilon$ for some $\epsilon > 0$.
 Let $k$ be a \RBF-kernel with bandwidth parameter space $\SigmaSpace = \posDefSet{d}$. Let $f \in \diffableFunctions{\inputSpace}{2}$ such that $f$ is definite, almost everywhere.
 Then the optimal training density for \LLS is asymptotically given by
 \begin{align}
 \label{eq:definiteOptimalSamplingFiniteLLS}
 \pOptLPS{1, \posDefSet{d}}{n}(x) \propto \textstyle \left[\fctnComplexityLPS{1, \posDefSet{d}}{n}(x) q(x)\right]^{\frac{4+d}{8+d}}v(x)^{\frac{4}{8+d}}(1 + o(1)).
 \end{align}
\end{mycor}
\begin{proof}
Basically, the proof is analogous to the proof in \thmref{thm:OptimalSampling}, apart from the following adaptions:
Instead of \corollaryref{cor:isotropicLOBforLPS} we apply \corollaryref{cor:LOBforLLSandDefiniteF} for the asymptotic formulation of \LOB. 
For the asymptotic variance, we apply \eqref{eq:asymptoticVarianceLLS} in \thmref{thm:llsMSE_v2} instead of \eqref{eq:asymptoticVarianceLPS} in \thmref{thm:asymptoticBiasVarianceForLPS}.
Instead of \defref{def:isotropicComplexityForLPS} and \eqref{eq:asymptoticIsotropicComplexityLPS}, we apply \defref{def:generalizedComplexityLLSdefinite} and \eqref{eq:asymptoticGeneralizedComplexityLLS} for the \LFC.
With these, we can get to the point
\begin{align*}
\MSE_{1}^{}\left(x, \LOBfunctionOfLPS{1, \posDefSet{d}}{n}(x) | \bm{X}_{n}\right) = \bar{C}_{1}^{} \Big[\frac{v(x)}{p(x)n}\Big]^\frac{4}{4+d}\fctnComplexityLPS{1, \posDefSet{d}}{\infty}(x) + \convergenceInProbability{}{n^{-\frac{4}{4+d}}},
\end{align*}
with $\bar{C}_{1}^{} = \frac{4+d}{4}R(k) C_{1}^{-d}$, as in \thmref{thm:OptimalSampling} from where we can proceed in complete analogy.
\end{proof}
Here, extending the above results to the indefinite case poses a problem, since we cannot take the functional derivative with respect to step-function $p$.
However, we suggest to use the straight-forward extension as a heuristic solution:
\begin{myconjecture}
 \label{con:indefiniteOptimalSampling}
 Let $v, q \in \diffableFunctions{\inputSpace, \nonnegativeReal{}}{0}$ for a compact input space \inputSpace, where $q$ is a test density such that $\textstyle \mathop{\mathlarger{\int}}_{\hspace*{-5pt}\inputSpace}q(x)dx = 1$. 
 Additionally, assume that $v$ and $q$ are bounded away from zero. That is, $v,q \geq \epsilon$ for some $\epsilon > 0$.
 Let $\alpha = L + \beta$ with $L\in\N, L \geq 2$ and $\beta\in(0,1]$ with $f \in \diffableFunctions{\inputSpace}{\alpha}$ such that $\alpha(f,x) \equiv \alpha$, almost everywhere. Let $k \in \diffableFunctions{\nonnegativeReal{}, \nonnegativeReal{}}{\lfloor L/2\rfloor}$ be a \RBF-kernel with bandwidth parameter space $\SigmaSpace = \posDefSet{d}$ for $d \geq 2$.
 Then, for almost everywhere indefinite $f$, the training density
 \begin{align}
 \label{eq:indefiniteOptimalSamplingFiniteLLS}
 \pOptLPS{1, \posDefSet{d}}{n}(x) \propto \textstyle \left[\fctnComplexityLPS{1, \posDefSet{d}}{n}(x) q(x)\right]^{\frac{2\alpha+d}{4\alpha+d}}v(x)^{\frac{2\alpha}{4\alpha+d}}(1 + o(1))
 \end{align}
 is asymptotically superior to \randomTestSampling for \LLS.
\end{myconjecture}

\section{Higher-order Bias Expansion}
\label{sec:higherOrderBias}
As already noted, the function $f$ to infer will typically not be definite, almost everywhere.
In the indefinite regime, the second-order bias-term can be eliminated systematically.
Here, the straight-forward approach for extension of the 
analysis of \LOB is to take higher-order bias-terms into consideration.
Being mathematically more precise, we will now analyze the higher-order bias-decomposition.
\begin{restatable}[The asymptotic bias of $l$-th order]{myprop}{asymptoticBiasForm}
\label{prop:biasTermForm}
Assume the kernel $k$ to be symmetric and let $\alpha = L + \beta$ with $L\in\N, L \geq 2$ and $\beta\in(0,1]$. Furthermore let $h_n^{} = n^{-\frac{1}{2\alpha+d}}$, $p \in \diffableFunctions{\inputSpace}{\alpha-2}$ and $f \in \diffableFunctions{\inputSpace}{\alpha}$.
For a fixed $x\in\inputSpaceInterior$ and $\bandwidth\in\posDefSet{d}$ we can decompose the conditional bias of \LLS as
\begin{align*}
\textstyle\Bias_{1}^{}\left[x, h_n^{}\bandwidth | \bm{X}_{n}\right] = \textstyle \mySum{l=1}{\lfloor{L/2}\rfloor }\biasLLS{h_n^{}\bandwidth}{x}{2l} + \upperBoundedInProbability{}{h_n^{\alpha}}.
\end{align*}
The bias-terms of odd order vanish, whereas the bias-terms of even order are given by
\begin{align}
 \label{eq:asymptoticBias}
 \biasLLS{\bandwidth}{x}{2l} =\textstyle \hspace{-11pt}\sum\limits_{j_1,\ldots,j_{2l} = 1}^{d} \hspace{-10pt}A^{l,x}_{j_1,\ldots,j_{2l}}
 \hspace{-12pt}\sum\limits_{i_1,\ldots,i_{2l} = 1}^{d}B^{l,x}_{i_1,\ldots,i_{2l}}\prod\limits_{r = 1}^{2l} \bandwidth_{i_r, j_r},
\end{align}
where the tensors $A^{l,x}$ and $B^{l,x}$ depend on derivatives (with a total order of $2l$) in $x$ of $f$ of order $2$ to $2l$, and $p$ up to order $2l-2$.
In particular, $\biasLLS{\bandwidth}{x}{2l}$ is continuous in $x$.
\end{restatable}
\begin{proof}
Recall from \eqref{eq:trueFiniteBias} that
\[\textstyle\Bias_{1}^{}\left[x, \bandwidth | \bm{X}_{n}\right] = f(x) - e_1^\top\left(X_{1}^{}(x)^\top W_x^{\bandwidth} X_{1}^{}(x)\right)^{-1}X_{1}^{}(x)^\top W_x^{\bandwidth} f(\bm{X}_n^{}).
\]
Let
\[f(\bm{X}_n^{}) = \mySum{l=0}{L}T_l^{f,x}(\bm{X}_n^{}) + \upperBoundedInProbability{}{\pnorm{\bm{X}_n^{}-x}{}^\alpha}\]
be the Taylor expansion of the true training function values, where
\[
T_l^{f,x}(\bm{X}_n^{}) = \begin{bmatrix} T_l^{f,x}(x_1) & \ldots & T_l^{f,x}(x_n)\end{bmatrix}^\top\]
are the Taylor expansion terms of $l$-th order of the respective training samples.
That is,
\[T_l^{f,x}(x') = \frac{1}{l!}\sum\limits_{i_1,\ldots,i_l=1}^{d} [D_f^l(x)]_{i_1,\ldots,i_l}\prod_{r=1}^l(x'-x)_{i_r}.\]
Defining $N_x^{h_n^{}\bandwidth} = e_1^\top\left(\frac{1}{n}X_{1}^{}(x)^\top W_x^{h_n^{}\bandwidth} X_{1}^{}(x)\right)^{-1}$, and noting that the first two Taylor expansion terms can be written as $X_{1}^{}(x)\begin{bmatrix} f(x) & D_f(x)\end{bmatrix}^\top$, we obtain
\begin{align*}
\textstyle\Bias_{1}^{}&\left[x, h_n^{}\bandwidth | \bm{X}_{n}\right]\\
&= N_x^{h_n^{}\bandwidth} X_{1}^{}(x)^\top W_x^{h_n^{}\bandwidth}\hspace{-2pt}\left[\frac{1}{n}\mySum{l=2}{L}T_l^{f,x}(\bm{X}_n^{}) + \upperBoundedInProbability{}{n^{-1}\pnorm{\bm{X}_n^{}-x}{}^\alpha}\right]\\
&= N_x^{h_n^{}\bandwidth} X_{1}^{}(x)^\top W_x^{h_n^{}\bandwidth}\hspace{-2pt}\left[\frac{1}{n}\mySum{l=2}{L}T_l^{f,x}(\bm{X}_n^{})\right] + \upperBoundedInProbability{}{h_n^{\alpha}},
\end{align*}

Denoting $\displaystyle\frac{1}{n}X_{1}^{}(x)^\top W_x^{h_n^{}\bandwidth} X_{1}^{}(x) = \begin{bmatrix} a & b^\top \\ b & C\end{bmatrix}$ with $\begin{bmatrix} e & g^\top\\g & H\end{bmatrix} = \begin{bmatrix} a & b^\top \\ b & C\end{bmatrix}^{-1}$,
we have 
$N_x^{h_n^{}\bandwidth} = \begin{bmatrix} e & g^\top\end{bmatrix}$.
Note that in the following we can ignore the error from Monte Carlo integration, as its convergence error of $n^{-\frac{1}{2}}$ is $\convergenceInProbability{}{h_n^{\alpha}}$ for the optimal rate $h_n^{} \propto n^{-\frac{1}{2\alpha+d}}$.
Hence we can write
\begin{align*}
 \begin{bmatrix} a & b^\top \\ b & C\end{bmatrix} &= \frac{1}{n}\sum\limits_{i=1}^n k^{h_n^{}\bandwidth}(x_i^{},x)\begin{bmatrix} 1 & (x_i^{}-x)^\top \\ x_i^{}-x & (x_i^{}-x)(x_i^{}-x)^\top\end{bmatrix}\\
 &= \int k(u)du\sum\limits_{k=0}^{L-2}T_{k}^{p,x}(x+h_n^{}\bandwidth u)\begin{bmatrix} 1 & h_n^{} u^\top\bandwidth \\ h_n^{}\bandwidth u & h_n^2\bandwidth u u^\top \bandwidth\end{bmatrix}\\
 &\quad+ \upperBoundedInProbability{}{\begin{bmatrix} h_n^{\alpha-2} & h_n^{\alpha-1}\onesT{d} \\ h_n^{\alpha-1}\ones{d} & h_n^{\alpha}\ones{d}\onesT{d}\end{bmatrix}}\\
  &= \begin{bmatrix} 1 & 0 \\ 0 & h_n^{}\bandwidth \end{bmatrix}
  \sum\limits_{l=0}^{\lfloor (L-2)/2 \rfloor}\begin{bmatrix} P_a^{2l}(h_n^{}\bandwidth) & 0\\ 0 & P_C^{2l}(h_n^{}\bandwidth)\end{bmatrix}
  \begin{bmatrix} 1 & 0 \\ 0 & h_n^{}\bandwidth \end{bmatrix}\\
  &\quad+ \begin{bmatrix} 1 & 0 \\ 0 & h_n^{}\bandwidth \end{bmatrix}
  \sum\limits_{l=0}^{\lfloor (L-3)/2 \rfloor}\begin{bmatrix} 0 & P_b^{2l+1}(h_n^{}\bandwidth)^\top\\ P_b^{2l+1}(h_n^{}\bandwidth) & 0\end{bmatrix}
  \begin{bmatrix} 1 & 0 \\ 0 & h_n^{}\bandwidth \end{bmatrix}\\
  &\quad+ \upperBoundedInProbability{}{\begin{bmatrix} h_n^{\alpha-2} & h_n^{\alpha-1}\onesT{d} \\ h_n^{\alpha-1}\ones{d} & h_n^{\alpha}\ones{d}\onesT{d}\end{bmatrix}},
\end{align*}
where all odd-order terms in $u$ vanish when integrating with respect to an \RBF-kernel, and
\begin{align*}
 &P_a^{2l}(Z) := \int k(u)du T_{2l}^{p,x}(x+Zu)\\
 &P_b^{2l+1}(Z) := \int k(u)du T_{2l+1}^{p,x}(x+Zu)u\\
 &P_C^{2l}(Z) := \int k(u)du T_{2l}^{p,x}(x+Zu)uu^\top.
\end{align*}

The denoted monomial order of these terms corresponds to the contained number of bandwidth matrices.
The expansions of $a,b$ and $C$ depend on derivatives of $p$ up to order $L-2$).
As a function of $h_n$, when applying the chain-rule, there exists an expansion (in root $h = 0$)
\[C^{-1} = (h_n^{}\bandwidth)^{-1}\sum\limits_{l=0}^{\lfloor (L-2)/2 \rfloor}P_{C^{-1}}^{2l}(h_n^{}\bandwidth)(h_n^{}\bandwidth)^{-1} + \upperBoundedInProbability{}{h_n^{\alpha-4}\ones{d}\onesT{d}},\]
that can be constructed in terms of the expansion of $C$. Thus it will also depend on derivatives of $p$ up to order ($L-2$).
Aligning the adequate terms, we can therefore expand
\begin{align*}
 b^\top C^{-1} b = \sum\limits_{l=1}^{\lfloor (L-2)/2 \rfloor}P_{b^\top C^{-1} b}^{2l}(h_n^{}\bandwidth) + \upperBoundedInProbability{}{h_n^{\alpha-2}}.
\end{align*}
Using the fact that $e^{-1} = a - b^\top C^{-1}b$, we also have
\begin{align*}
 e^{-1} = P_{a}^{0}(h_n^{}\bandwidth) + \sum\limits_{l=1}^{\lfloor (L-2)/2 \rfloor} \left[P_{a}^{2l}(h_n^{}\bandwidth) - P_{b^\top C^{-1} b}^{2l}(h_n^{}\bandwidth)\right] + \upperBoundedInProbability{}{h_n^{\alpha-2}},
\end{align*}
where  $P_{a}^{0}(h_n^{}\bandwidth) = p(x)$.
Applying the chain-rule, the expansion of
\[e = \sum\limits_{l=0}^{\lfloor (L-2)/2 \rfloor}P_{e}^{2l}(h_n^{}\bandwidth) + \upperBoundedInProbability{}{h_n^{\alpha-2}}\]
exists and also depend on derivatives of $p$ up to order ($L-2$).
Note that $P_{e}^{0}(h_n^{}\bandwidth) = p(x)^{-1}$.

Using the fact that $g^\top = -eb^\top C^{-1}$, with adequate alignment of the terms of their expansions, it is
\begin{align*}
 g^\top = \sum\limits_{l=0}^{\lfloor (L-3)/2  \rfloor}P_{g}^{2l+1}(h_n^{}\bandwidth)^\top(h_n^{}\bandwidth)^{-1} + \upperBoundedInProbability{}{h_n^{\alpha-3}\onesT{d}}
\end{align*}
which does depend on derivatives of $p$ up to order ($L-2$). Note that, if $L \geq 3$, then
$P_{g}^{1}(h_n^{}\bandwidth)^\top = -D_p^1(x)p(x)^{-2}h_n^{}\bandwidth$.
Therefore
\begin{align*}
N_x^{h_n^{}\bandwidth} = &\begin{bmatrix} \sum\limits_{l=0}^{\lfloor (L-2)/2  \rfloor}P_{e}^{2l}(h_n^{}\bandwidth)& \sum\limits_{l=0}^{\lfloor L-3)/2  \rfloor}P_{g}^{2l+1}(h_n^{}\bandwidth)^\top \end{bmatrix}\begin{bmatrix} 1 & 0 \\ 0 & h_n^{}\bandwidth \end{bmatrix}^{-1}\\
+ &\upperBoundedInProbability{}{\begin{bmatrix} h_n^{\alpha-2} & h_n^{\alpha-3}\onesT{d}\end{bmatrix}}.
\end{align*}
Now we can also expand
\begin{align*}
\frac{1}{n}X_{1}^{}(x)^\top &W_x^{h_n^{}\bandwidth} T_l^{f,x}(\bm{X}_n^{}) = \frac{1}{n}\mySum{i=1}{n}\hspace{-2pt}k^{h_n^{}\bandwidth}(x,x_i^{})T_l^{f,x}(x_i^{})\begin{bmatrix} 1 \\ x_i^{}-x\end{bmatrix}\\
= &\begin{bmatrix} 1 & 0 \\ 0 & h_n^{}\bandwidth \end{bmatrix}\int k(u)du T_l^{f,x}(x+h_n^{}\bandwidth u) \cdot \sum\limits_{r=0}^{L-l} T_{r}^{p,x}(x+h_n^{}\bandwidth u)
\begin{bmatrix} 1 \\  u\end{bmatrix}\\
&+ \upperBoundedInProbability{}{\begin{bmatrix} h_n^{\alpha} & h_n^{\alpha+1}\onesT{d}\end{bmatrix}^\top}.
\end{align*}
Note that here also all odd-order terms in $u$ vanish for \RBF-kernels.
Aggregating the polynomial terms of $\frac{1}{n}X_{1}^{}(x)^\top W_x^{h_n^{}\bandwidth}\hspace{-2pt}\left[\mySum{l=2}{L}T_l^{f,x}(\bm{X}_n^{})\right]$ of total order $2l$, we obtain
\begin{align*}
b_{2l}&(x,h_n^{}\bandwidth)\\
&= \begin{bmatrix} 1 & 0 \\ 0 & h_n^{}\bandwidth \end{bmatrix}\int k(u)du
\begin{bmatrix} \sum\limits_{r=2}^{2l} T_r^{f,x}(x+h_n^{}\bandwidth u)T_{2l-r}^{p,x}(x+h_n^{}\bandwidth u) \\ \sum\limits_{r=2}^{2l-1} T_r^{f,x}(x+h_n^{}\bandwidth u)T_{2l-1-r}^{p,x}(x+h_n^{}\bandwidth u)u \end{bmatrix}.
\end{align*}
Subsequently, when multiplying $\frac{1}{n}X_{1}^{}(x)^\top W_x^{h_n^{}\bandwidth}\hspace{-2pt}\left[\mySum{l=2}{L}T_l^{f,x}(\bm{X}_n^{})\right]$ and $N_x^{h_n^{}\bandwidth}$, any combination of an explicit term with an $O_p$-term is in $\upperBoundedInProbability{}{h_n^{\alpha}}$.
Particularly, multiplying $\upperBoundedInProbability{}{\begin{bmatrix} h_n^{\alpha-2} & h_n^{\alpha-3}\onesT{d}\end{bmatrix}}$ with any $b_{2l}$-term is $\upperBoundedInProbability{}{h_n^{\alpha}}$, because $b_{2}(x,h_n^{}\bandwidth) = \upperBoundedInProbability{}{\begin{bmatrix} h_n^2 & 0 \end{bmatrix}^\top}$ and $b_{4}(x,h_n^{}\bandwidth) = \upperBoundedInProbability{}{h_n^4\ones{d+1}}$.

Finally, we can decompose the conditional bias as
\begin{align*}
\textstyle\Bias_{1}^{}\left[x, h_n^{}\bandwidth | \bm{X}_{n}\right] = \mySum{l=1}{\lfloor L/2\rfloor}\biasLLS{h_n^{}\bandwidth}{x}{2l} + \upperBoundedInProbability{}{h_n^{\alpha}},
\end{align*}
where the bias of order $2l$ is given by
\begin{align*}
&\biasLLS{Z}{x}{2l} = \sum\limits_{m=1}^{l} \begin{bmatrix} P_{e}^{2(l-m)}(Z) & P_{g}^{2(l-m)+1}(Z)^\top \end{bmatrix}\cdot\\
&\begin{bmatrix} \sum\limits_{r=2}^{2m}\int T_r^{f,x}(x+Z u)T_{2m-r}^{p,x}(x+Z u)k(u)du \\ 
\sum\limits_{r=2}^{2m-1}\int T_r^{f,x}(x+Z u)T_{2m-1-r}^{p,x}(x+Z u)uk(u)du \end{bmatrix}.
\end{align*}
Via construction, $\biasLLS{\bandwidth}{x}{2l}$ depends on derivatives of $f$ of second or higher-order, and $p$ (with a total order of $2l$).
The maximal orders are therefore ($2l-2$) for $p$ and $2l$ for $f$.
Denote by $A^{l,x}$ and $B^{l,x}$ the adequate coefficient-tensors. Then we can write
\begin{align*}
 \biasLLS{h_n^{}\bandwidth}{x}{2l} = h_n^{2l} \hspace{-11pt}\sum\limits_{j_1,\ldots,j_{2l} = 1}^{d} \hspace{-10pt}A^{l,x}_{j_1,\ldots,j_{2l}}
 \hspace{-12pt}\sum\limits_{i_1,\ldots,i_{2l} = 1}^{d}B^{l,x}_{i_1,\ldots,i_{2l}}\prod\limits_{r = 1}^{2l} \bandwidth_{i_r, j_r}.\qedAtBottomLine
\end{align*}
\end{proof}

Note that all odd-order bias-terms vanish by symmetry of the kernel $k$.
The higher-order bias-decomposition enormously simplifies under the assumption of uniformly distributed inputs:
\begin{restatable}{mycor}{uniformBiasTermForm}
\label{cor:uniformBiasTermForm}
May the assumptions of \propref{prop:biasTermForm} hold.
If additionally $p\sim\uniformDist{\inputSpace}$ is uniformly distributed, the asymptotic even-order bias-terms simplify to
\begin{align*}
\biasLLS{\bandwidth}{x}{2l}
= \textstyle\hspace{-10pt}\sum\limits_{i_1,\ldots,i_{2l} = 1}^d\hspace{-10pt}\frac{\mu(i_1,\ldots,i_{2l})}{(2l)!}\hspace{-10pt}\sum\limits_{j_1,\ldots,j_{2l}=1}^{d}\hspace{-10pt}[D_f^{2l}(x)]_{j_1,\ldots,j_{2l}} \prod_{r=1}^{2l}\bandwidth_{i_r j_r},
\end{align*}
where $\mu(i_1,\ldots,i_{2l}) = \prod\limits_{p = 1}^{d} \mu_{c_p(i_1,\ldots,i_{2l})}$ for $\mu_c = \int u^ck(u)du$ the moments of the \RBF-kernel $k$ with $\mu_0 = 1$ (such that $k$ is a probability density on $\reals{d}$) and $c_p(i_1,\ldots,i_{2l}) = \abs{\condset{r}{i_r = p}}$.
\end{restatable}
\begin{proof}
In the case of $p\sim\uniformDist{\inputSpace}$, the only non-vanishing term in $\biasLLS{\bandwidth}{x}{2l}$ from \eqref{eq:asymptoticBias} is
\begin{align*}
\underbrace{P_{e}^{0}(\bandwidth)}_{p(x)^{-1}} &\int T_{2l}^{f,x}(x+\bandwidth u)\underbrace{T_{0}^{p,x}(x+\bandwidth u)}_{p(x)}k(u)du
= \int T_{2l}^{f,x}(x+\bandwidth u)k(u)du\\
&= \frac{1}{(2l)!}\sum\limits_{j_1,\ldots,j_{2l}=1}^{d}\hspace{-5pt}[D_f^{2l}(x)]_{j_1,\ldots,j_{2l}}\int k(u) \prod_{r=1}^{2l}(\bandwidth u)_{j_r}du\\
&= \hspace{-10pt}\sum\limits_{i_1,\ldots,i_{2l} = 1}^d\hspace{-10pt}\frac{\mu(i_1,\ldots,i_{2l})}{(2l)!}\hspace{-10pt}\sum\limits_{j_1,\ldots,j_{2l}=1}^{d}\hspace{-10pt}[D_f^{2l}(x)]_{j_1,\ldots,j_{2l}} \prod_{r=1}^{2l}\bandwidth_{i_r j_r}.\qedAtBottomLine
\end{align*}
\end{proof}
\begin{myRemark}
The second-order bias of \LLS does not depend on the training density $p$ in general. Therefore it holds for an arbitrary, smooth enough $p$ that
\begin{align*}
 &\biasLLS{\bandwidth}{x}{2}\\
 &= \hspace{-4pt}\sum\limits_{i_1,i_2 = 1}^d\hspace{-4pt}\frac{\mu(i_1,i_2)}{2}\hspace{-4pt}\sum\limits_{j_1,j_2=1}^{d}\hspace{-0pt}[D_f^{2}(x)]_{j_1,j_2} \prod_{r=1}^{2}\bandwidth_{i_r j_r}
 = \hspace{-0pt}\sum\limits_{i = 1}^d\hspace{-0pt}\frac{\mu_2}{2}\hspace{-4pt}\sum\limits_{j_1,j_2=1}^{d}\hspace{-0pt}[D_f^{2}(x)]_{j_1,j_2} \prod_{r=1}^{2}\bandwidth_{i j_r}\\
 &= \frac{\mu_2}{2}\sum\limits_{i = 1}^d\left[\bandwidth D_f^{2}(x) \bandwidth\right]_{i,i} = \frac{\mu_2}{2}\trace(\bandwidth D_f^{2}(x) \bandwidth) = \frac{\mu_2}{2}\trace(D_f^{2}(x)\bandwidth^2).
\end{align*}
\end{myRemark}
While this simplifies form of the bias is still quite tedious, all appearing terms are now at least given in explicit form. This enables the construction of a toy-example, as given in the next section.

\section{A minimal indefinite Toy-Example}
\label{sec:nonIsotropicIndefiniteExample}
As a preconsideration, let $\vanishingSet{l}{x} = \condset{\bandwidth\in\posDefSet{d}}{\biasLLS{\bandwidth}{x}{2l} = 0}$ be the vanishing sets of the respective, even-order bias-terms.
We observe that $\vanishingSet{1}{x} = \condset{\bandwidth\in\posDefSet{d}}{\trace(D^2_f(x)\bandwidth^2) = 0}$ is a sub-manifold of $\posDefSet{d}$ of at most $(\dim(\posDefSet{d})-1)$ dimensions for $D^2_f(x)\neq 0$.
Letting $\SigmaSpace\subseteq\posDefSet{d}$ be our bandwidth candidate space, we thus assume that the candidate subspace $\SigmaSpace \cap \bigcap_{l = 0}^{L} T_l^x$ which simultaneously eliminates bias-terms up to order $2L$ will shrink and finally vanish as we increase $L$.
Under mild assumptions on $f$ and $p$ we can assume that this subspace decreases in dimension by at least $1$ with each consecutive conditioning on the elimination of the next higher-order bias-term. We end up with a bias-term of maximal order $2D$ that cannot be eliminated anymore, where $D = \dim(\SigmaSpace) \leq \frac{d(d+1)}{2}$. This suggests that we could obtain minimax-optimal convergence for a generic $f\in\diffableFunctions{\inputSpace}{2D}$, when applying \LLS with \SigmaSpace:
While being infeasible from the practical point-of-view, we could construct \LOB by first finding the subspace of \SigmaSpace that simultaneously eliminates as many leading-order bias terms as possible, followed by optimizing the trade-off between the first non-vanishing bias-term and the variance over this bandwidth candidate subspace.

The reasoning above however turns out to be naively pessimistic, since \LOB does not necessarily lie in the specified terminal subspace.
The true behavior of \LOB is surprisingly different, and we will now construct a toy-example to reveal the principle.

Let $\inputSpace = [0,1]^2$ with uniformly distributed inputs $p \sim \uniformDist{\inputSpace}$ and observations $p(y | x) =  \Gauss{y}{f(x)}{v(x)}$, with
\begin{align*}
f(x) = \exp\{a x_1\} + \log(0.1+b x_2)\quad\text{and}\qquad v(x) = 10^{-4},
\end{align*}
where we have set $a = 2$ and $b = 3$. An example dataset can be seen to the left in \figref{fig:biasCancelingExperiment_dataset_anisotropicNaiveLOBestimate}.
\fig{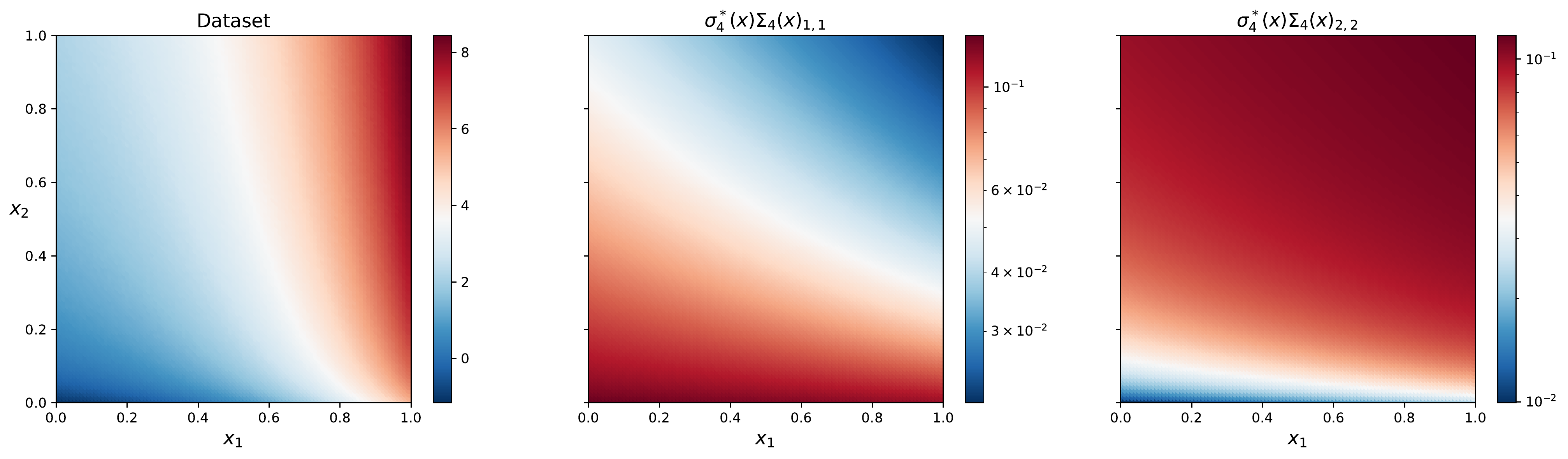}{1}{The bias canceling experiment: An example dataset (left). The two components (middle and right) of the naive anisotropic bandwidth construction up to the global scaling in $n$, as a means to estimate anisotropic \LOB.}{0.8cm}{0.0cm}
Furthermore let $k(z) = \exp\{-0.5\pnorm{z}{}^2\}$ be the Gaussian kernel and $\SigmaSpace = \condset{\diag(\sigma_1,\sigma_2)}{\sigma_1,\sigma_2 > 0}$ the set of anisotropic bandwidths.
The Hessian of $f$ is given by
\[D^2_f(x)_{i_1,i_2} = \begin{cases} D^2_f(x)_1 := a^2\exp\{a x_1\}&, i_1=i_2=1, \\ D^2_f(x)_2 := -b^2(0.1+b x_2)^{-2}&, i_1=i_2=2, \\ 0&\text{, else},\end{cases}\]
whereas the fourth-order partial derivatives are given by
\[D^4_f(x)_{i_1,\ldots,i_4} = \begin{cases} D^4_f(x)_1 := a^4\exp\{a x_1\}&, i_1=\ldots=i_4=1, \\ D^4_f(x)_2 := -6b^4(0.1+b x_2)^{-4}&, i_1=\ldots=i_4=2, \\ 0&\text{, else}.\end{cases}\]
Recall from \eqref{eq:asymptoticScndOrderBiasLLS} that the second-order bias is given by
$\biasLLS{\bandwidth}{x}{2} = \frac{1}{2}\mu_2\trace(D^2_f(x)\bandwidth^2)$. In the isotropic case, $\biasLLS{\sigma\idMatrix{d}}{x}{2}$ will vanish almost nowhere. 
Letting $\bandwidth_2(x) \equiv \idMatrix{d}$, and when optimizing
\begin{align*}
 \MSE_{1}^{}&\left(x, \sigma\bandwidth_2(x) | \bm{X}_{n}\right)\\
  &= \biasLLS{\sigma\bandwidth_2(x)}{x}{2}^2 + \varianceLPS{\sigma\bandwidth_2(x)}{x}{1}
  + \convergenceInProbability{}{\sigma^4+n^{-1}\sigma^{-d}}
\end{align*}
with respect to $\sigma$, the optimum is obtained for $\sigma^* = h_{2,n}^{}\sigma_2^*(x)$, where $h_{2,n}^{} = n^{-\frac{1}{4+d}}$ and 
\[\sigma_2^*(x) = \left[\frac{R(k)v(x)d}{4p(x)\abs{\bandwidth_2(x)}\biasLLS{\bandwidth_2(x)}{x}{2}^2}\right]^{\frac{1}{4+d}}.\]
Defining
\[\bandwidth_2^n(x) = h_{2,n}^{}\sigma_2^*(x)\bandwidth_2(x),\]
the associated performance law is then given by
\[\MSE_{1}^{}\left(x, \bandwidth_2^n(x) | \bm{X}_{n}\right) = \upperBoundedInProbability{}{h_{2,n}^4} = \upperBoundedInProbability{}{n^{-\frac{4}{4+d}}}.\]
Note that the choice of the spatial component $\sigma_2^*(x)$ only influences the \MSE by a constant factor, whereas the global decay rate $h_{2,n}^{}$ affects the decay law of the \MSE.

Further note that, if the leading bias-term vanishes for some $x\in\inputSpace$, the asymptotic analysis suggests $\sigma_2^*(x_i)\rightarrow\infty$ as $x_i\rightarrow x$.
In contrast, when dealing with finite data, the remaining higher-order bias -- which is not treated in this case -- would prevent the real \LOB from exploding.
Such roots of the leading bias-term -- even if only occurring on a set of measure zero -- will therefore negatively impact the \MSE on a substantial share of the input space.
Since our example suffers from such roots, i.e., for $x_1 = a^{-1}\log(b^2(0.1+b x_2)^{-2}a^{-2})$, and since we focus on the decay law as opposed to the optimal constant of the \MSE, we suggest to trade off the general scaling of the bias in $\bandwidth$ and $\abs{\bandwidth}$ in a label-agnostic way, being more robust.
Therefore we will replace $\biasLLS{\bandwidth_2(x)}{x}{2}$ with $\pnorm{\bandwidth_2(x)}{}^2$ in $\sigma_2^*$.

In the anisotropic case, when following the naive approach, we would first estimate $\vanishingSet{1}{x} = \condset{\bandwidth\in\SigmaSpace\!}{\!\biasLLS{\bandwidth}{x}{2} = 0} = \condset{\bandwidth\in\SigmaSpace\!}{\!\trace(D^2_f(x)\bandwidth^2) = 0}$,
which is given in particular by $\vanishingSet{1}{x} = \condset{\sigma\bandwidth_4(x)}{\sigma > 0}$, where 
\[\bandwidth_4(x) = \diag(1, s_4(x))\quad\text{with}\qquad s_4(x) = \sqrt{-D^2_f(x)_1/D^2_f(x)_2}.\]
Over $\vanishingSet{1}{x}$, we then optimize 
\begin{align*}
 \MSE_{1}^{}&\left(x, \sigma\bandwidth_4(x) | \bm{X}_{n}\right)\\
  &= \biasLLS{\sigma\bandwidth_4(x)}{x}{4}^2 + \varianceLPS{\sigma\bandwidth_4(x)}{x}{1}
  + \convergenceInProbability{}{\sigma^8+n^{-1}\sigma^{-d}}
\end{align*}
with respect to $\sigma$, where
$\biasLLS{\sigma\bandwidth_4(x)}{x}{4} = \frac{\mu_4}{24}\sigma^4(D^4_f(x)_1 + s_4(x)^4D^4_f(x)_2) = \upperBoundedInProbability{}{\sigma^4\pnorm{\bandwidth_4(x)}{}^4}$.
Hence, the optimum is obtained for $\sigma^* = h_{4,n}^{}\sigma_4^*(x)$, where $h_{4,n}^{} = n^{-\frac{1}{8+d}}$ and 
\[\sigma_4^*(x) = \left[\frac{R(k)v(x)d}{8p(x)\abs{\bandwidth_4(x)}\biasLLS{\bandwidth_4(x)}{x}{4}^2}\right]^{\frac{1}{8+d}}.\]
\fig{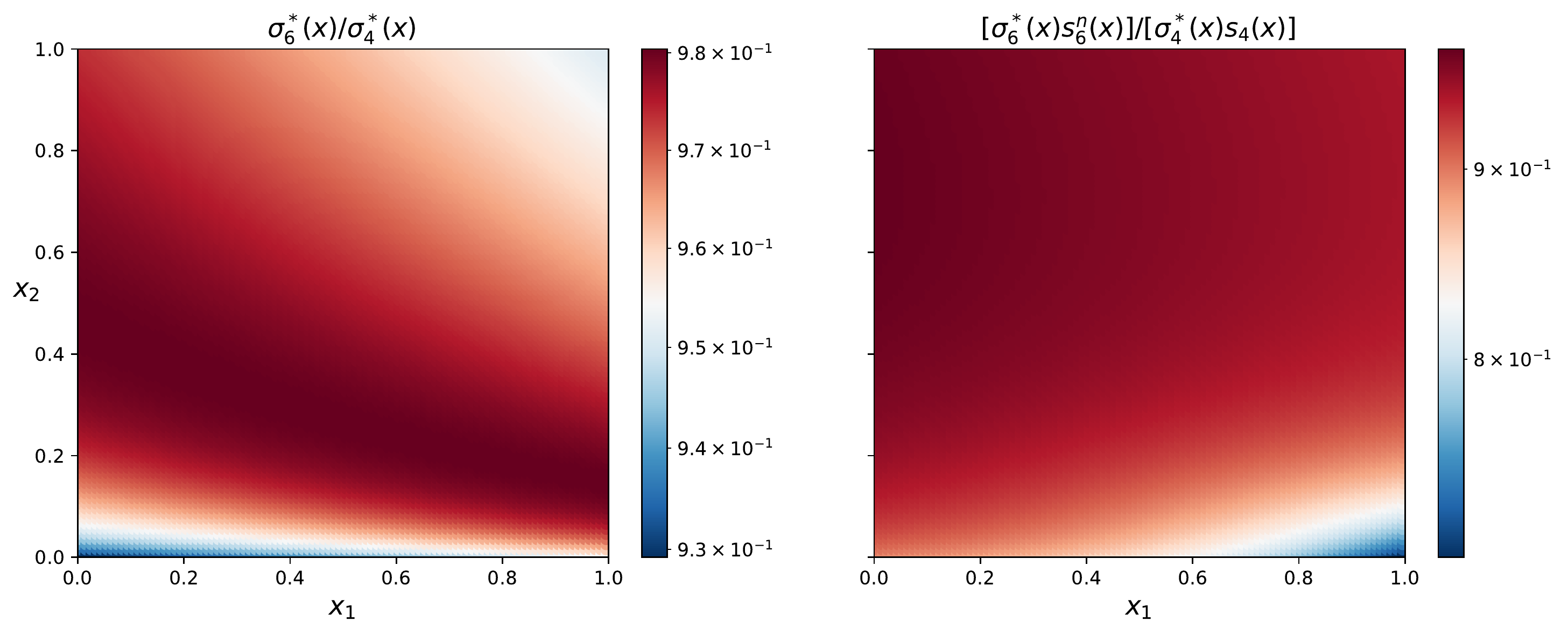}{1}{The bias canceling experiment: The ratio of the two components of the improved and the naive anisotropic bandwidth construction, evaluated at $h_n = 0.1$.}{0.8cm}{0.0cm}
Therefore the optimal bandwidth in $x$ with this construction is asymptotically given by \[\bandwidth_4^n(x) = h_{4,n}^{}\sigma_4^*(x)\bandwidth_4(x),\]
with the associated improved performance law
\[\MSE_{1}^{}\left(x, \bandwidth_4^n(x) | \bm{X}_{n}\right) = \upperBoundedInProbability{}{h_{4,n}^8} = \upperBoundedInProbability{}{n^{-\frac{8}{8+d}}}.\]
For the same stability reasons as in the isotropic case, and since we focus on the decay-law, we replace $\biasLLS{\bandwidth_4(x)}{x}{4}$ with $\pnorm{\bandwidth_4(x)}{}^4$ in $\sigma_4^*$.
We have plotted this bandwidth construction in \figref{fig:biasCancelingExperiment_dataset_anisotropicNaiveLOBestimate} (middle and right).

While the construction above seems intuitive through its straight-forward calculation, it is not optimal.
We will now show, that the vanishing second-order bias can be exploited to delete the fourth-order bias. The sequence of bandwidths $(\bandwidth_6^n(x))_{n\in\N}$ that features this property will then follow the performance law
\[\MSE_{1}^{}\left(x, \bandwidth_6^n(x) | \bm{X}_{n}\right) = \upperBoundedInProbability{}{n^{-\frac{12}{12+d}}}.\]
An empirical comparison of \LOB and \MSE decay laws obtained by the three constructions can be found in \figref{fig:biasCancelingExperiment_empirical_vs_asymptoticDecayLaws}.
Note that while this is still not the optimal law, it confirms the sub-optimality of the naive construction and demonstrates the true principle that \LOB follows in non-isotropic scenarios.

\fig{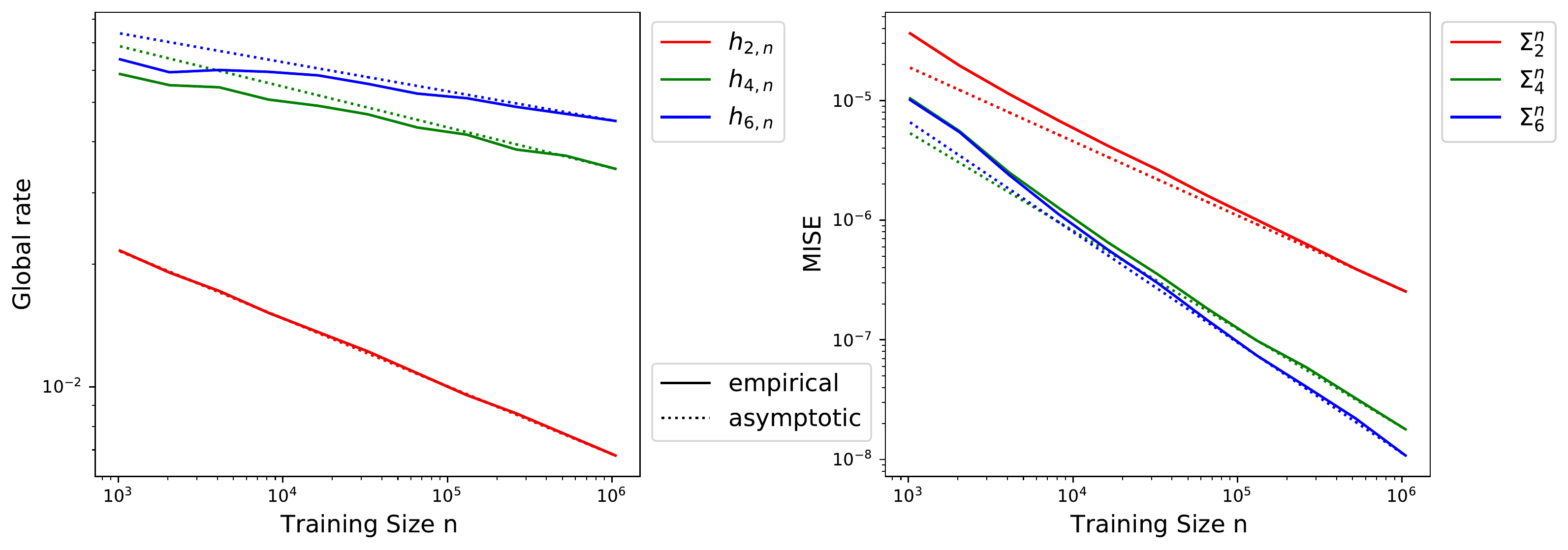}{1}{The bias canceling experiment: Empirical results and theoretically given, asymptotic decay laws of the global scaling (left) and the achieved \MISE (right) in the isotropic, naive anisotropic and improved anisotropic bandwidth case. The experimental results are averaged over 10 repetitions. The \MISE is measured on the interior $[0.1,0.9]^2$ of the input space.}{0.8cm}{0.0cm}

The asymptotic forms of bias and variance were derived for a fixed $\bandwidth$, where only the rate $h_n$ changes with $n$.
For example, for a fixed $\bandwidth$ consider the fourth-order expansion
\[\textstyle\Bias_{1}^{}\left[x, h_n^{}\bandwidth | \bm{X}_{n}\right] = \biasLLS{h_n^{}\bandwidth}{x}{2} + \biasLLS{h_n^{}\bandwidth}{x}{4} + \convergenceInProbability{}{h_n^4}.\]

When we seek for a sequence $\bandwidth_x^n$ that eliminates the fourth-order bias-term by the second-order bias-term, we cannot simply solve the asymptotic equation
\[h_n^4\biasLLS{\bandwidth_x^n}{x}{4} = - h_n^2\biasLLS{\bandwidth_x^n}{x}{2},\]
in order to obtain $\textstyle\Bias_{1}^{}\left[x, h_n^{}\bandwidth_x^n | \bm{X}_{n}\right] = \convergenceInProbability{}{h_n^4}$.
This is because the asymptotic form holds not true in first place for a changing $\bandwidth_x^n$ in $n$.
We address this issue as follows.
For an arbitrary function $F\in\diffableFunctions{\inputSpace}{0}$ consider its Monte Carlo integration
\[\frac{1}{n}\mySum{i=1}{n}\hspace{-2pt}k^{\bandwidth}(x_i^{}-x)F(x_i) = \int k(u)p(x+\bandwidth u)F(x+\bandwidth u)du + O(n^{-\frac{1}{2}}).\]
The proof of \propref{prop:biasTermForm} is based on Taylor-expansions of
several such terms: Recall that
\begin{align*}
\textstyle\Bias_{1}^{}\left[x, h_n^{}\bandwidth | \bm{X}_{n}\right]
&= N_x^{h_n^{}\bandwidth} X_{1}^{}(x)^\top W_x^{h_n^{}\bandwidth}\hspace{-2pt}\left[\frac{1}{n}\mySum{l=2}{L}T_l^{f,x}(\bm{X}_n^{})\right] + \convergenceInProbability{}{h_n^{L}},
\end{align*}
where it was
\[N_x^{h_n^{}\bandwidth} = e_1^\top\left(\frac{1}{n}\sum\limits_{i=1}^n k^{h_n^{}\bandwidth}(x_i^{},x)\begin{bmatrix} 1 & (x_i^{}-x)^\top \\ x_i^{}-x & (x_i^{}-x)(x_i^{}-x)^\top\end{bmatrix}\right)^{-1}\]
and
\[X_{1}^{}(x)^\top W_x^{h_n^{}\bandwidth}\hspace{-2pt}\left[\frac{1}{n}T_l^{f,x}(\bm{X}_n^{})\right] = \frac{1}{n}\sum\limits_{i=1}^n k^{h_n^{}\bandwidth}(x_i^{},x)\begin{bmatrix} 1 \\ x_i^{}-x \end{bmatrix}T_l^{f,x}(x_i).\]

Now, the first derivative of the Gaussian kernel with respect to the bandwidth is given by
\[\frac{\partial k^{\bandwidth}(z)}{\partial \bandwidth} = -k^{\bandwidth}(z)(\idMatrix{d} - \bandwidth^{-1}zz^\top\bandwidth^{-1})\bandwidth^{-1}.\]
For $\bandwidth = \diag(\sigma_1,\ldots,\sigma_d)$, the second derivative of the Gaussian kernel is given by
\[\frac{\partial^2 k^{\bandwidth}(z)}{\partial\sigma_i\partial\sigma_j} = k^{\bandwidth}(z)\left[(1-\frac{z_i^2}{\sigma_i^2})(1-\frac{z_j^2}{\sigma_j^2}) + \delta_{ij}(1-3\frac{z_i^2}{\sigma_i^2})\right].\]
Now, let $x\in\inputSpace$ and $\bandwidth_x = \diag(\sigma_1,\sigma_2)\in\SigmaSpace$ be fixed.
For a sequence $\bandwidth_x^n= \diag(\sigma^n_1,\sigma^n_2)$ we define $\varepsilon_n = \pnorm{\bandwidth_x^{-1}(\bandwidth_x^n - \bandwidth_x)}{}$. If $h_n^{}, \varepsilon_n = \convergenceInProbability{}{1}$, we can expand
\begin{align*}
 &\frac{1}{n}\mySum{i=1}{n}\hspace{-2pt}k^{h_n^{}\bandwidth_x^n}(x_i^{}-x)F(x_i)\\
 &= \frac{1}{n}\mySum{i=1}{n}\hspace{-2pt}F(x_i)k^{h_n^{}\bandwidth_x}(x_i^{}-x)\bigg[
 1 - \mySum{j=1}{d} [1-\frac{z_j^2}{\sigma_j^2}][\frac{\sigma^n_j}{\sigma_j} - 1]\\
 &+ \frac{1}{2}\mySum{j,k=1}{d}\left[ [1-\frac{z_j^2}{\sigma_j^2}][1-\frac{z_k^2}{\sigma_k^2}] + \delta_{j,k}[1-3\frac{z_j^2}{\sigma_j^2}]\right] [\frac{\sigma^n_j}{\sigma_j} - 1][\frac{\sigma^n_k}{\sigma_k} - 1] + \convergenceInProbability{}{\varepsilon_n^2}\bigg]\\
 &= \int du k(u)p(x+h_n^{}\bandwidth_x u)F(x+h_n^{}\bandwidth_x u)\bigg[1 - \mySum{j=1}{d} [1-u_j^2][\frac{\sigma^n_j}{\sigma_j} - 1]\\
 &+ \frac{1}{2}\mySum{j,k=1}{d}\left[[1-u_j^2][1-u_k^2] + \delta_{j,k}[1-3u_j^2]\right] [\frac{\sigma^n_j}{\sigma_j} - 1][\frac{\sigma^n_k}{\sigma_k} - 1] + \convergenceInProbability{}{\varepsilon_n^2}\bigg].
\end{align*}
We first consider $F(z) := T_2^{f,x}(z) = \frac{1}{2}(z-x)^\top D^2_f(x)(z-x)$ in order to deal with the second-order bias.
Here, we can choose the fixed bandwidth $\bandwidth_x = \sigma\bandwidth_4(x) \in \vanishingSet{1}{x}$ such that $\biasLLS{\bandwidth_x}{x}{2} = 0$, which removes the second-order error. Additionally, we set up the candidate solution $\bandwidth_x^n = \sigma\diag(1, s_6^n(x))$.
We then observe that $\frac{\sigma^n_1}{\sigma_1} - 1 = 0$ and $\frac{\sigma^n_2}{\sigma_2} - 1 = \frac{s_6^n(x)}{s_4(x)} - 1 = \varepsilon_n$.
With $p \equiv 1$ it is
\begin{align*}
 &\frac{1}{n}\mySum{i=1}{n}\hspace{-2pt}k^{h_n^{}\bandwidth_x^n}(x_i^{}-x)T_2^{f,x}(x_i) + \convergenceInProbability{}{h_n^{2}\varepsilon_n^2}\\
 &= \int du k(u)T_2^{f,x}(x+h_n^{}\bandwidth_x u)\left[1 - \varepsilon_n[1-u_2^2] + \frac{1}{2}\varepsilon_n^2([1-u_2^2]^2 + [1-3u_2^2])\right] \\
 &= \underbrace{\biasLLS{h_n^{}\bandwidth_x}{x}{2}}_{= 0}(1 - \varepsilon_n + \varepsilon_n^2) + \frac{1}{2}h_n^2\sigma^2\bigg[\varepsilon_n\bigg[D_f^2(x)_1\mu_2^2 + \underbrace{s_4(x)^2D_f^2(x)_2}_{-D_f^2(x)_1}\mu_4\bigg]\\
 &+ \frac{1}{2}\varepsilon_n^2\left[D_f^2(x)_1(\mu_2\mu_4-5\mu_2^2)
 + s_4(x)^2D_f^2(x)_2(\mu_6 - 5\mu_4)\right]\bigg]\\
 &= \frac{1}{2}h_n^2\sigma^2\bigg[\varepsilon_n D_f^2(x)_1\left[\mu_2^2 - \mu_4\right]
 + \frac{1}{2}\varepsilon_n^2 D_f^2(x)_1\left[\mu_2\mu_4-5\mu_2^2
 - \mu_6 + 5\mu_4\right]\bigg].
\end{align*}
Analogously, when dealing with the fourth-order bias, we set
\[F(z) := T_4^{f,x}(z) = \frac{1}{24}\left[(z_1-x_1)^4D_f^4(x)_1 + (z_2-x_2)^4D_f^4(x)_2\right],\]
giving
\begin{align*}
 \frac{1}{n}\mySum{i=1}{n}\hspace{-2pt}k^{h_n^{}\bandwidth_x^n}&(x,x_i^{})T_4^{f,x}(x_i) + \upperBoundedInProbability{}{h_n^{4}\varepsilon_n^2}\\
 =\;&h_n^4\sigma^4\biasLLS{\bandwidth_4(x)}{x}{4}\\
 &- \frac{h_n^4\sigma^4}{24}\int du k(u)[u_1^4D_f^4(x)_1 + u_2^4s_4(x)^4D_f^4(x)_2]\varepsilon_n[1-u_2^2]\\
 =\;&h_n^4\sigma^4\biasLLS{\bandwidth_4(x)}{x}{4}\\
 &+ \frac{h_n^4\sigma^4\varepsilon_n}{24}\left[D_f^4(x)_1[\mu_2\mu_4-\mu_4] + s_4(x)^4D_f^4(x)_2[\mu_6 - \mu_4]\right].
\end{align*}
Finally, we can write $N_x^{h_n^{}\bandwidth} = \begin{bmatrix} 1 + \convergenceInProbability{}{\varepsilon_n^2} & 0\end{bmatrix}$.
Therefore we can expand
\begin{align*}
 \sideset{}{_{1}^{}}\Bias\left(x, h_n^{}\bandwidth_x^n | \bm{X}_{n}\right) = 
 &\frac{1}{2}h_n^2\sigma^2\varepsilon_n D_f^2(x)_1\left[\mu_2^2 - \mu_4\right]\\
 &+ h_n^4\sigma^4\biasLLS{\bandwidth_4(x)}{x}{4} + \upperBoundedInProbability{}{h_n^{2}\varepsilon_n^2 + h_n^{4}\varepsilon_n + h_n^{6}}.
\end{align*}
We can solve $\sideset{}{_{1}^{}}\Bias\left(x, h_n^{}\bandwidth_x^n | \bm{X}_{n}\right) = \upperBoundedInProbability{}{h_n^{2}\varepsilon_n^2 + h_n^{4}\varepsilon_n + h_n^{6}}$ for $s_6^n(x)$, recalling that $\varepsilon_n = s_6^n(x)\big/s_4(x) - 1$, which gives
\[s_6^n(x) = s_4(x)\left(1 - \frac{2h_n^{2}\sigma^2\biasLLS{\bandwidth_4(x)}{x}{4}}{D_f^2(x)_1\left[\mu_2^2 - \mu_4\right]}\right).\]
With this, it is $\varepsilon_n = \upperBoundedInProbability{}{h_n^2}$
such that $\sideset{}{_{1}^{}}\Bias\left(x, h_n^{}\bandwidth_x^n | \bm{X}_{n}\right) = \upperBoundedInProbability{}{h_n^6}$.

Applying $\bandwidth_6^n(x) = \sigma\diag(1, s_6^n(x))$, we finally aggregate all sixth-order bias-terms that arise as $\Bias\left(x, \bandwidth_6^n(x) | \bm{X}_{n}\right) = \sigma^6b_6(x,\bandwidth_4(x)) + \convergenceInProbability{}{\sigma^6}$, where
\begin{align*}
&b_6(x,\bandwidth_4(x)) = \biasLLS{\bandwidth_4(x)}{x}{6}\\ &+\frac{1}{4}(\frac{2\biasLLS{\bandwidth_4(x)}{x}{4}}{D_f^2(x)_1\left[\mu_2^2 - \mu_4\right]})^2D_f^2(x)_1\left[\mu_2\mu_4-5\mu_2^2
 - \mu_6 + 5\mu_4\right]\\
&+ \frac{1}{24}(\frac{2\biasLLS{\bandwidth_4(x)}{x}{4}}{D_f^2(x)_1\left[\mu_2^2 - \mu_4\right]})\left[D_f^4(x)_1(\mu_2\mu_4-\mu_4) + s_4(x)^4D_f^4(x)_2(\mu_6 - \mu_4)\right],
\end{align*}
and $\biasLLS{\bandwidth_4(x)}{x}{6} = \frac{\mu_6}{720}\left[D_f^6(x)_1 + D_f^6(x)_2s_4(x)^6\right]$.
Accordingly, we write
\begin{align*}
\MSE_{1}^{}&\left(x, \bandwidth_6^n(x) | \bm{X}_{n}\right)\\
 &= \sigma^{12}b_6(x,\bandwidth_4(x))^2 + \varianceLPS{\sigma\bandwidth_4(x)}{x}{1}
  + \convergenceInProbability{}{\sigma^{12}+n^{-1}\sigma^{-d}}.
\end{align*}
Here, the optimum is obtained for $\sigma^* = h_{6,n}^{}\sigma_6^*(x)$, where $h_{6,n}^{} = n^{-\frac{1}{12+d}}$ and 
\[\sigma_6^*(x) = \left[\frac{R(k)v(x)d}{12p(x)\abs{\bandwidth_4(x)}b_6(x,\bandwidth_4(x))^2}\right]^{\frac{1}{12+d}}.\]
Thus, for $\bandwidth_6^n(x) = h_{6,n}^{}\sigma_6^*(x)\diag(1, s_6^n(x))$, it is
\begin{align*}
 &\MSE_{1}^{}\left(x, \bandwidth_6^n(x) | \bm{X}_{n}\right)
 = \upperBoundedInProbability{}{h_{6,n}^{12}} = \upperBoundedInProbability{}{n^{-\frac{12}{12+d}}}.
\end{align*}
Like in the previous cases we replace $b_6(x,\bandwidth_4(x))$ with $\pnorm{\bandwidth_4(x)}{}^6$ in $\sigma_6^*$.
We have plotted this bandwidth construction relative to the native, anisotropic construction in \figref{fig:biasCancelingExperiment_ratioBias4vsBias2cancelingBandwidth}.

So the trick is not to eliminate each respective higher-order bias-term on its own, but to abuse the slow convergence of the second-order bias-term together with a systematic deviation from its root in order to generate an anti-bias that cancels the effect of higher-order bias-terms.
Note that this phenomenon works only in the direction of canceling higher-order bias from strictly vanishing lower-order bias and not vice versa.
Following the concept of the toy-example, we are now in the position to analyze locally optimal non-isotropic bandwidths of \LLS in the indefinite regime.

\section{Non-isotropic \LOB of \LLS for indefinite functions}
\label{sec:nonIsotropicLOB}
In the following, we focus on the interior of the input space, denoted by \inputSpaceInterior, and assume the kernel function $k$ to be of bounded support.
In particular, let $s>0$ such that $k(t) \equiv 0, \forall t > s$. 
When analyzing the asymptotics of \LPS, we generally make the following consistency requirement on the applied bandwidth matrix as the sample size grows (see \cite{ruppert1994}):
\begin{mydef}
\label{def:consistentBandwidth}
A bandwidth sequence $(\bandwidth^n)_{n\in\N}$ is consistent, if
\begin{align*}
 \abs{\bandwidth^n}^{-1}n^{-1} \rightarrow 0,\qquad\pnorm{\bandwidth^n}{} \rightarrow 0\quad\text{and}\qquad\pnorm{\bandwidth^n}{} \cdot \pnorm{[\bandwidth^n]^{-1}}{} \leq \conditionNumberUpperBound
\end{align*}
for some global constant $1 \leq \conditionNumberUpperBound < \infty$.
\end{mydef}
These requirements are necessary -- though not sufficient -- to make the \LLS predictor \emph{weakly universally consistent} in first place:
The first two requirements are known from the isotropic bandwidth analysis. They make the pure kernel weights consistent, which appear in the Nadaraya-Watson estimator (\LPS of order $\LPSorder = 0$)
\citep[Theorem 5.1 and 5.4]{gyorfi2002distribution}.
When considering \LPS of higher order, it was shown that the \emph{local linear weight function} can be trimmed, relative to the pure kernel weights, in order to make them consistent \citep[Corollary 4]{stone1977consistent}.
In case of non-isotropic bandwidths, a bound on the conditioning number is additionally required to make the local linear weight function bounded:
For an unbounded weight function, trimming would lead to predictions that are arbitrarily off the \LPS model.
\begin{myRemark}
For a specific regression task, when searching for \LOB over the unbounded set \SigmaSpace, it will occur that for a sequence of optimizers $\bandwidth_x^n$ of $\MSE_{1}^{}\left(x, \bandwidth| \bm{X}_{n}\right)$, it is $\pnorm{\bandwidth_x^n}{}, \not\rightarrow 0$ or even $\pnorm{\bandwidth_x^n}{}\rightarrow \infty$.
To see this, consider $\bandwidth^n = n\idMatrix{d}$. Then the \LLS predictor
$\predictorLPS{1}{\bandwidth^n}$ asymptotically becomes a global linear fit to $f$.
Necessarily, $\E\left[\predictorLPS{1}{\bandwidth^n}(\cdot) | \bm{X}_n^{}\right]$ intersects the function $f$ in some $x\in\inputSpaceInterior$.
If it would not, we could shift the offset until the linear fit touches $f$, by which we reduce the \MSE everywhere over \inputSpace. In particular, this means that there always exists some $x\in\inputSpaceInterior$ where \LLS is totally free of bias for $\bandwidth^n = n\idMatrix{d}$. Hence, $\MSE_{1}^{}\left(x, n\idMatrix{d}| \bm{X}_{n}\right) = \upperBoundedInProbability{}{n^{-1}}$ at a rate that is superior to any bandwidth sequence with $\pnorm{\bandwidth_x^n}{} \rightarrow 0$.
\end{myRemark}
This consistency assumption is also made in the related work, when analyzing the isotropic and the non-isotropic, definite cases. Recall that there, in addition, a non-vanishing leading bias-term is required.
This requirement enables an explicit construction of \LOB, which particularly answers the question of the existence (and uniqueness) of a minimizer.
In contrast, in the indefinite case it will turn out that no explicit construction of \LOB is possible.
Yet, we are able to prove the existence of a minimizer by assuming certain pointwise regularity of the problem.
For convenience, recall from \eqref{def:consistentSequenceSpace} the space of consistent sequences
\begin{align*}
\consistentSigmaSequencSpace = \condset{(\bandwidth^n)_{n\in\N} \subset \SigmaSpace}{\abs{\bandwidth^n}^{-1}n^{-1} \rightarrow 0, \pnorm{\bandwidth^n}{} \rightarrow 0, \pnorm{\bandwidth^n}{} \cdot \pnorm{[\bandwidth^n]^{-1}}{} \leq \conditionNumberUpperBound}.
\end{align*}
In order to prove \thmref{thm:generalizedLOBandMSE}, we will first show in \lemref{lem:minimizerExistence} that there exists an optimal sequence $(\bandwidth_x^n)_{n\in\N} \in \consistentSigmaSequencSpace$, for which $\MSE_{1}^{}\left(x, \bandwidth_x^n| \bm{X}_{n}\right) = \lowerBoundedInProbability{}{n^{-\frac{2\alpha}{2\alpha+d}}}$ must necessarily hold.
Then we show in \lemref{lem:achievableRate} that we are able to construct a sequence $(\bandwidth_x^n)_{n\in\N} \in \consistentSigmaSequencSpace$ such that $\MSE_{1}^{}\left(x, \bandwidth_x^n | \bm{X}_{n}\right) = \upperBoundedInProbability{}{n^{-\frac{2\alpha}{2\alpha+d}}}$.
\begin{mylem}[Minimizer Existence]
\label{lem:minimizerExistence}
Let $x\in\inputSpaceInterior$ such that $f$ is indefinite in $x$ and for $\alpha := \alpha(f,x)$ it is $2 < \alpha < \infty$. Furthermore assume that $v$ is continuous in $x$ and $p\in\Lambda^{\alpha-2}(x)$ with $p(x), v(x) > 0$.
Then, there exists $(\bandwidth_x^n)_{n\in\N} \in \consistentSigmaSequencSpace$ and $N_x \in\N$ such that for all $(\bandwidth^n)_{n\in\N} \in \consistentSigmaSequencSpace$, if $n \geq N_x$, it holds that
\[\textstyle\MSE_{1}^{}\left(x, \bandwidth_x^n| \bm{X}_{n}\right) \leq \MSE_{1}^{}\left(x, \bandwidth^n | \bm{X}_{n}\right).\]
For such a sequence, $\MSE_{1}^{}\left(x, \bandwidth_x^n| \bm{X}_{n}\right) = \lowerBoundedInProbability{}{n^{-\frac{2\alpha}{2\alpha+d}}}$ must hold.
\end{mylem}
\begin{proof}
For an arbitrary sequence $(\bandwidth^n)_{n\in\N} \in \consistentSigmaSequencSpace$ that fulfills the consistency condition in \defref{def:consistentBandwidth},
let $h_n = \pnorm{\bandwidth^n}{}$ and $\bar \bandwidth^n = h_n^{-1}\bandwidth^n$.
Since $\alpha(f,x) < \infty$ and $\pnorm{\bandwidth^n}{}\rightarrow 0$ as $n\rightarrow\infty$, it follows for the bias in $x$ that
\[
\textstyle\Bias_{1}^{}\left[x, \bandwidth^n | \bm{X}_{n}\right] = \lowerBoundedInProbability{}{\pnorm{\bandwidth^n}{}^{\alpha}}.
\]
Hence, we can write
\[
\textstyle\Bias_{1}^{}\left[x, \bandwidth^n | \bm{X}_{n}\right]^2 \geq C(\bar \bandwidth^n)h_n^{2\alpha} + \convergenceInProbability{}{h_n^{2\alpha}},
\]
where $0 < \epsilon_x \leq C(\bar \bandwidth^n)$.
The lower bound $\epsilon_x$ exists uniformly for any bandwidth sequence: If not, we could construct $(\bandwidth_n)_{n\in\N} \in \consistentSigmaSequencSpace$ with $C(\bar \bandwidth^n)\rightarrow 0$, for which, in particular, $\Bias_{1}^{}\left[x, \bandwidth^n | \bm{X}_{n}\right]^2 = \convergenceInProbability{}{\pnorm{\bandwidth^n}{}^{\alpha}}$, in contradiction to $\Bias_{1}^{}\left[x, \bandwidth^n | \bm{X}_{n}\right] = \lowerBoundedInProbability{}{\pnorm{\bandwidth^n}{}^{\alpha}}$.

On the other hand, due to the consistency condition in \defref{def:consistentBandwidth}, it is
\[
h_n^d = \pnorm{\bandwidth^n}{}^d \geq \abs{\bandwidth^n} \geq \pnorm{\bandwidth^n}{} / \pnorm{[\bandwidth^n]^{-1}}{}^{d-1} \geq h_n^d \conditionNumberUpperBound^{d-1},
\]
such that with $\Var_{1}^{}\left[x, \bandwidth^n | \bm{X}_{n}\right] = \frac{R(k)v(x)}{p(x)n\abs{\bandwidth^n}} + \convergenceInProbability{}{n^{-1}\abs{\bandwidth^n}^{-1}}$ it is
\[
\textstyle \frac{R(k)v(x)}{p(x)nh_n^d} \leq \Var_{1}^{}\left[x, \bandwidth^n | \bm{X}_{n}\right] + \convergenceInProbability{}{n^{-1}h_n^{d}}
\leq \frac{R(k)v(x)}{p(x)nh_n^d\conditionNumberUpperBound^{d-1}}.
\]

Therefore $\Var_{1}^{}\left[x, \bandwidth^n | \bm{X}_{n}\right] = \exactRateconvergenceInProbability{\noexpand\big}{n^{-1}h_n^{-d}}$, such that
\begin{align*}
\MSE_{1}^{}&\textstyle\left(x, \bandwidth^n| \bm{X}_{n}\right) = \Bias_{1}^{}\left[x, \bandwidth^n | \bm{X}_{n}\right]^2 + 
\Var_{1}^{}\left[x, \bandwidth^n | \bm{X}_{n}\right]\\
&\geq C(\bar \bandwidth^n) h_n^{2\alpha} + \exactRateconvergenceInProbability{\noexpand\big}{n^{-1}h_n^{-d}} + \convergenceInProbability{\noexpand\big}{h_n^{2\alpha} + n^{-1}h_n^{-d}}\\
&\geq \epsilon_x h_n^{2\alpha} + \exactRateconvergenceInProbability{\noexpand\big}{n^{-1}h_n^{-d}} + \convergenceInProbability{\noexpand\big}{h_n^{2\alpha} + n^{-1}h_n^{-d}}.
\end{align*}
For this lower bound of the \MSE, the optimal decay rate is given by $\underline{h}_n = n^{-\frac{1}{2\alpha+d}}$.
Note that the last claim of the lemma directly follows, since $\MSE_{1}^{}\left(x, \bandwidth^n| \bm{X}_{n}\right) = \lowerBoundedInProbability{}{\underline{h}_n^{2\alpha}} = \lowerBoundedInProbability{\noexpand\big}{n^{-\frac{2\alpha}{2\alpha+d}}}$.

Let us now define by $\gamma = \frac{1}{2\alpha+d+1} > 0$ a strict lower bound on the bandwidth decay exponent.
The larger the bias is, the faster an optimal bandwidth sequence has to decay: Since $\Bias_{1}^{}\left[x, \bandwidth^n | \bm{X}_{n}\right] = \lowerBoundedInProbability{\noexpand\big}{\pnorm{\bandwidth^n}{}^{\alpha}}$, let us assume without loss of generality that $\Bias_{1}^{}\left[x, \bandwidth^n | \bm{X}_{n}\right] = \exactRateconvergenceInProbability{\noexpand\big}{\pnorm{\bandwidth^n}{}^{\beta}}$ for some $\beta \leq \alpha$.
Then, the optimal decay rate becomes $h_n^* = n^{-\frac{1}{2\beta+d}} = \upperBoundedInProbability{}{\underline{h}_n} = \convergenceInProbability{}{n^{-\gamma}}$.
We can therefore narrow down the optimal bandwidth candidates to the sets
\[
\textstyle S_{n} = \condset{\bandwidth\in\SigmaSpace}{\pnorm{\bandwidth}{} \leq n^{-\gamma}, \abs{\bandwidth} \geq n^{-1}}.
\]
On the one hand, the sets $S_{n}$ are compact. Using that $\MSE_{1}^{}\left(x, \bandwidth| \bm{X}_{n}\right)$ is continuous in $\bandwidth\in\SigmaSpace$, there exists $\bandwidth_x^n\in S_{n}$ such that
\[
\textstyle\MSE_{1}^{}\left(x, \bandwidth_x^n| \bm{X}_{n}\right) = \min_{\bandwidth\in S_{n}} \MSE_{1}^{}\left(x, \bandwidth| \bm{X}_{n}\right).
\]
On the other hand, any sequence that asymptotically resides outside of $S_{n}$, is necessarily suboptimal or not consistent.
In this regard, consider any sequence $(\bandwidth^n)_{n\in\N} \in \consistentSigmaSequencSpace$. Then
\begin{align*}
 \textstyle\liminf_{n\rightarrow\infty}\MSE_{1}^{}\left(x, \bandwidth^n| \bm{X}_{n}\right) - \MSE_{1}^{}\left(x, \bandwidth_x^n| \bm{X}_{n}\right) \geq 0.\qedAtBottomLine
\end{align*}
\end{proof}

\begin{mylem}[Achievable Rate]
\label{lem:achievableRate}
Let $\SigmaSpace\subseteq\posDefSet{d}$ be conic with $\dim(\SigmaSpace) \geq 2$.
Let $x\in\inputSpaceInterior$ where for $\alpha := \alpha(f,x)$ it holds that $2 < \alpha < \infty$ and write $\alpha = L + \beta$ with $L\in\N, L \geq 2$ and $\beta\in(0,1]$. Furthermore assume that $k \in \diffableFunctions{\nonnegativeReal{}, \nonnegativeReal{}}{\lfloor L/2\rfloor}$, $v$ is continuous in $x$ and $p\in\Lambda^{\alpha-2}(x)$ with $p(x), v(x) > 0$.
If $f$ is indefinite in $x$,
there exists a sequence $(\bandwidth_x^n)_{n\in\N} \in \consistentSigmaSequencSpace$ such that with $h_n^{} = n^{-\frac{1}{2\alpha+d}}$,
\begin{align*}
 &\textstyle\MSE_{1}^{}\left(x, \bandwidth_x^n | \bm{X}_{n}\right)
 = \upperBoundedInProbability{}{h_{n}^{2\alpha}} = \upperBoundedInProbability{}{n^{-\frac{2\alpha}{2\alpha+d}}}.
\end{align*}
\end{mylem}
\begin{proof}
Recall the vanishing set of the second-order bias, given by
\[\vanishingSet{1}{x} = \condset{\bandwidth\in\SigmaSpace}{\biasLLS{\bandwidth}{x}{2} = 0} = \condset{\bandwidth\in\SigmaSpace}{\trace(D^2_f(x)\bandwidth^2) = 0},\]
and let $\bandwidth_x \in \vanishingSet{1}{x}$ be fixed.
We will show that for $n$ large enough there always exists  $\bandwidth_x^n \in \SigmaSpace$ close to $\bandwidth_x$ such that with $h_n \rightarrow 0$ we will have $\Bias_{\LPSorder}^{}\left(x, h_n^{}\bandwidth_x^n | \bm{X}_{n}\right) = \upperBoundedInProbability{}{h_n^\alpha}$:

Let $l = \lfloor L/2\rfloor$. For $E_n \in \symmetricSet{d}$ we define $\bandwidth_x^n = \bandwidth_x + E_n$.
Letting $\varepsilon_n := \pnorm{E_n}{}$ it is $\bandwidth_x^n \in \posDefSet{d}$ for $\varepsilon_n$ small enough.

Assuming $h_n, \varepsilon_n \rightarrow 0$, and using the differentiability of the kernel $k$, we can expand
\begin{align*}
 \textstyle\Bias_{1}^{}\left[x, h_n^{}\bandwidth_x^n | \bm{X}_{n}\right] = \Bias_{1}^{}\left[x, h_n^{}\bandwidth_x | \bm{X}_{n}\right] + \!\!\sum\limits_{l=1}^{\lfloor{L/2}\rfloor}\sum\limits_{k=1}^{m}b_{1,2l,k}(x, h_n^{}\bandwidth_x, E_n) + \upperBoundedInProbability{}{h_n^\alpha},
\end{align*}
where we aggregate with $b_{1,2l,k}(x, h_n^{}\bandwidth_x, E_n) = \exactRateconvergenceInProbability{\noexpand\big}{h_n^{2l}\varepsilon_n^{k}}$ all terms of the respective order that may appear. It was
\begin{align*}
\textstyle\Bias_{1}^{}\left[x, h_n^{}\bandwidth | \bm{X}_{n}\right] = \textstyle \mySum{l=1}{\lfloor{L/2}\rfloor }\biasLLS{h_n^{}\bandwidth}{x}{2l} + \upperBoundedInProbability{}{h_n^{\alpha}},
\end{align*}
where we have by $\bandwidth_x \in \vanishingSet{1}{x}$ that the first term $\biasLLS{h_n^{}\bandwidth_x}{x}{2} = 0$.

We now separate all terms into two groups as follows:
\[t_1^n(E_n) = \mySum{k=1}{m}b_{1,2,k}(x, h_n^{}\bandwidth_x, E_n)\]
and
\[t_2^n(E_n) = \mySum{l=2}{\lfloor{L/2}\rfloor}\left[\biasLLS{h_n^{}\bandwidth}{x}{2l} + \mySum{k=1}{m}b_{1,2l,k}(x, h_n^{}\bandwidth_x, E_n)\right].\]
Therefore we have $\textstyle\Bias_{1}^{}\left[x, h_n^{}\bandwidth_x^n | \bm{X}_{n}\right] = t_1^n(E_n) + t_2^n(E_n) + \upperBoundedInProbability{}{h_n^\alpha}$.

\begin{enumerate}
 \item if $\varepsilon_n \rightarrow 0$, then $\sign(t_2^n(E_n))$ does not depend on $\sign(E_n)$. This is because the leading term of $t_2^n(E_n)$ is $\biasLLS{h_n^{}\bandwidth}{x}{4} = \upperBoundedInProbability{}{h_n^4}$ in this case.
 \item There exists $E\in\symmetricSet{d}$ such that $\bandwidth_x \pm E \in \posDefSet{d}$ and $\sign(t_1^n(E_n)) = \pm 1$
 \item $t_1^n(\lambda E) \rightarrow 0$ as $\lambda\rightarrow 0$.
 \item For fixed $E_n \equiv E$ with $\bandwidth_x^n = \bandwidth_x + E$ it is $t_1^n(E) = \exactRateconvergenceInProbability{}{h_n^2}$, whereas $t_2^n(E) = \upperBoundedInProbability{}{h_n^4}$.
\end{enumerate}
Because of (1) and (2), we can choose $\sign(t_1^n) = - \sign(t_2^n)$ for $n$ large enough. When plugging in the $E$ from (2) in (4), for $n$ large enough $\abs{t_2^n(E)} \leq \abs{t_1^n(E)}$. When we increase $n$ beyond that, we find $\lambda_n\in(0,1)$ with (3) such that $\abs{t_2^n(\lambda_n E)} \equiv \abs{t_1^n(\lambda_n E)}$. Here it is necessarily $\lambda_n = \upperBoundedInProbability{}{\varepsilon_n}\rightarrow 0$.
Now, with $E_n = \lambda_n E$ and $\abs{t_2^n(E_n)} \equiv \abs{t_1^n(E_n)}$ and $\sign(t_1^n(E_n)) = - \sign(t_2^n(E_n))$, it follows $t_1^n(E_n) + t_2^n(E_n) \equiv 0$ such that 
\[\textstyle\Bias_{1}^{}\left[x, h_n^{}\bandwidth_x^n | \bm{X}_{n}\right] = t_1^n(E_n) + t_2^n(E_n) + \upperBoundedInProbability{}{h_n^\alpha} = \upperBoundedInProbability{}{h_n^\alpha}.\]
Note that, since $t_1^n(E_n) = \upperBoundedInProbability{}{h_n^2\varepsilon_n}$, $t_2^n(E_n) = \upperBoundedInProbability{}{h_n^4}$ and $t_1^n(E_n) + t_2^n(E_n) \equiv 0$, it is necessarily $\varepsilon_n = \upperBoundedInProbability{}{h_n^2}$.
Furthermore, with\newline $\Var_{1}^{}\left[x, h_n^{}\bandwidth_x^n | \bm{X}_{n}\right] = \Var_{1}^{}\left[x, h_n^{}\bandwidth_x | \bm{X}_{n}\right](1 + \upperBoundedInProbability{}{\varepsilon_n}) = \upperBoundedInProbability{}{h_n^{-d}n^{-1}}$ it is
\begin{align*}
 \MSE_{1}^{}&\textstyle\left(x, h_n\bandwidth_x^n | \bm{X}_{n}\right)\\
 &\textstyle= \Bias_{1}^{}\left[x, h_n^{}\bandwidth_x^n | \bm{X}_{n}\right]^2 + \Var_{1}^{}\left[x, h_n^{}\bandwidth_x^n | \bm{X}_{n}\right]
 = \upperBoundedInProbability{}{h_{n}^{2\alpha} + h_n^{-d}n^{-1}}.
\end{align*}
With the optimal trade-off, being $h_n^{} = n^{-\frac{1}{2\alpha+d}}$, it is therefore \[\textstyle\MSE_{1}^{}\left(x, h_n\bandwidth_x^n | \bm{X}_{n}\right) = \upperBoundedInProbability{}{h_{n}^{2\alpha}} = \upperBoundedInProbability{}{n^{-\frac{2\alpha}{2\alpha+d}}},\]
as claimed.
\end{proof}
In the light of \lemref{lem:minimizerExistence}, the bandwidth sequence from \lemref{lem:achievableRate} is rate-optimal, such that up to a constant, no better result can be obtained with the \LLS model class, using a consistent bandwidth sequence.

\generalizedLOBandMSE*
\begin{proof}
The claim immediately follows from combining \lemref{lem:minimizerExistence} and \ref{lem:achievableRate}.
\end{proof}

\section{The reciprocal determinant of non-isotropic \LOB of \LLS for indefinite functions}
\label{sec:nonIsotropicLOBdeterminant}
Recall that we know nothing about its uniqueness of \LOB in the indefinite regime, for which reason we have introduced the set-valued definition \eqref{eq:generalizedLOBDefinition} of \LOB via
\[
 \LOBfunctionOfLPS{1, \posDefSet{d}}{n}(x)
  = \textstyle 
  \condset{\bandwidth \in S_{n}}{\MSE_{1}^{}\left(x, \bandwidth | \bm{X}_{n}\right) = \min_{\bandwidth'\in S_{n}}\MSE_{1}^{}\left(x, \bandwidth' | \bm{X}_{n}\right)}.
\]
We will now show that, for $\bandwidth_n,\bandwidth_n' \in \LOBfunctionOfLPS{1, \posDefSet{d}}{n}(x)$ 
with $\bandwidth_n \neq \bandwidth_n'$ it is $\textstyle \abs{h_n^{-1}\bandwidth_n} - \abs{h_n^{-1}\bandwidth_n'} = \convergenceInProbability{}{1}$, where $h_n^{} = n^{-\frac{1}{2\alpha+d}}$.
So with abuse of notation, $x\mapsto \abs{h_n^{-1}\LOBfunctionOfLPS{1, \posDefSet{d}}{n}(x)}$ is an asymptotically well-defined function over the input space.
First of all note that we can generalize the balancing property in \corollaryref{cor:biasVarianceBalance} to the non-isotropic case:
\begin{restatable}{mylem}{generalizedBiasVarianceBalance}
\label{lem:generalizedBiasVarianceBalance}
May the assumptions of \thmref{thm:generalizedLOBandMSE} hold and let
$(\bandwidth_x^n)_{n\in\N}\in\consistentSigmaSequencSpace$ be an optimal bandwidth sequence in $x\in\inputSpaceInterior$.
That is, $\bandwidth_x^n \in \LOBfunctionOfLPS{1, \posDefSet{d}}{n}(x)$.
The conditional bias can be expressed through the conditional variance as \[\textstyle\Bias_{1}^{}\left[x, \bandwidth_x^n | \bm{X}_{n}\right]^2 = \frac{d}{2\alpha}\varianceLPS{\bandwidth_x^n}{x}{1} + \convergenceInProbability{\noexpand\big}{n^{-\frac{2\alpha}{2\alpha+d}}}.\]
Hence the conditional \MSE can asymptotically be expressed as
\begin{align}
\label{eq:generalizedBiasVarianceOnParExpression}
\textstyle\MSE_{1}^{}\left(x, \bandwidth_x^n | \bm{X}_{n}\right) = \frac{2\alpha+d}{2\alpha}\varianceLPS{\bandwidth_x^n}{x}{1} + \convergenceInProbability{\noexpand\big}{n^{-\frac{2\alpha}{2\alpha+d}}}.
\end{align}
\end{restatable}
\begin{proof}
For the optimal decay rate $h_n^{} = n^{-\frac{1}{2\alpha+d}}$ define
$\bar\bandwidth_x^n := h_n^{-1}\bandwidth_x^n$.
Let us write $\textstyle\Bias_{1}^{}\left[x, \bandwidth_x^n | \bm{X}_{n}\right]^2 = C(\bar\bandwidth_x^n)h_n^{2\alpha} + \convergenceInProbability{}{h_n^{2\alpha}}$ such that
\[\textstyle\MSE_{1}^{}\left(x, h_n^{}\bar\bandwidth_x^n | \bm{X}_{n}\right) = C(\bar\bandwidth_x^n)h_n^{2\alpha} + \frac{R(k)v(x)}{h_n^{d}\abs{\bar\bandwidth_x^n} p(x)n}
+ \convergenceInProbability{}{h_n^{2\alpha}+n^{-1}h_n^{-d}}.\]
Necessarily, $C(\bar\bandwidth_x^n)\not\rightarrow\infty$ and $\abs{\bar\bandwidth_x^n}\not\rightarrow 0$ since else $\textstyle\MSE_{1}^{}\left(x, h_n^{}\bar\bandwidth_x^n | \bm{X}_{n}\right) = \divergenceInProbability{}{h_n^{2\alpha}}$, contradicting the optimality of the sequence.
On the other hand, $C(\bar\bandwidth_x^n)\not\rightarrow 0$ and $\abs{\bar\bandwidth_x^n}\not\rightarrow \infty$ as we could else adjust $h_n^{}$ accordingly, which results in a \MSE convergence rate faster than the optimal rate.
Therefore both sequences, $C(\bar\bandwidth_x^n)$ and $\abs{\bar\bandwidth_x^n}$, have at least one accumulation point.
Without loss of generality, we assume $C(\bar\bandwidth_x^n) \rightarrow C$ and $\abs{\bar\bandwidth_x^n} \rightarrow A$, since the result holds true for any subsequence converging to an arbitrary combination of accumulation points.

The optimal $h_n^{}$ can be found by setting the derivative of the leading error terms to zero:
\begin{align*}
&0 = \frac{\partial}{\partial h_n^{}}\MSE_{1}^{}\left(x, h_n^{}\bar\bandwidth_x^n | \bm{X}_{n}\right) = 2\alpha Ch_n^{2\alpha-1} - \frac{d R(k)v(x)}{h_n^{d+1}A p(x)n}\\
\Leftrightarrow\; &h_n^{2\alpha+d} = \frac{d R(k)v(x)}{2\alpha CA p(x)n}.
\end{align*}
However, since $\bandwidth_x^n = h_n^{}\bar\bandwidth_x^n$ is optimal for $h_n^{} = n^{-\frac{1}{2\alpha+d}}$, it already must hold
\[\frac{d R(k)v(x)}{2\alpha CA p(x)} = 1 \Leftrightarrow C = \frac{d R(k)v(x)}{2\alpha A p(x)}.\]
Therefore
\begin{align*}
 \textstyle\Bias_{1}^{}&\left[x, \bandwidth_x^n | \bm{X}_{n}\right]^2 + \convergenceInProbability{}{h_n^{2\alpha}}\\
 &= h_n^{2\alpha}C = h_n^{2\alpha}\frac{d R(k)v(x)}{2\alpha A p(x)} = \frac{d}{2\alpha}\varianceLPS{\bandwidth_x^n}{x}{1} + \convergenceInProbability{}{h_n^{2\alpha}},
\end{align*}
such that 
\begin{align*}
\textstyle\MSE_{1}^{}\left(x, \bandwidth_x^n | \bm{X}_{n}\right) = \frac{2\alpha+d}{2\alpha}\varianceLPS{\bandwidth_x^n}{x}{1}+ \convergenceInProbability{}{h_n^{2\alpha}}.\qedAtBottomLine
\end{align*}
\end{proof}
From this, we can deduce the asymptotic uniqueness of the reciprocal determinant of $\LOBfunctionOfLPS{1, \posDefSet{d}}{n}(x)$:
\begin{restatable}{mycor}{uniqueLOBdeterminant}
\label{cor:uniqueLOBdeterminant}
May the assumptions of \thmref{thm:generalizedLOBandMSE} hold for $x\in\inputSpaceInterior$.\newline
Furthermore let $(\bandwidth_n)_{n\in\N}, (\bandwidth_n')_{n\in\N}\in\consistentSigmaSequencSpace$ be two consistent, optimal bandwidth sequences.
That is, $\bandwidth_n, \bandwidth_n' \in \LOBfunctionOfLPS{1, \posDefSet{d}}{n}(x)$. Then
\[\abs{h_n^{-1}\bandwidth_n}^{-1} = \abs{h_n^{-1}\bandwidth_n'}^{-1} + \convergenceInProbability{}{1}.\]
\end{restatable}
\begin{proof}
It holds $\textstyle\MSE_{1}^{}\left(x, \bandwidth_n | \bm{X}_{n}\right) = \MSE_{1}^{}\left(x, \bandwidth_n' | \bm{X}_{n}\right) + \convergenceInProbability{}{h_n^{2\alpha}}$ due to the optimality of both sequences, such that according to \lemref{lem:generalizedBiasVarianceBalance} it is
\[\frac{2\alpha+d}{2\alpha}\frac{R(k)v(x)}{\abs{\bandwidth_n} p(x)n} + \convergenceInProbability{}{h_n^{2\alpha}} = \frac{2\alpha+d}{2\alpha}\frac{R(k)v(x)}{\abs{\bandwidth_n'} p(x)n} + \convergenceInProbability{}{h_n^{2\alpha}}.\]
Therefore
\begin{align*}
 &[n\abs{\bandwidth_n}]^{-1} = [n\abs{\bandwidth_n'}]^{-1} + \convergenceInProbability{}{h_n^{2\alpha}}
 \Leftrightarrow \abs{h_n^{-1}\bandwidth_n}^{-1} = \abs{h_n^{-1}\bandwidth_n'}^{-1} + \convergenceInProbability{}{1}\qedAtBottomLine
\end{align*}
\end{proof}
From here on, we can treat $D_n(x) = \abs{h_n^{-1}\LOBfunctionOfLPS{1, \posDefSet{d}}{n}(x)}$ as a function that is asymptotically well-defined. This family of functions features pointwise convergence: 
\begin{restatable}{mylem}{uniqueAsymptoticLOBdeterminant}
\label{lem:uniqueAsymptoticLOBdeterminant}
There exists a function $\fctn{D}{\inputSpace}{\positiveReal{}}$ such that 
\begin{align}
\label{eq:asymptoticGeneralNormalizedBWdeterminant}
 D_n(x) = D(x) + \convergenceInProbability{}{1}.
\end{align}
\end{restatable}
\begin{proof}
Let us do the following preliminary consideration:
At (sampling-)time $n$ we regard $x_{n+1}$ as close to $x$, relative to $\LOBfunctionOfLPS{1, \posDefSet{d}}{n}(x)$, if $\pnorm{\LOBfunctionOfLPS{1, \posDefSet{d}}{n}(x)^{-1}(x-x_{n+1})}{} \leq s$, where it is $[0,s]$ the support of the kernel $k$. Let therefore $p_n = \Prob\{\text{'}x_{n+1}\text{ is close to }x\text{'}\} = c\abs{\LOBfunctionOfLPS{1, \posDefSet{d}}{n}(x)}$ for some adequate constant $c > 0$. Note that $p_n = \upperBoundedInProbability{}{h_n^d}$. Furthermore let $D_n(x)$ be known and fixed.

Now, if $x_{n+1}$ is not close to $x$, then $x_{n+1}$ has no influence on the \MSE and \LOB in $x$. Thus, $\LOBfunctionOfLPS{1, \posDefSet{d}}{n+1}(x) = \LOBfunctionOfLPS{1, \posDefSet{d}}{n}(x)$ such that
\[
D_{n+1}(x) = \Big|h_{n+1}^{-1}\LOBfunctionOfLPS{1, \posDefSet{d}}{n+1}(x)\Big|^{-1}\!\!\! = \frac{h_{n+1}^{d}}{h_{n}^{d}}\Big|h_{n}^{-1}\LOBfunctionOfLPS{1, \posDefSet{d}}{n}(x)\Big|^{-1}\!\!\! = \Big[\frac{n}{n+1}\Big]^\frac{d}{2\alpha+d}D_{n}(x).
\]
If $x_{n+1}$ is close to $x$, then it has an influence on \LOB in $x$, which is proportional to the effective sample size $N_n(x) := n c'\abs{\LOBfunctionOfLPS{1, \posDefSet{d}}{n}(x)}$ in $x$ at time $n$, for some adequate constant $c' > 0$.
In particular, the effective sample size grows as
\[N_{n+1}(x) = N_n(x) + 1 = (n + [c'\abs{\LOBfunctionOfLPS{1, \posDefSet{d}}{n}(x)}]^{-1}) c'\abs{\LOBfunctionOfLPS{1, \posDefSet{d}}{n}(x)}.\]
In summary, the event that $x_{n+1}$ is not close to $x$ leads to a decrease of $D_{n+1}(x)$, whereas $x_{n+1}$ is close to $x$ will in general lead to an increase of $D_{n+1}(x)$.
Additionally, as $n$ grows, the per-sample impact on $D_n(x)$ decreases in both directions.

Like in the proof of \lemref{lem:generalizedBiasVarianceBalance}, we know that
$D_n(x)$ is a sequence that is bounded in probability: If not, then the bandwidth sequence $\LOBfunctionOfLPS{1, \posDefSet{d}}{n}(x)$ would lead to a decay-law of the \MSE that is different to the optimal $\upperBoundedInProbability{}{h_n^{2\alpha}}$, which either contradicts the optimality of this decay-law or the optimality of the bandwidth sequence itself.
Therefore there exist $0 < \underline{D} < \overline{D} < \infty$, which bound $D_n(x)$ in probability. That is, for all $\varepsilon > 0$ there exists $\mathfrak{N}\in\N$ such that for all $n\geq \mathfrak{N}$:
\[\Prob(\underline{D} \leq D_n(x) \leq \overline{D}) > 1 - \varepsilon\]
Now, if $D(x)$ would not exist for which $D_n(x) = D(x) + \convergenceInProbability{}{1}$ holds,
then there exists $\delta > 0$ such that for any $0 < \underline{D} < \overline{D} < \infty$, which bound $D_n(x)$ in probability, it is $\overline{D} - \underline{D} \geq \delta$.
Without loss of generality let us choose these bounds such that $\overline{D} - \underline{D} = \delta$.
Then there exist $\underline{\varepsilon}, \overline{\varepsilon} > 0$ such that for infinitely many $n\in\N$ it is 
\[
\Prob(D_n(x) \leq \underline{D} + \frac{\delta}{4}) \geq \underline{\varepsilon} \quad\text{and}\qquad \Prob(D_n(x) \geq \overline{D} - \frac{\delta}{4}) \geq \overline{\varepsilon}.
\]
In particular, we can construct index sequences $(n_m)_{n\in\N}, (n_m')_{n\in\N}$ such that $n_m < n_m' < n_{m+1}$, $n_m \rightarrow \infty$, and for all $m\in\N$:
\[
\Prob(D_{n_m}(x) \leq \underline{D} + \frac{\delta}{4}) > \underline{\varepsilon} \quad\text{and}\qquad \Prob(D_{n_m'}(x) \geq \overline{D} - \frac{\delta}{4}) \geq \overline{\varepsilon}.
\]
Then, for all $m\in\N$, $\Prob(D_{n_m'}(x) - D_{n_{m+1}}(x) > \frac{\delta}{2}) > \underline{\varepsilon}\overline{\varepsilon} =: \tilde{\varepsilon}$.

We now can choose $M > m$ such that $\LOBfunctionOfLPS{1}{n_m'}(x) \geq_{L} \LOBfunctionOfLPS{1}{n_M}(x)$ in the sense of the \emph{Loewner} order. That is, $\LOBfunctionOfLPS{1}{n_m'}(x) - \LOBfunctionOfLPS{1}{n_M}(x)$ is positive semi-definite.
Then, letting $N = n_M' - n_{M}$, the sampled training inputs
\[
x_{n_M+1} := z_1, \ldots, x_{n_M'} := z_N
\]
lead to an increase from $D_{n_M}(x)$ to $D_{n_M'}(x)$ of at least $\frac{\delta}{2}$ with probability greater than $\tilde{\varepsilon}$.

Now that $\LOBfunctionOfLPS{1}{n_m'}(x) \geq_{L} \LOBfunctionOfLPS{1}{n_M}(x)$, the samples $z_1,\ldots,z_N$ are closer to $x$ at time $n_m'$ than at time $n_M$. Therefore, when observing
\[
x_{n_m'+1} = z_1, \ldots, x_{n_m'+N} = z_N
\]
at this earlier point-in-time, when moving from $D_{n_m'}(x)$ to $D_{n_m'+N}(x)$, we would encounter at least the same amount of samples -- if not more -- that lead to an increase.
Combining this with the fact that the per-sample fluctuations are stronger at time $n_m'$ than at $n_M$, it is 
\[
\Prob(D_{n_m'+N}(x) - D_{n_m'}(x) \geq \frac{\delta}{2}) \geq \Prob(D_{n_M'}(x) - D_{n_M}(x) \geq \frac{\delta}{2}) > \tilde{\varepsilon}.
\]
Since it was $\Prob(D_{n_m'}(x) \geq \overline{D} - \frac{\delta}{4}) \geq \underline{\varepsilon}$, it follows
\begin{align*}
\Prob(D_{n_m'+N}&(x) \geq \overline{D} + \frac{\delta}{4})\\
&\geq \Prob(D_{n_m'}(x) \geq \overline{D} - \frac{\delta}{4}, D_{n_m'+N}(x) - D_{n_m'}(x) \geq \frac{\delta}{2})
\geq \underline{\varepsilon}\tilde{\varepsilon} =: \epsilon > 0.
\end{align*}
Finally, for any $\mathfrak{N}\in\N$ we can choose $m \geq \mathfrak{N}, M > m$ and $N = n_M' - n_{M}$ such that, respectively, $n_m'+N > \mathfrak{N}$ with $\Prob(D_{n_m'+N}(x) \geq \overline{D} + \frac{\delta}{4}) > \epsilon$.
This is in contradiction to the choice of $\overline{D}$ as an upper bound of $D_n(x)$ in probability.
\end{proof}

We will assume from now on, that $f\in\diffableFunctions{\inputSpace}{\alpha}$, and furthermore, that 
$\alpha(f,x) \equiv \alpha$ for all $x\in\inputSpaceInterior$.
In this case, we can choose the functions $D_n$ to be continuous.
\begin{mylem}[Continuity]
\label{lem:continuousNormalizedDeterminant}
Let $f\in\diffableFunctions{\inputSpace}{\alpha}$ and may the assumptions of \thmref{thm:generalizedLOBandMSE} hold uniformly with $\alpha(f,x) \equiv \alpha$ for all $x\in\inputSpaceInterior$.
Then there exist $D_n\in\diffableFunctions{\inputSpace,\positiveReal{}}{0}$ such that
\[\abs{h_n^{-1}\LOBfunctionOfLPS{1, \posDefSet{d}}{n}(x)}^{-1} = D_n(x) + \convergenceInProbability{}{1},\]
almost everywhere in \inputSpace.
\end{mylem}
\begin{proof}
From the proof of \lemref{lem:minimizerExistence} we know that for $\gamma = \frac{1}{2\alpha+d+1} > 0$, we can bound $\pnorm{\bandwidth_x^n}{} \leq n^{-\gamma}$ and $\abs{\bandwidth_x^n} \geq \frac{1}{n}$ for almost every $x \in \inputSpaceInterior$, $n$ large enough and $\bandwidth_x^n\in \LOBfunctionOfLPS{1, \posDefSet{d}}{n}(x)$.
Let
\[
\textstyle S_{n} = \condset{\bandwidth\in\SigmaSpace}{\pnorm{\bandwidth}{} \leq n^{-\gamma}, \abs{\bandwidth} \geq n^{-1}},
\]
and recall that $S_n$ is compact. Using that $\MSE_{1}^{}\left(x, \bandwidth| \bm{X}_{n}\right)$ is continuous in both, $x\in\inputSpace$ and $\bandwidth\in\SigmaSpace$, it generally holds that
\[
\textstyle x \mapsto \min_{\bandwidth\in S_n}\MSE_{1}^{}\left(x, \bandwidth | \bm{X}_{n}\right) \in \diffableFunctions{\inputSpace,\positiveReal{}}{0}.
\]
is a continuous function over the input space. Note that, for $n$ large enough,
\[
 \textstyle\MSE_{1}^{}\left(x, \bandwidth_x^n | \bm{X}_{n}\right) = \min_{\bandwidth\in S_n}\MSE_{1}^{}\left(x, \bandwidth | \bm{X}_{n}\right)
\]
and, according to \lemref{eq:generalizedBiasVarianceOnParExpression}, that
\begin{align*}
 \MSE_{1}^{}&\left(x, \bandwidth_x^n | \bm{X}_{n}\right)\\
 &= \frac{2\alpha+d}{2\alpha}\varianceLPS{\bandwidth_x^n}{x}{1} + \convergenceInProbability{}{h_n^{2\alpha}} = \frac{2\alpha+d}{2\alpha}\frac{R(k)v(x)}{\abs{\bandwidth_z^n} p(x)n} + \convergenceInProbability{}{h_n^{2\alpha}}.
\end{align*}
When rearranging terms, it is
\[
\abs{h_n^{-1}\bandwidth_x^n}^{-1} = \underbrace{h_n^{-2\alpha}\MSE_{1}^{}\left(x, \bandwidth_x^n | \bm{X}_{n}\right)\frac{2\alpha}{2\alpha+d}\frac{p(x)}{R(k)v(x)}}_{=:D^\circ_n(x)} + \convergenceInProbability{}{1}.
\]

With all functions, $\MSE_{1}^{}\left(x, \bandwidth_x^n | \bm{X}_{n}\right)$, $p(x)$ and $v(x)$, being continuous in $x$, it follows that $D^\circ_n(x)$ -- as defined above -- is continuous in $x$.
Now that the above proof holds true for almost every $x\in\inputSpaceInterior$, the function $D^\circ_n$ is well-defined and continuous, almost everywhere in \inputSpaceInterior. Finally, the closure of $D^\circ_n$ in \inputSpace, given by
\[
D_n(x) = \lim_{z\in\inputSpaceInterior,z\rightarrow x} D^\circ_n(z),
\]
is continuous and fulfills via construction almost everywhere
\begin{align*}
\abs{h_n^{-1}\LOBfunctionOfLPS{1, \posDefSet{d}}{n}(x)}^{-1} = D_n(x) + \convergenceInProbability{}{1}.\qedAtBottomLine
\end{align*}
\end{proof}

\asymptoticLOBdeterminant*
\begin{proof}
The claim follows from combining \corollaryref{cor:uniqueLOBdeterminant}, \lemref{lem:uniqueAsymptoticLOBdeterminant} and \ref{lem:continuousNormalizedDeterminant}.
\end{proof}

We now know that $D_n$ is a family of continuous functions with pointwise limits $D(x) = \lim_{n\rightarrow\infty}D_n(x)$. 
But since we are not aware about uniform convergence of $D_n$, we cannot imply the continuity of $D(x)$. Yet, we can try to deduce from the scaling behavior of $D$ how to adjust $D_n$ for the effects of $v$ and $p$:

Since the bias-component does not depend on the noise level $v$ in general, it is $D_n(x) \propto v(x)^{-\frac{d}{2\alpha+d}}$ which follows straight-forward to the isotropic case.
In contrast, we have seen that higher-order bias-components depend on $p$ through its derivatives.
Even though it is impossible to construct the true asymptotic bias explicitly in the indefinite regime, it is therefore likely to depend on $p$ in a non-trivial way. 

Yet, in case of $p \sim \uniformDist{\inputSpace}$ this problem does not occur, since all derivatives of $p$ vanish. The same argument will hold for training densities that can be written as step-functions
\[p(x) = \mySum{s=1}{S} \indicatorFunction{X_s}{x} P_s,\]
where $\inputSpace = X_1 \uplus \ldots \uplus X_S$ is a partition of the input space with constant density values $P_s > 0$.

\end{document}